\documentclass{article}

\usepackage{arxiv}
\usepackage{times}
\usepackage{epsfig}
\usepackage{graphicx}
\usepackage{amsmath}
\usepackage{amssymb}
\usepackage{booktabs}

\newcommand{\cHiro}[1]{{\color{black}{#1}}}

\usepackage{amsthm} 

\newtheorem{definition}{Definition}
\newtheorem{lemma}{Lemma}
\newtheorem{remark}{Remark}
\newtheorem{hypothesis}{Hypothesis}
\newtheorem{assumption}{Assumption}
\newtheorem{observation}{Observation}

\newif\ifshowfigs
\showfigstrue 

\newif\ifshowapp
\showappfalse

\newif\ifshowappdetail
\showappdetailtrue 

\newif\ifshowappcite
\showappcitefalse

\def\bsq#1{
	\lq{#1}\rq
}
\usepackage{amsfonts} 

\newcommand{\norm}[1]{\left\lVert#1\right\rVert}

\usepackage{mathtools}
\DeclarePairedDelimiterX{\infdivx}[2]{(}{)}{%
	#1\;\delimsize\|\;#2%
}

\DeclareMathOperator{\Tr}{Tr}

\DeclareMathOperator{\EX}{\mathbb{E}}

\usepackage{booktabs} 
\usepackage{graphicx} 
\usepackage{stmaryrd} 
\usepackage{textcomp} 

\usepackage{adjustbox}

\usepackage[para]{threeparttable}
\usepackage{array}

\usepackage{color}

\usepackage{dsfont}

\usepackage[thinc]{esdiff}

\usepackage{mathtools}

\usepackage{caption}

\usepackage{float}

\usepackage[bottom]{footmisc}

\usepackage{amsmath}
\usepackage[noend]{algpseudocode}

\algnewcommand{\Initialize}[1]{%
	\State \textbf{Initialize:}
	\State \hspace*{\algorithmicindent}\parbox[t]{0.8\linewidth}{\raggedright #1}
}

\usepackage{enumitem} 
\usepackage{authblk} 

\title{Understanding Likelihood of Normalizing Flow and Image Complexity
	through the Lens of Out-of-Distribution Detection}

\author[1]{Genki Osada}
\author[1]{Tsubasa Takahashi}
\author[2]{Takashi Nishide}

\affil[1]{LY Corporation, Japan}
\affil[2]{University of Tsukuba, Japan}

\usepackage{xr}
\makeatletter
\newcommand*{\addFileDependency}[1]{
	\typeout{(#1)}
	\@addtofilelist{#1}
	\IfFileExists{#1}{}{\typeout{No file #1.}}
}\makeatother

\newcommand*{\myexternaldocument}[1]{%
	\externaldocument{#1}%
	\addFileDependency{#1.tex}%
	\addFileDependency{#1.aux}%
}
\myexternaldocument{UntrustworthyL_sub_app_240127}

\usepackage{lipsum}

\begin{document}

	\maketitle
\begin{abstract}
	Out-of-distribution (OOD) detection is crucial to safety-critical machine learning applications and has been extensively studied.
	While recent studies have predominantly focused on classifier-based methods,
	research on deep generative model (DGM)-based methods have lagged relatively.
	This disparity may be attributed to a perplexing phenomenon: DGMs often assign higher likelihoods to unknown OOD inputs than to their known training data.
	This paper focuses on explaining the underlying mechanism of this phenomenon.
	We propose a hypothesis that less complex images concentrate in high-density regions in the latent space, resulting in a higher likelihood assignment in the Normalizing Flow (NF).
	We experimentally demonstrate its validity for five NF architectures, concluding that their likelihood is untrustworthy.
	Additionally, we show that this problem can be alleviated by treating image complexity as an independent variable.
	Finally, we provide evidence of the potential applicability of our hypothesis in another DGM, PixelCNN++.
\end{abstract}

\section{Introduction}
\label{sec:intro}
Deep neural network (DNN) models deployed in real-world systems often encounter out-of-distribution (OOD) inputs —
samples from a different distribution of the training set.
DNN models often make incorrect predictions with high confidence for OOD inputs~\cite{7298640}, making it crucial to distinguish OOD inputs from in-distribution (In-Dist) ones at test time, especially for safety-critical applications such as autonomous driving~\cite{du2022vos} and medical diagnosis~\cite{pmlr-v121-linmans20a}.
Numerous methods have been proposed to improve empirical performance~\cite{yang2022openood, DBLP:journals/corr/abs-2110-11334}, and analytical approaches have also been studied~\cite{morteza2022provable,NEURIPS2020_543e8374}.

One major approach to OOD detection is using deep generative models (DGMs).
Particularly, Normalizing Flows (NFs) and Autoregressive (AR) models are the two most commonly selected DGMs due to their ability to compute exact model likelihoods.
The DGM-based approach is attractive for its strengths: it does not require labeled data to train detection models, and the performance of OOD detection is independent of the number of classification classes~\cite{Huang_2021_CVPR}.
However, most of the recent progress in this field has been based on another prevailing approach, the classifier-based method~\cite{yang2022openood, DBLP:journals/corr/abs-2110-11334}.
The primary factor hindering the progress of DGM-based methods may be the observation presented by
\cite{nalisnick2018do,choi2019generative}:
DGMs were expected to assign a higher likelihood, i.e., probabilistic density, to In-Dist inputs than to OOD inputs, but it turned out that this is not the case in some particular cases.
We refer to this phenomenon as
the \emph{failure of likelihood}.
When an OOD detection method is deployed to real-world applications with no control over its inputs, this unreliability is unacceptable.
Above all, correctly estimating the likelihood is a fundamental feature that DGMs are expected to fulfill.
Hence, research on DGM-based methods has focused on addressing this issue  before enhancing detection performance and tackling more challenging tasks.

Several hypotheses have been proposed to explain this phenomenon.
One argues that image complexity has some influence on the likelihood~\cite{nalisnick2018do, Serra2020Input}, while another suggests the need to account for the notion of the \emph{typical set}~\cite{choi2019generative,nalisnick2020detecting}.
However, these methods still fail in specific cases, as we will show later, and the mechanism behind the failure of likelihood remains unclear.

In this paper, we address the cause of the failure of likelihood.
We focus on Normalizing Flows due to their analytical tractability.
The contributions of this work are as follows:
\begin{itemize}
	\item
	We present a hypothesis that explains how image complexity can arbitrarily control the likelihood of NFs and how the density concentration in latent distribution causes this phenomenon.

	\item Our hypothesis provides a unified explanation for two separate questions: why images with less complexity are assigned high likelihood and why OOD inputs are regarded as the typical set samples in some cases.

	\item We experimentally verify our argument with five different NF architectures and conclude that OOD detections based on the likelihood of NFs are untrustworthy.
\end{itemize}

The relationship between the likelihood and image complexity that we explain in this paper
finds implicit support in Independent Component Analysis (ICA)~\cite{10.1016/S0893-6080(00)00026-5},
as well as in Shannon's source coding theorem~\cite{shannon1948mathematical},
as discussed later in this section.

Additionally, we propose a countermeasure and argue that the failure of likelihood can be overcome by leveraging information about the root cause of the problem, namely image complexity.
We demonstrate that a Gaussian mixture model (GMM) with two variables, image complexity and likelihood in latent space, can detect all of the OOD datasets we used in our evaluation, thereby highlighting the validity of our argument from another perspective.

Finally, given a recent study in which AR models can be seen as a specific type of NF~\cite{NEURIPS2020_26ed695e},
we show the potential applicability of our hypothesis to AR models as well through the experiments using PixelCNN++.


\paragraph{Related works.}
\label{sec:related}
Two main hypotheses have been proposed to explain the failure of likelihood-based OOD detection:
the involvement of \emph{image complexity} and the necessity of considering  \emph{typical set}.
Experimental studies by \cite{Serra2020Input, NEURIPS2019_1e795968} demonstrated that the log-likelihood $\log p({\bf x})$ increases with less complex images,
but the underlying mechanism remained unclear.
\cite{NEURIPS2020_f106b7f9, NEURIPS2020_ecb9fe2f} attributed this phenomenon to the architecture of the NFs.
However, the same phenomenon has also been observed in PixelCNN~\cite{nalisnick2018do, NEURIPS2019_1e795968}, implying that it is not exclusively caused by the architecture of the NFs.
\cite{nalisnick2018do} suggested that the phenomenon could be theoretically explained.
Their analysis implicitly assumed that the likelihood in the latent space and the determinant of the Jacobian matrix remain consistent across different datasets.
However, as shown in Figs.\ \ref{fig:manipulated_c10} (right) and \ref{fig:comp_vs_z__ood}, in reality, these values differ considerably between datasets.
We explain that this variation arises from differences in image complexities, ultimately serving as a factor that can arbitrarily influence $\log p({\bf x})$.
We experimentally show that our hypothesis likely applies to PixelCNN as well.

\cite{choi2019generative,nalisnick2020detecting}  introduced a testing method based on the typical set in latent space.
However, \cite{NEURIPS2020_66121d1f,choi2019generative,pmlr-v139-zhang21g} concluded that its performance was insufficient.
Our hypothesis also explains \emph{why} typicality-based testing fails.
Additionally, in Appendix \ref{sec:related_additional},
we present other DGM-based approaches and an overview of classifier-based detection methods.

Beyond the context of OOD detection, the credibility of likelihood has been discussed.
\cite{Theis2016a} has shown that high likelihood does not necessarily indicate high sample quality produced by the model.
\cite{10.1007/978-3-642-04277-5_71} proposed that image complexity can be estimated by $\log p({\bf x})$ of ICA~\cite{10.1016/S0893-6080(00)00026-5}, which can be regarded as a linear NF~\cite{dinh2014nice}.
Their study implies that less complex image ${\bf x}$ is assigned a higher $\log p({\bf x})$ in ICA, aligning with our assertion.
Moreover, Shannon's source coding theorem~\cite{shannon1948mathematical} states that
the expected code length of a symbol ${\bf x}$ is bounded by its entropy:
$
\EX_{{\bf x} \sim P_{\text{x}}} [- \log p({\bf x})] \leq \EX_{{\bf x} \sim P_{\text{x}}} [ L ({\bf x}) ].
$
Here, $L ({\bf x})$ represents the length of the encoded message for ${\bf x}$ using a lossless compression algorithm, and $P_{\text{x}}$ denotes a data distribution, which is generally unknown.
Minimizing $\EX_{{\bf x} \sim P_{\text{x}}} [ L ({\bf x}) ]$ is achieved by assigning a smaller $L({\bf x})$ to ${\bf x}$ with a larger $\log p({\bf x})$ and a larger $L({\bf x})$ to ${\bf x}$ with a smaller $\log p({\bf x})$,
ensuring efficient encoding~\cite{bishop2006pattern}.
The complexity of an image ${\bf x}$ is defined in the next section as $L ({\bf x}) / d$, where ${\bf x}$ with less complexity have a smaller $L ({\bf x})$.
Therefore, the theorem implies that images with less complexity are assigned higher likelihoods, aligning with our claim.

\section{Background and Problem Statement}
This paper shows that the complexity of input images can control the likelihood of Normalizing Flows (NFs).
We first define the image complexity and describe the NFs.
Then, we discuss existing approaches and their failures, leading us to articulate the problem statement addressed in this paper.

\subsection{Preliminary}
\paragraph{Image complexity and entropy coding.}
Estimating the complexity of images lacks a definitive method; however, two common options are available: Kolmogorov complexity and Shannon’s entropy~\cite{10.5555/2381219.2381244}.
While a connection is suggested between the two~\cite{kolmogorov_and_shannon_osada},
the former is uncomputable.
Thus, the prevalent method is based on entropy, which is directly used in lossless compression (coding)~\cite{sayood2017introduction, 9694511}.
In this work, we use the following definition for image complexity used in previous studies~\cite{Serra2020Input, ijcai2021p292}.
\begin{definition}[Image complexity]
	\label{def:comp}
	Let $L ({\bf x})$ be the length of the bit string obtained after compressing ${\bf x} \in \{0, 1, \ldots, 255 \}^{d}$ using a lossless compression algorithm,
	{\normalfont \texttt{comp}}.
	Image complexity for ${\bf x}$ is defined as $C({\bf x}) = \frac{1}{d} L ({\bf x})$.
	The more complex ${\bf x}$ is, the larger $C({\bf x})$ is, and vice versa.
\end{definition}
We use the JPEG2000 compression, which is based on entropy coding, as the \texttt{comp} in all our experiments.

\paragraph{Normalizing Flow.}
\label{sec:nf}
We focus our investigation on NFs~\cHiro{\cite{pmlr-v37-rezende15, dinh2014nice, dinh2016density}} among DGMs.
NFs offer distinct advantages over other types of DGMs, as they enable the exact computation of likelihoods, unlike VAEs, and allow for the separation of the volume term, facilitating analysis~\cite{nalisnick2018do}.
NFs learn an invertible mapping $f: \mathcal{X} \rightarrow \mathcal{Z}$  that maps observable data ${\bf x}$ to the latent vector ${\bf z} = f({\bf x})$, where $\mathcal{X} \in \mathbb{R}^{d}$  is the data space and $\mathcal{Z} \in \mathbb{R}^{d}$ is the latent space.
We denote a distribution on $\mathcal{Z}$ as $P_{\text{z}}$ with probability density $p({\bf z})$.
NFs learn an approximate model distribution $\widehat{P}_{\text{x}}$ to match the unknown true distribution $P_{\text{x}}$ on $\mathcal{X}$.
Under the change of variable rule, the log density of $\widehat{P}_{\text{x}}$ is expressed as:
\begin{eqnarray}
	\log p({\bf x}) = \log p({\bf z}) + \log |\text{det} \ J_{f}({\bf x})|
	\label{eq:nf}
\end{eqnarray}
where $J_{f}({\bf x}) = df({\bf x}) / d{\bf x}$
is the Jacobian matrix of $f$ at ${\bf x}$,
and $\log |\text{det} \ J_{f}({\bf x})|$ is referred to as the \emph{volume}.
By maximizing $\log p({\bf z})$ and $\log |\text{det} \ J_{f}({\bf x})|$ simultaneously with respect to samples ${\bf x} \sim P_{\text{x}}$,
the NF model $f$ is trained to match $\widehat{P}_{\text{x}}$ with $P_{\text{x}}$.
In our work, $P_{\text{x}}$ represents In-Distribution (In-Dist), and we train an NF model using samples from In-Dist.

We investigate five architectures: Glow~\cite{NIPS2018_8224}, CV-Glow~\cite{nalisnick2018do}, iResNet~\cite{pmlr-v97-behrmann19a}, ResFlow~\cite{NEURIPS2019_5d0d5594}, and IDF~\cite{NEURIPS2019_9e9a30b7}.
iResNet and ResFlow
improve the stability of NFs
by controlling the Lipschitz constant of $f$ to be less than one while maintaining high expressive power.
IDF uses a categorical distribution for $P_{\text{z}}$, while the other four architectures use the standard Gaussian distribution.
CV-Glow, as the abbreviation for \emph{Glow with the constant volume}, has a fixed volume that depends only on the weight matrix of the $1 \times 1$ convolutions, resulting in a constant volume across all inputs ${\bf x}$.
Similarly, the volume of IDF is fixed at 0 to build a latent space with integer values.
It is worth noting that CV-Glow, IDF, and even PixelCNN++~\cite{salimans2017pixelcnn} exhibit a similar behavior due to the commonality of the constant volume, as we will see later.
The implementation of these models are described in Appendix \ref{sec:dgms}.

\begin{figure}[t]
	\centering
	\includegraphics[width=0.6 \columnwidth]{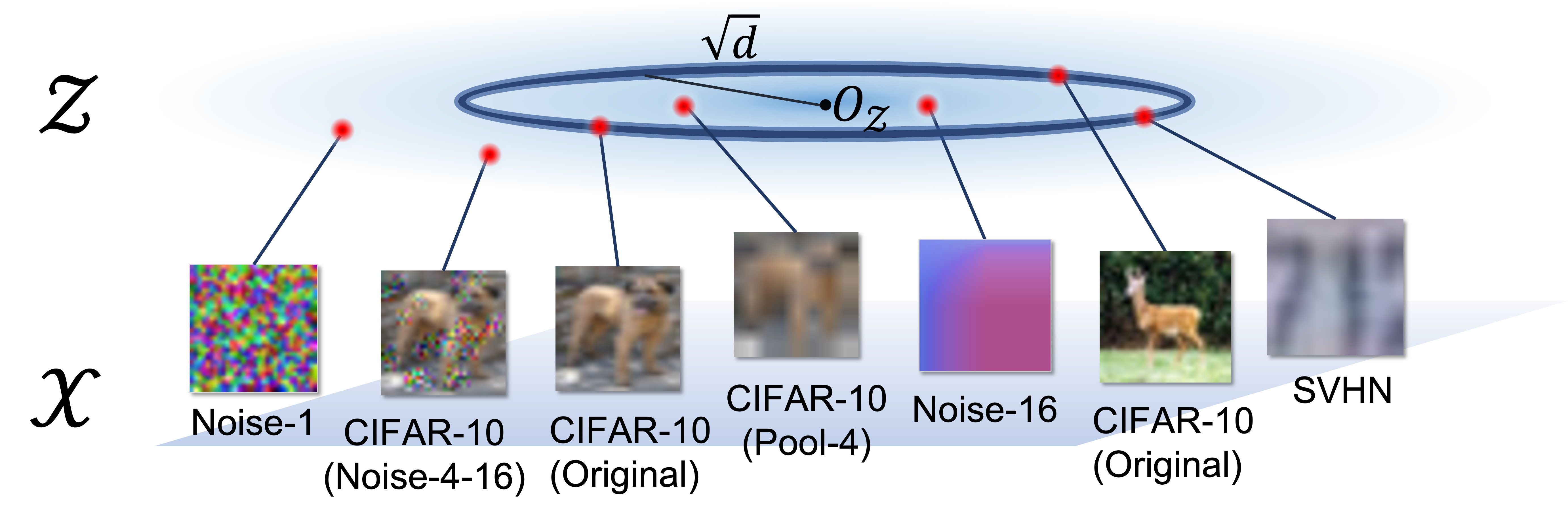}
	\caption{DCAS (Remark~\ref{rem:dcas}) attracts less complex images to the high-density region in latent space.
		$\mathcal{Z}$ represents a Gaussian latent space trained on CIFAR-10.
		$O_{\mathcal{Z}}$ represents the origin of $\mathcal{Z}$.
		The dark blue circle represents the typical set in $\mathcal{Z}$,
		identified as In-Dist by the typicality test.
		Complex OOD images like Noise-1 and CIFAR-10 (Noise-4-16) are mapped far from $O_{\mathcal{Z}}$.
		However, due to DCAS, less complex images like Noise-16 and CIFAR-10 (Pool-4) are mapped closer to $O_{\mathcal{Z}}$ than the circle.
		We hypothesize that SVHN (an OOD image) should be mapped far beyond the circle.
		However, due to its less complexity, it is attracted towards $O_{\mathcal{Z}}$ and coincidentally falls on the circle,
		leading to the misidentification of SVHN as In-Dist by the typicality test.}
	\label{fig:GaussianAnnulus_DCAS}
\end{figure}

\subsection{Failure of Existing Approaches}
\label{sec:existing}

\paragraph{Likelihood test.}

The likelihood test, introduced by \cite{bishop1995training}, is an OOD detection method that relies a density estimation model.
Treating the probabilistic density of input as a likelihood, it assumes that OOD examples would be assigned lower likelihoods compared to In-Dist examples.
However, counter-evidence presented by \cite{nalisnick2018do,choi2019generative} showed that
DGMs trained on CIFAR-10 assigned higher likelihoods to samples from SVHN (OOD) than to samples from CIFAR-10 (In-Dist), for instance.
This has sparked controversy and led to the proposal of two approaches described below as methods for improvement.

\paragraph{Complexity-aware likelihood test.}
Several studies have linked the failure of the likelihood test to image complexity \cite{nalisnick2018do, Serra2020Input, NEURIPS2019_1e795968, NEURIPS2020_f106b7f9, NEURIPS2020_ecb9fe2f}.
These studies experimentally showed that low-complexity regions in images contribute to an increase in the likelihood $\log p({\bf x})$.
One such study by \cite{Serra2020Input} introduced a complexity-aware likelihood test (CALT), which directly incorporates image complexity $C({\bf x})$.
Specifically, CALT computes the score $S_{\text{CALT}}({\bf x}) = \log p({\bf x}) +C({\bf x})$.
However, our experiments have revealed that CALT still exhibits failures in certain cases, indicating that merely adding $C({\bf x})$ to $\log p({\bf x})$ is insufficient for effective OOD detection. (See Appendix \ref{sec:calt_dz} and Section \ref{sec:gmm}.)

\paragraph{Typicality test.}
\cite{choi2019generative,nalisnick2020detecting} attributed the failure of the likelihood test to an ignorance of the notion of the \emph{typical set}.
In a $d$-dimensional isotropic Gaussian $\mathcal{N}(0,  \textbf{I}_{\text{d}})$, the typical set is expected to reside in a hypersphere with a radius of $\sqrt{d}$ with high probability, rather than around its mean (See Appendix \ref{sec:typical_set}).
Utilizing this property, \cite{choi2019generative,nalisnick2020detecting} have proposed
a method that identifies a test input ${\bf x}$ as OOD when it fall outside of the distribution's typical set, which we refer to as the typicality test in latent space (TTL).
TTL computes the score $S_{\text{TTL}} ({\bf x})= {\tt abs} (\norm{ {\bf z}} - \sqrt{d})$,
where ${\bf z} = f({\bf x})$ and $f$ is a trained NF model.
If ${\bf x}$ is OOD, $S_{\text{TTL}} ({\bf x})$ is expected to be large.
However, \cite{pmlr-v139-zhang21g,NEURIPS2020_66121d1f,choi2019generative} concluded that the performance of TTL was insufficient, and our experiments also revealed failure cases.
We observed that
$\norm{ {\bf z}}$ for In-Dist samples is concentrated around the theoretical value, $\sqrt{d}$ (i.e., $\sqrt{32 \times 32 \times 3} \simeq 55.4$ for CIFAR-10 and SVHN).
This indicates that NFs can correctly perceive test samples from In-Dist as the typical set.
However, in the case where the In-Dist is CIFAR-10, $\norm{ {\bf z}}$ for SVHN (OOD) is also highly concentrated around $55.4$,
indicating the failure of TTL to detect SVHN samples as OOD.
This situation is depicted in Fig.\ \ref{fig:GaussianAnnulus_DCAS}.
(For experimental results, see Fig.\ \ref{fig:calt_dz} (bottom two) in Appendix \ref{sec:calt_dz}.)
It becomes evident that even OOD inputs can reside within the typical set, challenging the effectiveness of the typicality test.

\subsection{Problem Statement}
\label{sec:questions}
The aforementioned issues can be summarized in the following two questions:
\begin{itemize}[nosep]
	\item Why does less image complexity result in the failure of the likelihood test?
	\item Why are OOD inputs often misclassified as typical set samples?
\end{itemize}
We address these questions in this study.

\section{Hypothesis and Experimental Validation}
We present Hypothesis~\ref{thm:c}, which comprehensively explains the questions posed in Section \ref{sec:questions}.
Subsequently, we provide experimental results that support the hypothesis.

\subsection{Our Hypothesis}
\label{sec:theory}

\begin{definition}[Local Lipschitz continuity]
	For a subset $\mathcal{A} \subset \mathcal{Z}$, we define an invertible function $f: \mathcal{X} \rightarrow \mathcal{Z}$ as locally $L_{\mathcal{A}}$-Lipschitz as follows:
	\begin{eqnarray}
		\norm{ f({\bf x}_{1}) - f({\bf x}_{2})  } \leq L_{\mathcal{A}} \norm{{\bf x}_{1} - {\bf x}_{2}}  , \forall f({\bf x}_{1}), f({\bf x}_{2})  \in \mathcal{A}.
	\end{eqnarray}
\end{definition}

\begin{assumption}
	\label{asm:semantic_continuous}
	We  consider that $f$ is implemented by an NF and assume that the latent space $\mathcal{Z}$ is semantically continuous \cite{dinh2016density}.
\end{assumption}

\begin{hypothesis}
	\label{thm:c}
	Let $f: \mathcal{X} \rightarrow \mathcal{Z}$ be an invertible function locally $L_{\mathcal{A}}$-Lipschitz for $\mathcal{A} \subset \mathcal{Z}$.
	For all ${\bf z}'  \in \mathcal{A}$, let ${\bf x}' =  f^{-1}({\bf z}')$.
	Also, let $\mathcal{B}^{\epsilon}_{{\bf z} } = \{ {\bf z}' \in \mathcal{A}: \norm{ {\bf z}' -{\bf z} }  < \epsilon \}$ and $\mathcal{B}^{\epsilon}_{{\bf z} }  \subset \mathcal{A}$ for a constant $\epsilon \geq 0$.
	Then, letting $C ({\bf x})$ be the image complexity of ${\bf x}$ (Definition \ref{def:comp}) and
	$C_{1}$ be a constant, we have
	\begin{eqnarray}
		\label{eq:main_c}
		\frac{\epsilon^{2}}{{L_{\mathcal{A}}}^{2}} \left(  1 - \mathbb{P} ( \mathcal{B}^{\epsilon}_{{\bf z} } ) \right)  \leq C_{1} \exp(C ({\bf x})).
	\end{eqnarray}
\end{hypothesis}
Considering $\mathcal{A}$ to be a very small region in $\mathcal{Z}$ and based on Assumption \ref{asm:semantic_continuous},
we posit that samples ${\bf x} =  f^{-1}({\bf z})$ and ${\bf x}' =  f^{-1}({\bf z}')$ for ${\bf z}, {\bf z}'  \in \mathcal{A}$
share common semantics.
This allows us to treat their complexities as approximately equal, i.e., $C ({\bf x}) \approx C ({\bf x}')$,
which we use in the derivation.
We derive Eq.\ \eqref{eq:main_c} from experimental observations and present it in Appendix \ref{sec:derive_by_exp}.
Assuming ${\bf x}$ follows a diagonal Gaussian distribution, we can analytically derive it as stated in Appendix \ref{sec:proof}, resulting in the right-hand side being $ \frac{d}{ 2 \pi e } \delta_{x}^{\frac{2}{d} } \exp \left( 2 C({\bf x}) \right)$ where $\delta_{x}$ is the volume of bins used in discretization.
Next, we present an observation necessary to state the following Remarks.

\begin{observation}
	\label{obs:corr}
	A positive correlation exists between $\log p({\bf z})$ and the volume $\log |\text{det} \ J_{f}({\bf x})|$ in response to the variation of input ${\bf x}$.
\end{observation}
We defer the explanation to Section \ref{sec:exp_for_hypo}, Experiment 3.
In the following Remarks, we consider that
${\bf z} = f({\bf x})$ is determined for a given input ${\bf x}$ such that $\mathbb{P} ( \mathcal{B}^{\epsilon}_{{\bf z} } )$  satisfies the inequality in Eq.\ \eqref{eq:main_c}.

\begin{figure}[t]
	\centering
	\begin{minipage}[r]{0.3 \columnwidth}
		\textsf{ {\scriptsize Pooling Noise Images : }}
	\end{minipage}
	\begin{minipage}[l]{0.6 \columnwidth}
		\includegraphics[width=0.08 \columnwidth]{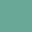}
		\includegraphics[width=0.08 \columnwidth]{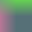}
		\includegraphics[width=0.08 \columnwidth]{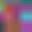}
		\includegraphics[width=0.08 \columnwidth]{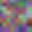}
		\includegraphics[width=0.08 \columnwidth]{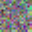}
		\includegraphics[width=0.08 \columnwidth]{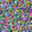}
	\end{minipage}
	\\
	\begin{minipage}[r]{0.3 \columnwidth}
		\textsf{ {\scriptsize Manipulated CIFAR-10 : }}

	\end{minipage}
	\begin{minipage}[l]{0.6 \columnwidth}
		\includegraphics[width=0.08 \columnwidth]{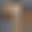}
		\includegraphics[width=0.08 \columnwidth]{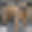}
		\includegraphics[width=0.08 \columnwidth]{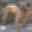}
		\includegraphics[width=0.08 \columnwidth]{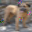}
		\includegraphics[width=0.08 \columnwidth]{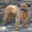}
		\includegraphics[width=0.08 \columnwidth]{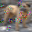}
	\end{minipage}
	\caption{Complexity controlled images. Image complexity increases from left to right in both rows.
		Top: \emph{Pooling noise images} with pooling size $\kappa$ decreases as 32, 16, 8, 4, 2, and 1 from left to right.
		Bottom: \emph{Manipulated CIFAR-10} with Pool-8, Pool-4, Pool-2, Noise-4-4, Noise-4-8, and Noise-4-16 from left to right.}
	\label{fig:ctl_noise}
\end{figure}

\begin{figure*}[t]
	\centering
	\includegraphics[height=3.2cm, trim={0 0 7.5cm 1.15cm},clip]{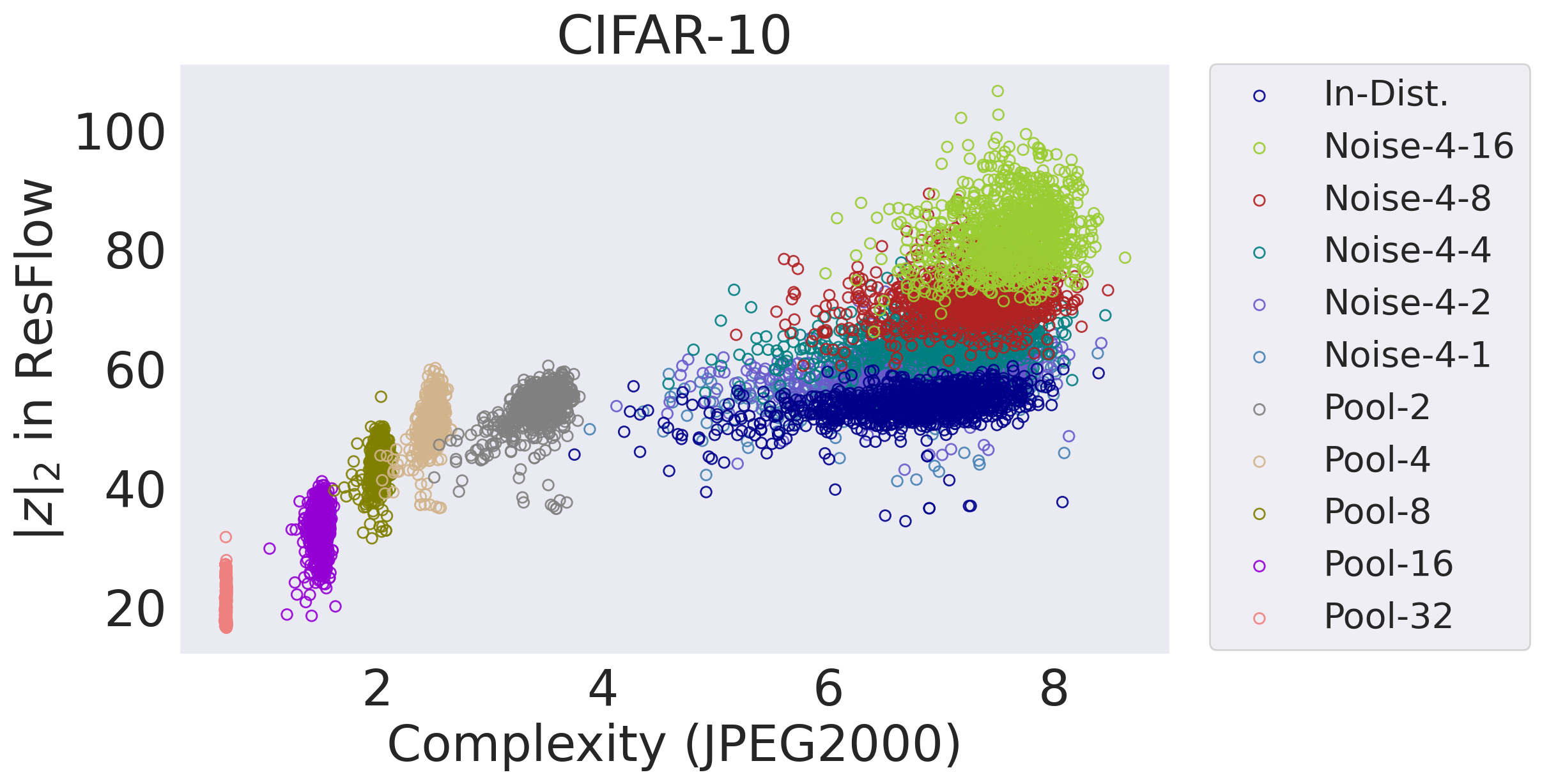}
	\hspace{10px}
	\includegraphics[height=3.2cm, trim={0 0 0 1.15cm},clip]{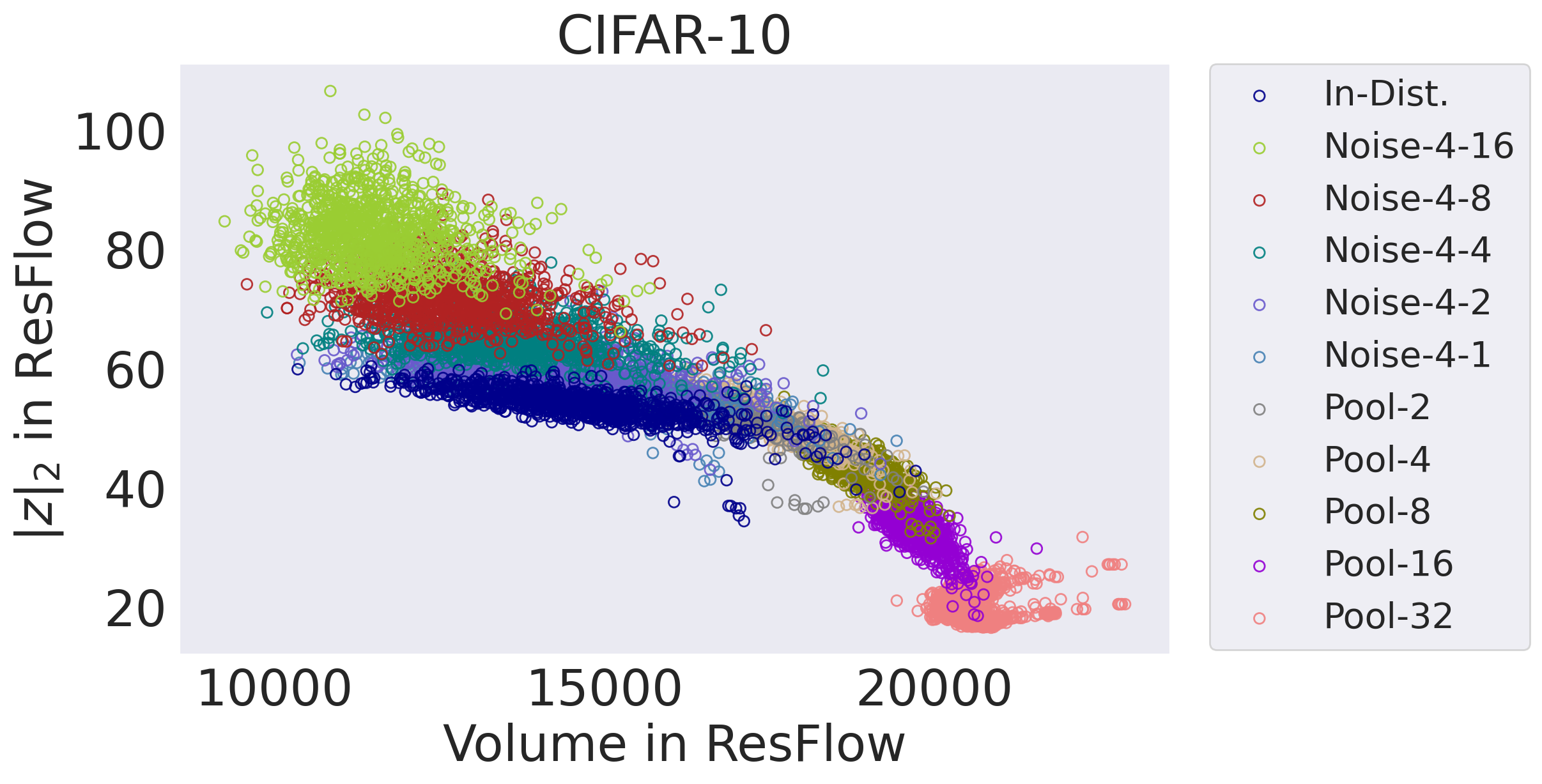}
	\caption{Plots for manipulated CIFAR-10. Left plot shows complexity vs.\ $\norm{{\bf z}}$, supporting Remark~\ref{rem:dcas}. Right plot shows volume vs.\ $\norm{{\bf z}}$, supporting Remark~\ref{rem:volume} and Observation \ref{obs:corr}. We note that $\norm{{\bf z}}  \propto - \sqrt{\log p({\bf z})}$.}
	\label{fig:manipulated_c10}
\end{figure*}


\begin{remark}
	\label{rem:dcas}
	\normalfont 
	As ${\bf x}$ becomes less complex, i.e., as $C({\bf x})$ becomes small, $\mathbb{P} ( \mathcal{B}^{\epsilon}_{{\bf z} } )$ becomes large.
	In order to increase $\mathbb{P} ( \mathcal{B}^{\epsilon}_{{\bf z} } )$, ${\bf z}$ needs to be in a high-density region in $\mathcal{Z}$,
	because the volume of $ \mathcal{B}^{\epsilon}_{{\bf z} }$, $\text{vol}(\mathcal{B}^{\epsilon}_{{\bf z} })$, is fixed by $\epsilon$.
	In other words, ${\bf z}$ for a less complex ${\bf x}$ will concentrate on a high-density region in $\mathcal{Z}$, meaning that $\log p({\bf z})$ for them will be large.
	Specifically, when the distribution of $\mathcal{Z}$ is $\mathcal{N}(0, \textbf{I}_{\text{d}})$, ${\bf z}$ is mapped to the region close to the origin, meaning that $\norm{{\bf z}}$ will be small for a less complex ${\bf x}$ (See Fig.~\ref{fig:GaussianAnnulus_DCAS}).
	We refer to this effect as Density Concentration Attraction for Simpleness, DCAS.
\end{remark}
\begin{remark}
	\label{rem:volume}
	\normalfont 
	As the input becomes less complex, $L_{\mathcal{A}}$ becomes large. Since
		$|\text{det} \ J_{f}({\bf x})| < L_{\mathcal{A}}^{d}$
	\cite{federer2014geometric}, the volume term, $\log |\text{det} \ J_{f}({\bf x})|$, for a less complex input is allowed to become large.
\end{remark}
\begin{remark}
	\normalfont 
	\label{rem:logpx}
	From Observation~\ref{obs:corr} and Remarks~\ref{rem:dcas}, \ref{rem:volume}, the decrease in image complexity increases both $\log p({\bf z})$ and $\log |\text{det} J_{f}({\bf x})|$.
	Consequently, according to Eq.\ \eqref{eq:nf}, $\log p({\bf x})$ increases for a less complex input.
\end{remark}

\subsection{Experimental Results}
\label{sec:exp_for_hypo}
To support the above Remarks, we conducted experiments using a dataset in which image complexity was systematically controlled.

\paragraph{Datasets.}
We constructed two datasets with controlled image complexity: \emph{pooling noise images} and \emph{manipulated CIFAR-10}.
The details are provided in Appendix~\ref{sec:pooling_noise_and_manipulated_C10}, and samples are shown in Fig.~\ref{fig:ctl_noise}.
The pooling noise images are generated by applying average pooling with different filter sizes, denoted as $\kappa$, to random noise images.
Each set of images is labeled as Noise-$\kappa$.
For the manipulated CIFAR-10 dataset, the complexity increases from Noise-4-1 to Noise-4-32 and decreases from Pool-2 to Pool-16, compared to the original image.
The NF models used in this section were trained on CIFAR-10~\cite{krizhevsky2009learning}.

\paragraph{Experiment 1: Complexity vs.\ $\log p({\bf z})$.}
\label{sec:exp1}
We examined Remark~\ref{rem:dcas} using the five NF models introduced in Section~\ref{sec:nf}.
We used $\norm{{\bf z}} (\propto - \sqrt{\log p({\bf z})})$ as a proxy for $\log p({\bf z})$
for the convenience of considering the typical set (for the four NFs with Gaussian latent distribution).
The plot for ResFlow on the manipulated CIFAR-10 is depicted in Fig.\ \ref{fig:manipulated_c10} (left).
It shows that $\log p({\bf z})$ increases (equivalently, $\norm{{\bf z}}$ decreases) as the complexity decreases,
providing support for Remark~\ref{rem:dcas}.
Similar plots were observed for all other nine cases, as shown in Figs.\ \ref{fig:extra_comp_vs_z_pooling} and \ref{fig:extra_comp_vs_z_manipulated} in Appendix \ref{sec:extra_exp_1}.

\paragraph{Experiment 2: $\log p({\bf z})$ vs.\ volume.}
\label{sec:exp2}
We examined Remark~\ref{rem:volume} using Glow, iResNet, and ResFlow, excluding CV-Glow and IDF, which have a constant (or zero) volume.
Fig.~\ref{fig:manipulated_c10} (right) presents the plot for ResFlow on the manipulated CIFAR-10.
It shows that a decrease in complexity (from Noise-4-16 to Pool-32) causes an increase in the volume, providing support for Remark~\ref{rem:volume}.
Similar plots were observed for all other five cases, as shown in Fig.~\ref{fig:extra_volume_vs_z} in Appendix~\ref{sec:extra_exp_2}.
Notably, Remark~\ref{rem:volume} asserts the opposite of the speculation presented in a previous study~\cite{ijcai2021p292} (that higher image complexity corresponds to a larger volume).

\paragraph{Experiment 3: Correlation between $\log p({\bf z})$  and volume.}
Fig.~\ref{fig:manipulated_c10} (right) also shows that the positive correlation between $\log p({\bf z})$ ($\propto -\norm{{\bf z}}^{2}$)
and the volume $\log |\text{det} \ J_{f}({\bf x})|$.
It indicates that
the two increase simultaneously in a balanced manner, which we have already presented as Observation~\ref{obs:corr}.
The balance between the two is a result of the learning process of NFs.
By maximizing Eq.\ \eqref{eq:nf}, an NF is trained to reach an equilibrium between two opposing forces: 1) attracting ${\bf z}$ to high-density region in $\mathcal{Z}$ to maximize $\log p({\bf z})$, and 2) scattering ${\bf z}$ to expand the volume $|\text{det} \ J_{f}({\bf x})|$~\cite{dinh2014nice,pmlr-v130-behrmann21a}.
In other words,  training an NF model involves optimizing the ratio between $\log p({\bf z})$ and the volume to partition the variation in $\log p({\bf x})$.
The correlation strength between the two, as discussed in Appendix~\ref{sec:arch_and_sensitivity}, is dependent on the architecture of NFs. We utilize this knowledge in the selection of architectures for OOD detection in Section~\ref{sec:gmm}.


\section{Likelihood is Untrustworthy}
\label{sec:}
With Remarks presented in the previous section, we address the questions posed in Section~\ref{sec:questions}
and assert the untrustworthiness of the likelihood of Normalizing Flows.
\paragraph{ \emph{Why does less image complexity result in the failure of the likelihood test?}}
Remark~\ref{rem:logpx} directly answers this question.
Inputs with less image complexity, even OOD ones, cause an increase in $\log p({\bf x})$, leading to their misidentification as In-Dist.
While some studies have experimentally shown that a decrease in the image complexity causes an increase in $\log p({\bf x})$, 
we have formulated the underlying mechanism in Eq.\ \eqref{eq:main_c} and identified that  the culprit is the effect caused by a density concentration in latent distribution, which we refer to as DCAS in Remark~\ref{rem:dcas}.

\paragraph{ \emph{Why are OOD inputs often misclassified as typical set samples?}}
Based on Remark~\ref{rem:dcas}, we explain the mechanism underlying the failure of TTL.
We provide an illustration in Fig.~\ref{fig:GaussianAnnulus_DCAS} for better intuition.
We posit that an NF trained on In-Dist data
inherently assigns a smaller $\log p({\bf z})$ to an OOD input ${\bf x}_{\text{ood}}$ compared to an In-Dist input ${\bf x}_{\text{in}}$.
In the case of NFs with a Gaussian latent space, the NF, $f$, attempts to map ${\bf z}_{\text{ood}} = f({\bf x}_{\text{ood}})$ to a region farther than the distance $\sqrt{d}$ from the origin, where the typical set samples, including ${\bf z}_{\text{in}} = f({\bf x}_{\text{in}})$, concentrate.
However, when the complexity of ${\bf x}_{\text{ood}}$ is lower, the effect of DCAS comes into play: ${\bf z}_{\text{ood}}$ is attracted to high-density region on $\mathcal{Z}$, i.e., the origin of the standard Gaussian, 
causing $\norm{{\bf z}_{\text{ood}}}$ to become smaller.
Consequently, $\norm{{\bf z}_{\text{ood}}}$ can be as small as $\sqrt{d}$,
rendering the TTL unable to differentiate such ${\bf x}_{\text{ood}}$ from the typical set, i.e., the In-Dist examples.
Therefore, the failure of TTL is caused by the balance between the \emph{original} $\norm{{\bf z}_{\text{ood}}}$ assuming the effect of DCAS could be removed, and  the extent to which the DCAS shrinks $\norm{{\bf z}_{\text{ood}}}$.
Both factors depend on the combination of In-Dist dataset and OOD inputs.
The combination of CIRAR-10 for In-Dist and SVHN for OOD is a case where $\norm{{\bf z}_{\text{ood}}}$ is coincidentally shrunk to approximately $\sqrt{d}$.
On the other hand, inputs with even less complexity than SVHN, such as Pool-16/32, are more intensely affected by the DCAS,
resulting in $\norm{{\bf z}_{\text{ood}}} < \sqrt{d}$.
In this case, ${\bf z}_{\text{ood}}$ falls out of the annulus where the typical set samples reside, and thus the TTL can successfully identify such $\norm{{\bf x}_{\text{ood}}}$ as OOD.

\paragraph{Untrustworthiness of $\log p({\bf z})$ and $\log p({\bf x})$.}
We explained the reasons why existing OOD detection methods fail above.
Now, we present the main claim of this study.
The impact of image complexity on the volume term can be mitigated by the selection of the NF architecture, and indeed, in fixed-volume architectures such as CV-Glow and IDF,
the effect described in Remark~\ref{rem:volume} is nullified (Appendix~\ref{sec:arch_and_sensitivity}).
However, the effect of DCAS described in Remark~\ref{rem:dcas} affects $\log p({\bf z})$ whenever the latent distribution $P_{\text{z}}$ has a density concentration, regardless of its distributional form.
Since there is no way to disable the DCAS, $\log p({\bf z})$ is unreliable.
This not only undermines the trustworthiness of  the TTL, but also renders $\log p({\bf x})$, which incorporates $\log p({\bf z})$, untrustworthy.
Therefore, we conclude that the likelihood of Normalizing Flows is untrustworthy.

\begin{figure*}[t]
	\centering
	\includegraphics[height=3.15cm, trim={0 0 6cm 1.15cm}, clip]{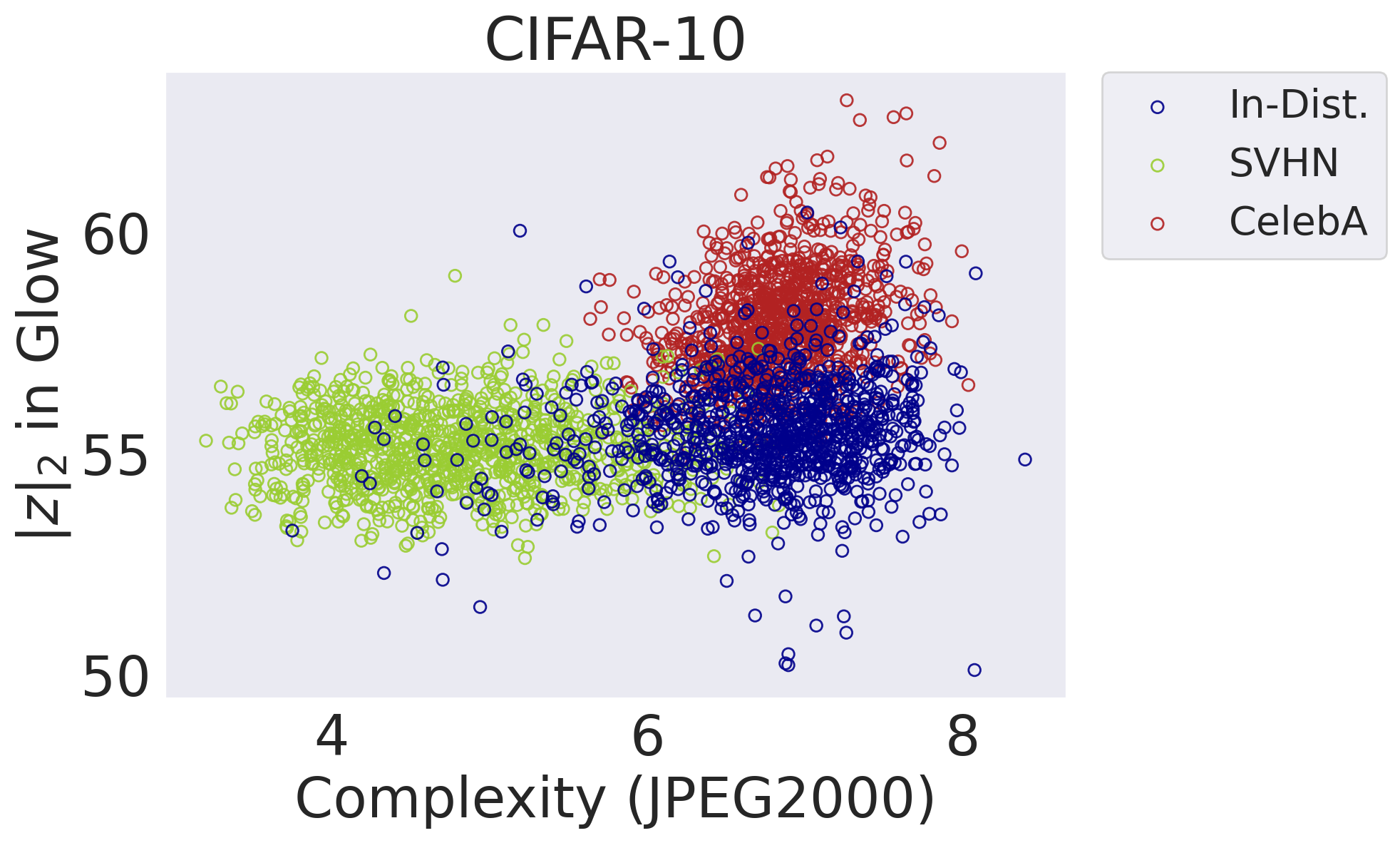}
	\hspace{10px}
	\includegraphics[height=3.15cm, trim={0 0 6cm 1.15cm}, clip]{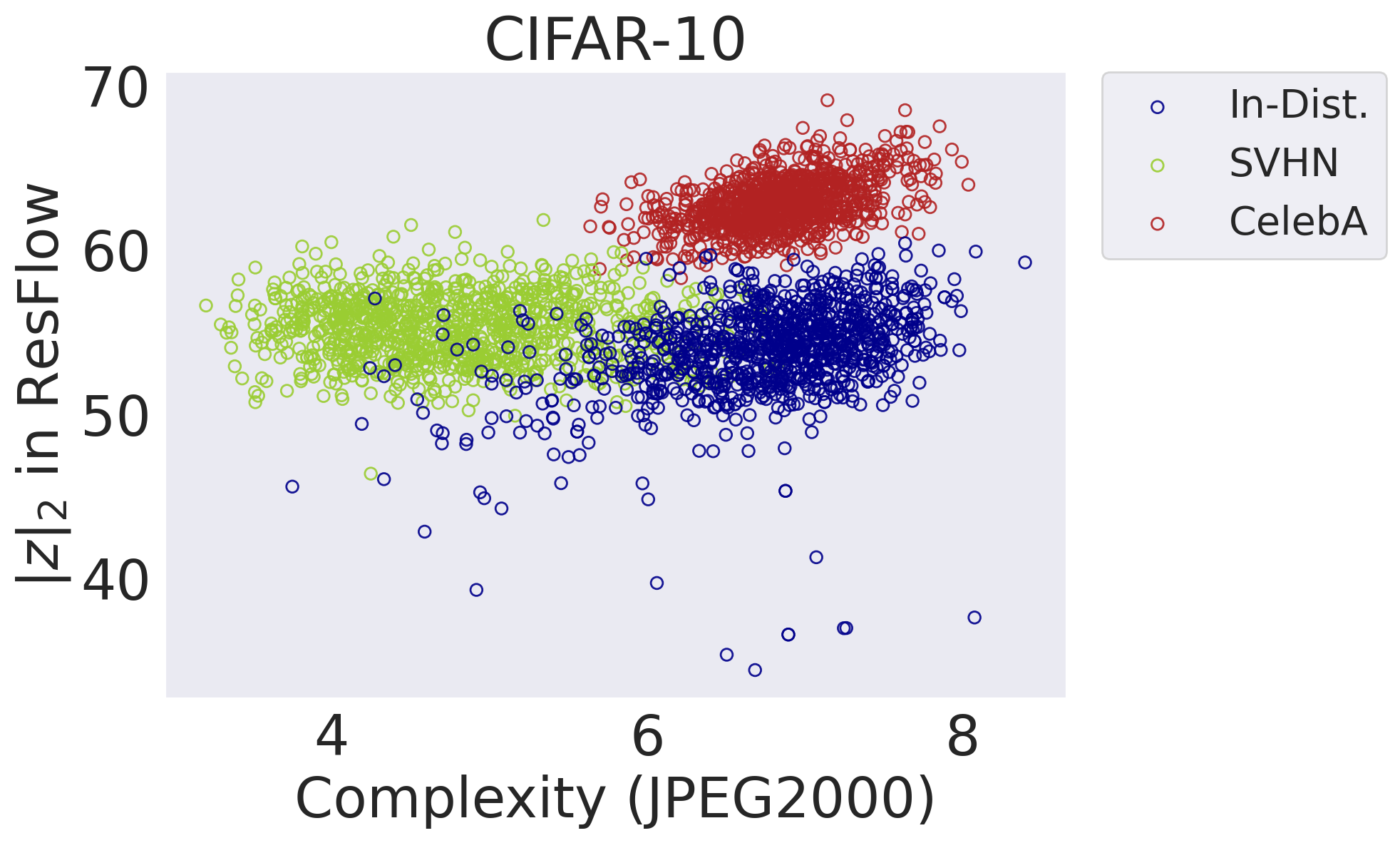}
	\hspace{10px}
	\includegraphics[height=3.15cm, trim={0 0 0 1.15cm}, clip]{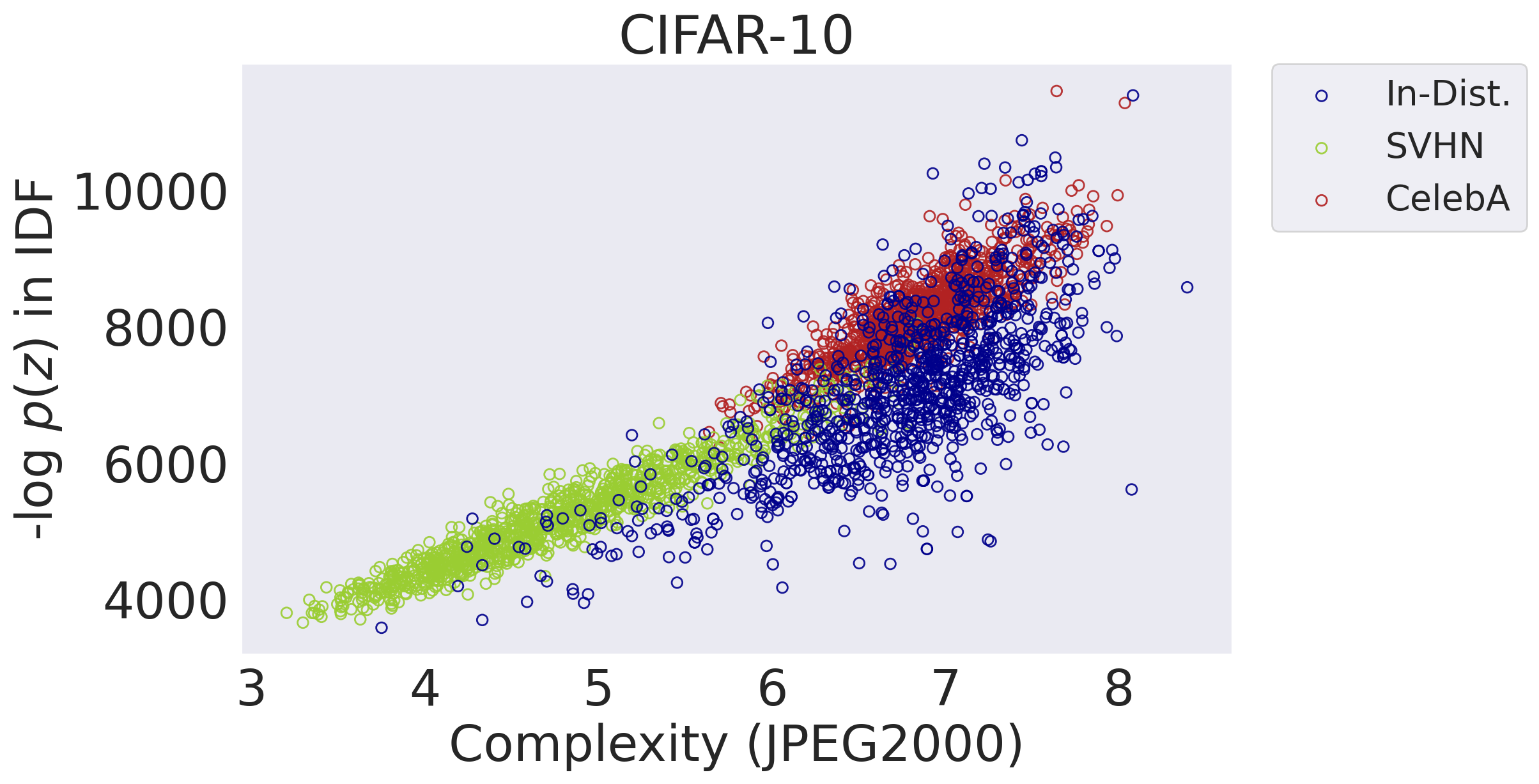}
	\caption{Complexity vs.\ $\norm{{\bf z}} ( \propto - \sqrt{\log p({\bf z})})$ for OOD datasets. Blue is In-Dist (CIFAR-10), green is SVHN, and Red is CelebA. From left to right, Glow, ResFlow, and IDF.
		Glow and ResFlow exhibit more pronounced separation between datasets compared to IDF.
		The GMM is trained to capture the in-distribution in these two-dimensional spaces.}
	\label{fig:comp_vs_z__ood}
\end{figure*}

\begin{table*}[t]
	\centering
				\fontsize{9pt}{10pt}\selectfont
				\begin{tabular}{lrrrrrrrrrrrr}
					\toprule
					&  SVHN &  CelebA &   TIN &   Bed &  Living &  Tower & N-1 & N-2 & N-4 & N-8 & N-16 & N-32 \\
					\midrule
					TTL         & \underline{47.86} &   89.46 & 84.62 & 90.77 &   91.70 &  89.76 &   100 &     \underline{57.32} &     \underline{36.88} &    91.70 &     96.41 &     99.97 \\
					CALT     & 94.08 &    \underline{52.78} &  \underline{36.01} &  \underline{54.28} &   36.12 &  61.35 &    97.74 &   100 &   100 &   100 &    100 &    100 \\
					LRB        &  \underline{54.16} &   65.69 &  \underline{38.33} &  \underline{34.56} &    \underline{32.81} &   \underline{28.59} &      \underline{0.79} &      \underline{0.26} &      \underline{6.47} &     \underline{15.31} &      \underline{14.87} &      \underline{11.81} \\
					LRG     & 75.08 &   94.08 & 81.41 & 67.92 &   61.96 &  68.97 &   100 &    73.51 &    99.21 &    99.23 &     97.51 &     96.37 \\
					WAIC    & 74.78 &    \underline{33.97} & 78.31 & 78.43 &   85.97 &  83.60 &      100 &   100 &      88.3 &    90.79 &     98.08 &     99.60 \\
					\midrule
					Ours (Glow) & 91.47 &   83.15 & 88.59 & 86.97 &   89.25 &  88.42 &   100 &    99.23 &   100 &   100 &    100 &    100 \\
					Ours (ResFlow) & 93.51 &   99.92 & 93.53 & 94.04 &   94.10 &  91.92 &   100 &    99.99 &   100 &    99.96 &    100 &    100 \\
					Ours (iResNet)  & 87.51 &   62.27 & 61.98 & 81.41 &   86.04 &  80.52 &   100 &    99.94 &    10.80 &    99.47 &    100 &    100 \\
					\bottomrule
				\end{tabular}

				\fontsize{9pt}{10pt}\selectfont
				\begin{tabular}{lrrrrrrrrrrrr}
					\toprule
					&  CIFAR-10 &  CelebA &    TIN &    Bed &  Living &  Tower & N-1 & N-2 & N-4 & N-8 & N-16 & N-32 \\
					\midrule
					TTL         & 97.59 &   99.98 &  99.83 &  99.99 &  100 &  99.85 &   100 &   100 &    99.98 &    78.34 &     81.69 &     99.89 \\
					CALT        &   \underline{8.12} &    \underline{4.12} &    \underline{8.80} &   \underline{14.14} &     \underline{9.17} &   \underline{21.58} &   100 &   100 &   100 &   100 &    100 &    100 \\
					LRB         &   \underline{1.45} &     \underline{5.05} &    \underline{7.01} &   \underline{14.10} &     \underline{9.89} &   \underline{14.67} &     \underline{53.86} &      \underline{0.00} &      \underline{0.11} &      \underline{6.29} &      \underline{16.20} &      \underline{11.68} \\
					LRG         & 94.99 &   99.96 & 100 & 100 &  100 & 100 &   100 &   100 &    99.72 &    99.04 &     96.71 &     96.24 \\
					WAIC        & 99.56 &   99.45 &  99.86 &  99.96 &   99.99 &  99.74 &   100 &   100 &   100 &    98.25 &     95.87 &     98.04 \\
					\midrule
					Ours (Glow)  & 98.01 &   99.96 &  99.91 &  99.98 &  100 &  99.86 &   100 &   100 &   100 &   100 &    100 &    100 \\
					Ours (ResFlow) & 65.96 &   63.76 &  85.76 &  68.96 &   84.86 &   \underline{57.04} &   100 &    92.12 &   100 &   100 &    100 &    100 \\
					Ours (iResNet)  & 89.60 &   98.86 &  98.99 &  99.69 &   99.89 &  99.03 &   100 &   100 &    99.97 &    65.61 &     99.16 &     99.22 \\
					\bottomrule
				\end{tabular}
	\caption{AUROC (\%)$\uparrow$. In-Dist datasets are CIFAR-10 (top) and SVHN (bottom).
		Due to space constraints, Noise-$\kappa$ is abbreviated as N-$\kappa$.
		Failure cases (lower than $60\%$) are underlined.}
	\label{tab:AUROC_CIFAR10_SVHN}
\end{table*}

\begin{table*}[t]
	\centering
				\fontsize{9pt}{10pt}\selectfont
				\begin{tabular}{lrrrrrrrrrrr}
					\toprule
					&  SVHN &  CelebA &   TIN &   Bed &  Living &  Tower & N-2 & N-4 & N-8 & N-16 & N-32 \\
					\midrule
					PixelCNN++ & 83.38 &    \underline{14.37} &  \underline{23.20} &  \underline{15.53} &    \underline{10.33} &   \underline{24.62} &         \underline{0.03} &     \underline{36.47} &    98.92 &    100 &    100 \\
					Ours & 95.34 &   75.71 & 74.91 & 76.74 &   78.77 &  80.68 &      100 &   100 &   100 &    100 &    100 \\
					\bottomrule
				\end{tabular}
	%
				\fontsize{9pt}{10pt}\selectfont
				\begin{tabular}{lrrrrrrrrrrr}
					\toprule
					&  CIFAR-10 &  CelebA &    TIN &    Bed &  Living &  Tower & N-2 & N-4 & N-8 & N-16 & N-32 \\
					\midrule
					PixelCNN++ &    \underline{1.86} &     \underline{0.02} &   \underline{0.75} &    \underline{0.18} &     \underline{0.05} &   \underline{0.76} &       \underline{0.00} &      \underline{0.00} &    66.80 &     99.30 &     96.96 \\
					Ours & 91.48 &   96.65 &  96.94 &  98.42 &   99.31 &  97.61 &     100 &   100 &   100 &    100 &    100 \\
					\bottomrule
				\end{tabular}
	\caption{AUROC (\%)$\uparrow$. \bsq{Ours} indicates complexity-aware PixelCNN++. \bsq{PixelCNN++} indicates the likelihood test with PixelCNN++ as a baseline. In-Dist datasets are CIFAR-10 (top) and SVHN (bottom).
		Due to space constraints, Noise-$\kappa$ is abbreviated as N-$\kappa$.
		Failure cases (lower than $60\%$) are underlined. The significant improvements shown here suggest that our hypothesis may be applicable not only to NFs but also to autoregressive models.}
\label{tab:AUROC_pcnnpp}
\end{table*}

\section{OOD Detection with Complexity Awareness}

In the previous section, we have shown that the OOD detection based on the likelihood of Normalizing Flows is untrustworthy.
In this section, we show that this situation can be overcome by exploiting the information on the root cause of the problem, i.e., image complexity.
We aim to further validate Hypothesis~\ref{thm:c} through the presented demonstration.


\subsection{Experiment 4:  Normalizing Flows}
\label{sec:gmm}
We aim to isolate the influence of image complexity $C({\bf x})$ on the likelihood.
To achieve this, we treat $C({\bf x})$ as an independent variable and train a multivariate detection model that takes two variables, $C({\bf x})$ and  $\log p({\bf z})$.
While the volume term $\log |\text{det} \ J_{f}({\bf x})|$ could also be used as an input variable, we choose not to include it due to its dependence on the NF architecture and higher computational cost for iResNet and ResFlow.
Among the various methods available for multivariable detection, we employ one of the simplest models, a Gaussian mixture model (GMM), for demonstration.

\paragraph{NF models.}
We use Glow, iResNet, and ResFlow in this demonstration because in these architectures,
$\log p({\bf z})$ is insensitive to $C({\bf x})$, which we believe is advantageous for distinguishing between different datasets.
For details regarding this architecture selection, refer to Appendices~\ref{sec:arch_for_oodd} and \ref{sec:arch_and_sensitivity}.
Fig.\ \ref{fig:comp_vs_z__ood} visually demonstrates
the superior capability of Glow and ResFlow to separate and differentiate each dataset, outperforming IDF.

\paragraph{Detection with GMM.}
Using the samples in the training portion of the In-Dist datasets, ${\bf x}_{\text{train}}$, we obtain a set of $C({\bf x}_{\text{train}})$ and $\log p(f({\bf x}_{\text{train}}))$ where $f$ is a trained NF model.
Then, we train a GMM $\mathcal{G}$ to capture the distribution of two-dimensional vectors composed of $C({\bf x}_{\text{train}})$ and $\log p({\bf z}_{\text{train}})$.
At the testing phase, given a test sample ${\bf x}_{\text{test}}$, we input $C({\bf x}_{\text{test}})$ and $\log p({\bf z}_{\text{test}})$ into the trained $\mathcal{G}$,
obtaining the likelihood score $S_{\text{GMM}}({\bf x}_{\text{test}})$ from $\mathcal{G}$.
A larger $S_{\text{GMM}}({\bf x}_{\text{test}})$ indicates a higher likelihood that ${\bf x}_{\text{test}}$ belongs to the In-Dist dataset.
For more details, see Appendix \ref{sec:gmm_detail}.

\paragraph{Datasets.}
When the In-Dist dataset is CIFAR-10 and SVHN,
we use CelebA \cite{liu2015faceattributes}, TinyImageNet (TIN) \cite{ILSVRC15},  LSUN \cite{journals/corr/YuZSSX15}, and the pooling noise images (Section~\ref{sec:exp1}) as OOD datasets.
For LSUN, which comprises various scene categories, we select the Bedroom (\emph{Bed}), Living room (\emph{Living}), and \emph{Tower} categories, treating each of them as individual OOD datasets.
When the In-Dist dataset is MNIST and FMNIST, we use OMNIGLOT \cite{lake2015human} and NotMNIST \cite{bulatov2011notmnist} as OOD datasets.
In experiments using ImageNet of size $224 \times 224$ as the In-Dist dataset,
we use the pooling noise images as OOD datasets.

\paragraph{Evaluation metrics.}
We evaluate our method using two standard metrics in OOD detection literature: the area under the receiver operating characteristic curve (AUROC) and the area under the precision-recall curve (AUPR).
These metrics provide an overall assessment of performance by varying the detection threshold.
Higher values for both metrics indicate better performance.
Given that the chance level for AUROC is 50\%, we set a minimum threshold of 60\%  to evaluate detectability.
Our primary focus is on the method's capability to detect any type of OODs.
Thus, we evaluate the methods based on whether the AUROC scores exceed 60\% for each OOD dataset used in the evaluation.
In the tables, scores below 60\% are underlined to indicate that they do not meet the detectability threshold.

\paragraph{Competitors.}
In addition to TTL and CALT, we compare our method to three other existing NF-based methods on CIFAR-10 and SVHN:
the Watanabe-Akaike Information Criterion (WAIC) \cite{choi2019generative}, the likelihood-ratio to background model (LRB) \cite{NEURIPS2019_1e795968}, and the likelihood-ratio to general model (LRG) \cite{NEURIPS2020_f106b7f9}.
Appendix \ref{sec:comparisons_detail} provides detailed descriptions of each method.

\paragraph{Results on MNIST and FMNIST.}
The AUROC and  AUPR are shown in Appendix \ref{sec:extra_exp_2} (Tables \ref{tab:AUROC_MNIST_FMNIST} and \ref{tab:AUPR_MNIST_FMNIST}).
Our methods detected all cases, with Glow achieving the highest performance, followed by ResFlow and iResNet.
\paragraph{Results on CIFAR-10 and SVHN.}
The AUROC and  AUPR are shown in Tables \ref{tab:AUROC_CIFAR10_SVHN} and \ref{tab:AUPR_CIFAR10_SVHN} in Appendix \ref{sec:extra_exp_2}, respectively.
Once again, our method with Glow demonstrated the most superior performance overall.
Our method with ResFlow performed best on CIFAR-10 but showed weaker performance on SVHN, scoring below 60\% on Tower.
Among the other methods, only LRG scored higher than 60\% in all test cases, excluding our methods,
our speculation on which is included in Appendix \ref{sec:aupr_C10_SVHN}.

\paragraph{Results on ImageNet.}
The AUROC and  AUPR are shown in Table~\ref{tab:ret_IMAGENET} in Appendix~\ref{sec:extra_exp_2}.
Our method achieved detection accuracy of 100\% or close to it across all evaluated pooling noise images,
validating our claims for large images.

\paragraph{Benefits of using two variables}
are visually shown in Fig.~\ref{fig:comp_vs_z__ood} (left and center).
Relying solely on $C({\bf x})$ results in the inability to differentiate between CIFAR-10 (In-Dist.) and CelebA.
Similarly, using only $\log p({\bf z})$ fails to differentiate between CIFAR-10 (In-Dist.) and SVHN.
Effective separation is achieved by utilizing both variables.

%
%
%

\section{Applicability to Autoregressive Model}
\label{sec:pcnnpp}
We now shift our focus from Normalizing Flows (NFs) to Autoregressive (AR) models, which are another type of DGM where instances of the failure of the likelihood test have also been observed~\cite{nalisnick2018do}.
While our Hypothesis~\ref{thm:c} is developed based on a function invertible between latent space and data space,
AR models are designed without the concept of latent space.
However, a recent interpretation has proposed considering AR models as a single-layer NF,
thus implying the presence of an implicit latent space~\cite{NEURIPS2020_26ed695e}.
In this interpretation, the latent distribution of PixelCNN++ is a discretized mixture of logistics.
With this perspective, we argue that
our hypothesis is also applicable to AR models and can provide an explanation for the failure of the likelihood tests with AR models.
To substantiate this claim, we present the following experimental results.

\subsection{Experiment 5: PixelCNN++ }
First, we present the results of the same experiments using GMM, previously done, but this time applied to PixelCNN++.
Similar to CV-Glow and IDF, the volume term in PixelCNN++ is fixed (zero) as $\log p({\bf x}) = \log p({\bf z})$, so the GMM was constructed on $\log p({\bf x})$ and $C({\bf x})$.
The AUROC results are provided in Table \ref{tab:AUROC_pcnnpp}, while the AUPR results are available in Table \ref{tab:AUPR_pcnnpp} in Appendix \ref{sec:aupr_C10_SVHN}.
The results clearly show that our complexity-aware method significantly improves the OOD detection performance for PixelCNN++ as well.
These outcomes suggest the existence of the DCAS described in Remark~\ref{rem:dcas} in PixelCNN++.

Second, as demonstrated in Appendix \ref{sec:extra_exp_3}, the response of $\log p({\bf x})$ (or $\log p({\bf z})$) in PixelCNN++ to image complexity closely resembles that observed in NFs with a fixed volume architecture, i.e., CV-Glow and IDF.
This finding suggests that the discussion pertaining to Remark~\ref{rem:volume} (Appendix~\ref{sec:arch_and_sensitivity}) is also applicable to PixelCNN++.
While further theoretical verification may be necessary to conclusively establish the applicability of Hypothesis~\ref{thm:c} to AR models, our results provide plausible evidence supporting this assertion.

\section{Limitation}
The OOD detection task we addressed was aimed to discern differences between datasets, such as CIFAR-10 and SVHN.
More challenging tasks such as semantic OOD detection — aimed at identifying objects not present in the training data  —  have remained unexplored in this study.
This is the limitation of this paper and is an avenue left for future work.

\section{Conclusion}
\label{sec:conclusion}
In this study, we  have proposed a hypothesis that explains the failure of OOD detection methods based on the likelihood of Normalizing Flows and Autoregressive models and delineates how the likelihood in those models is affected by varying image complexity.
We believe that the findings presented in this paper will contribute to the future advancement of DGM applications, particularly OOD detection.

\bibliographystyle{unsrt} 
\bibliography{my_bib_230808}

	\appendix
	\section{Additional Related Works}
\label{sec:related_additional}

\paragraph{DGM-based approaches.}
\cite{pmlr-v139-zhang21g} investigated the failures of the likelihood test and the typicality test for OOD detection.
They argued that without specifying the type of OOD to be detected, it is impossible to rely on a single statistical test, even if it is based on the likelihood.
They also discussed the failure of the typicality test when the out-distribution overlaps with the support of the in-distribution.
In \cite{NEURIPS2020_66121d1f}, the failure of the typicality test on latent space of a VAE~\cite{kingma2013auto} was attributed to the fact that its latent space did not form a perfect Gaussian.
While these studies demonstrated the failures of the likelihood test and the typicality test, our work takes a step further by providing an explanation about \emph{why} they fail.
Importantly, we show that the underlying causes of both testing methods' failures are fundamentally the same.

Several studies have highlighted that DGMs exhibit a sensitivity to the \emph{smoothness} of images,
which can lead to the failure of likelihood test.
The term  \emph{smooth} refers to regions where adjacent pixels have constant or very close values.
For example, \cite{NEURIPS2019_1e795968} demonstrated in experiments on MNIST and FMNIST that backgrounds with black pixels were assigned high likelihoods.
Similarly, \cite{nalisnick2018do} showed that converting RGB images to black and white increased their likelihoods.
The influence of smoothness on likelihood was also observed by \cite{NEURIPS2020_f106b7f9},
which attributed it to the nature of convolutional networks.
Additionally, \cite{NEURIPS2020_ecb9fe2f} focused on NFs and argued that the coupling layer structure in NFs makes them more predictable in smooth regions, and thus images containing many smooth regions will have higher likelihoods.
Moreover, \cite{cai2023frequency} proposed a method that emphasizes high-frequency features associated with object contours.
The concept of  \emph{less complexity} mentioned in this paper aligns with the notion of smoothness discussed in these studies.
However, we emphasize that the cause of irrationally increased likelihoods is not specific to a particular network structure but rather a density concentration in the latent space.

\section{Existing Methods}

\subsection{Typical Set and Typicality Test}
\label{sec:typical_set}
The concept of the \emph{typical set} refers to a set or region that contains a significant amount of probability mass in a given distribution.
It is a fundamental concept in probability theory (see~\cite{nalisnick2020detecting,cover2012elements} for formal definitions).
In the context of OOD detection, \cite{choi2019generative,nalisnick2020detecting}
argued that the failure of the likelihood test can be attributed to ignorance of the notion of the typical set.
While samples drawn from a DGM are highly likely to come from its typical set, in high-dimensional spaces, the typical set may not necessarily coincide with high-density or high-likelihood regions.
This phenomenon becomes more pronounced as the dimensionality increases.
While specifying the region of the typical set for arbitrary distributions is difficult, it is possible for an isotropic Gaussian distribution.
It is well known that the typical set of samples from a $d$-dimensional Gaussian $\mathcal{N}(0,  \textbf{I}_{\text{d}})$ resides in a hypersphere with a radius of $\sqrt{d}$ with high probability.
Specifically, for any $\epsilon \in (0,1)$, we have
\begin{eqnarray}
\mathbb{P} \left(  \sqrt{ d(1 - \epsilon)} < \norm{{\bf z}}  < \sqrt{ d(1 + \epsilon)} \right)  \geq  1 - 2 \text{exp} \left( - \frac{d \epsilon^{2}}{8}\right).
\label{eq:tailbound}
\end{eqnarray}
We refer the reader to \cite{osadaChernoffTailBound}, for example, for the derivation.
As the dimension $d$ increases, the radius of the hypersphere also increases, causing the region where the typical set (i.e., In-Dist samples) resides to move away from the region of highest likelihood, which corresponds to the mean of the Gaussian distribution.
Consequently, In-Dist examples are often assigned low likelihoods in high-dimensional spaces.
To address this issue, \cite{choi2019generative,nalisnick2020detecting} proposed a method, the typicality test in latent space (TTL).
The TTL utilizes NFs with latent distribution of $\mathcal{N}(0, \textbf{I}_{\text{d}})$ to obtain the latent vector ${\bf z}$ corresponding to a test input  ${\bf x}$.
It then computes the score  $S_{\text{TTL}} ({\bf x})= {\tt abs} (\norm{ {\bf z}} - \sqrt{d})$, where $\lVert \cdot \rVert$ denotes the $L_2$ norm.
If ${\bf x}$ falls outside the typical set, the score $S_{\text{TTL}}({\bf x})$ will be large, indicating that ${\bf x}$ is likely to be an OOD input.
We present our experimental results on the TTL in Appendix \ref{sec:calt_dz}.

\subsection{Our Experiments}
\label{sec:calt_dz}

\begin{figure*}[h]
\centering
\includegraphics[width=0.4 \textwidth]{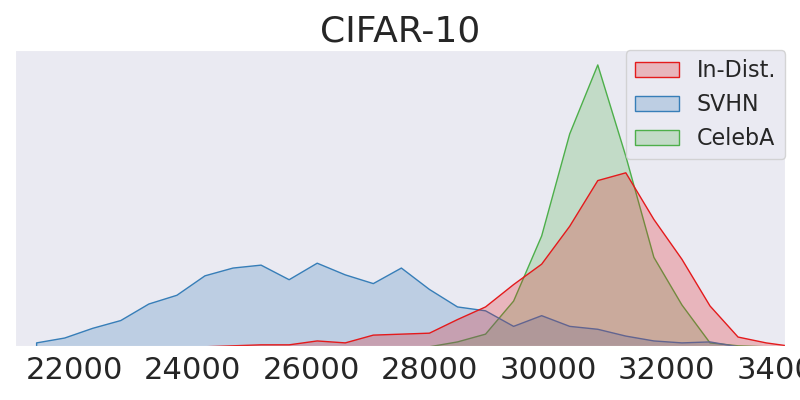}
\hspace{15px}
\includegraphics[width=0.4 \textwidth]{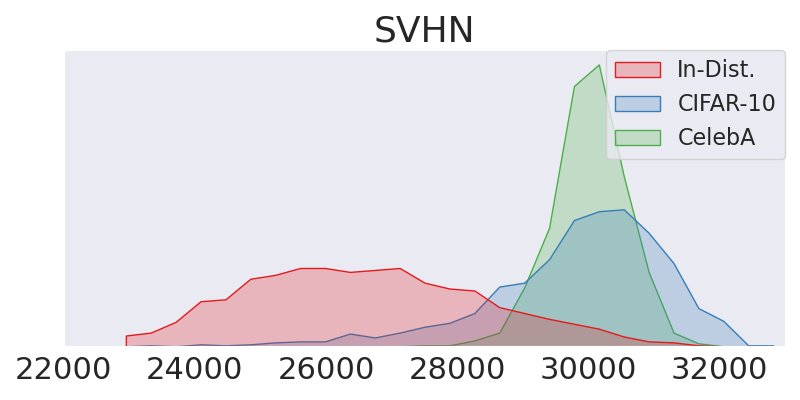}
\\
\includegraphics[width=0.4 \textwidth]{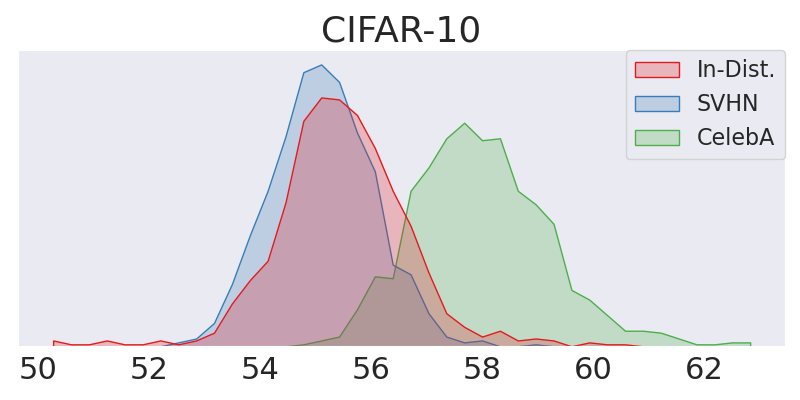}
\hspace{15px}
\includegraphics[width=0.4 \textwidth]{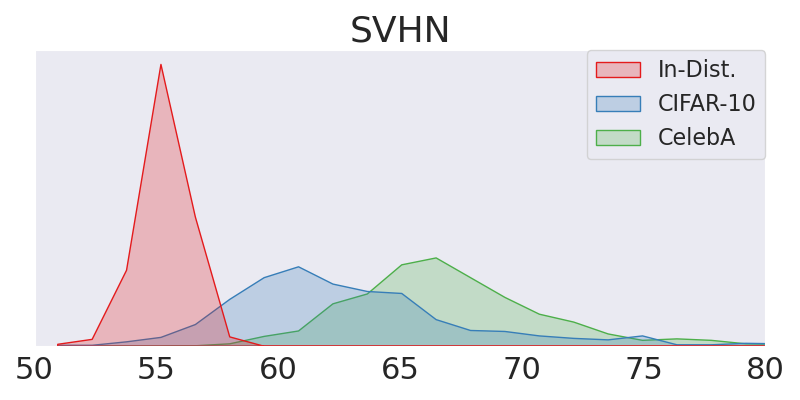}
\caption{Existing methods fail in specific combinations. The cases where In-Dist is CIFAR-10 are shown in the left columns, and the cases where In-Dist is SVHN are shown in the right columns. Top row: histograms of CALT. The $x$-axis is $S_{\text{CALT}}({\bf x}) = \log p({\bf x}) +C({\bf x})$.
	Bottom row: histograms of TTL. The $x$-axis is $\norm{{\bf z}} ( \propto - \sqrt{\log p({\bf z})})$.
	CALT fails to detect CIFAR-10 samples as OOD when SVHN is  In-Dist (top-right).
	TTL cannot identify SVHN samples as OOD when CIFAR-10 is  In-Dist (bottom-left).
}
\label{fig:calt_dz}
\end{figure*}

\begin{figure*}[h]
\centering
\includegraphics[width=0.48 \textwidth]{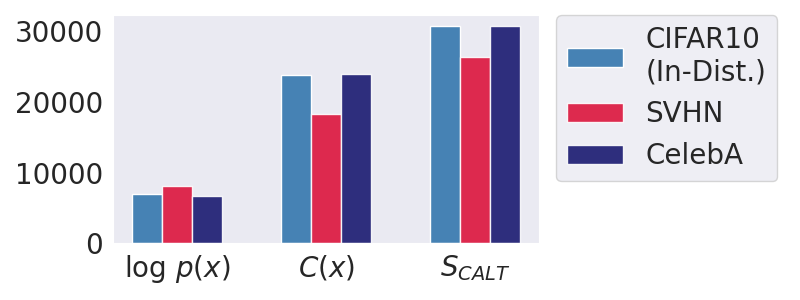}
\includegraphics[width=0.48 \textwidth]{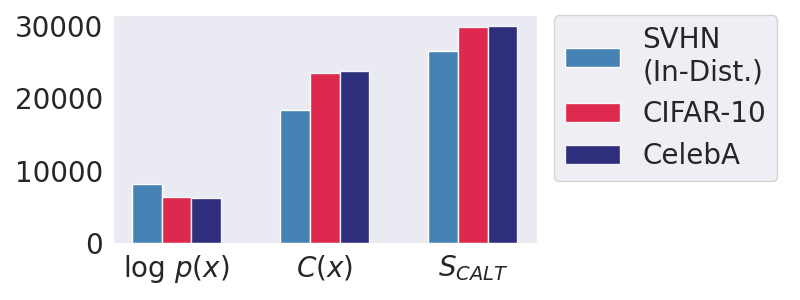}
\caption{Decomposing $S_{\text{CALT}}({\bf x})$ into $\log p({\bf x})$ and $C({\bf x})$ on CIFAR-10 (left) and  SVHN (right). }
\label{fig:calt_detail}
\end{figure*}

We experimented two existing OOD detection methods, the complexity aware likelihood test (CALT) and the typicality test in latent space (TTL).
The experiments were performed with Glow on SVHN and CIFAR-10, and we also used CelebA as the OOD datasets with resizing to $32 \time 32$.
The test samples are the same we used in our experiments.
Fig.\ \ref{fig:calt_dz} (top two) shows the histogram of the score that CALT bases on, and Fig.\ \ref{fig:calt_detail} shows its components.
It shows that CALT fails to detect CIFAR-10 samples as OOD when SVHN is  In-Dist (top-right).
Fig.\ \ref{fig:calt_dz} (bottom two) shows the histogram of $\norm{ {\bf z}}$ that the TTL bases on.
It shows that TTL cannot identify SVHN samples as OOD when CIFAR-10 is  In-Dist (bottom-left).

\section{Derivation of Hypothesis~\ref{thm:c}}
We establish Hypothesis~\ref{thm:c}.
In the derivation, we consider generating samples ${\bf x}$ in the data space $\mathcal{X}$ from a small region $\mathcal{A}$ in the latent space $\mathcal{Z}$.
In the derivation provided in Appendix~\ref{sec:proof}, we assume that the set of ${\bf x}$ generated from $\mathcal{A} \subset \mathcal{Z}$ follows a Gaussian distribution.
A more generalized scenario, without the Gaussian assumption, is derived from experimental observations in  Appendix~\ref{sec:derive_by_exp} .

\subsection{Analysis}
\label{sec:proof}
Introducing two lemmas, we derive Observation~\ref{obs:theoretical_derived_conjecture}, which aligns with Hypothesis~\ref{thm:c}.
We assume that the probability density function is Gaussian in a subdomain $f^{-1}(\mathcal{A}) \subset \mathcal{X}$.
The invertible function (practically, we assume Normalizing Flow) $f: \mathcal{X} \rightarrow \mathcal{Z}$ is locally $L_{\mathcal{A}}$-Lipschitz with respect to a small region $\mathcal{A} \subset \mathcal{Z}$.
We note that in a different region $\mathcal{A}'$, the Lipschitz coefficient of $f$ differs and becomes $L_{\mathcal{A}'}$.
In other words, the local Lipschitz coefficient of $f$ depends on which region in $ \mathcal{Z}$ the sample ${\bf x}$ is mapped to.
The relationship assumed in Lemmas~\ref{thm:var} and \ref{lem:ent_to_mse} is illustrated in Fig.~\ref{fig:thm2}.

\begin{figure*}[t]
\centering
\includegraphics[width=0.45 \textwidth]{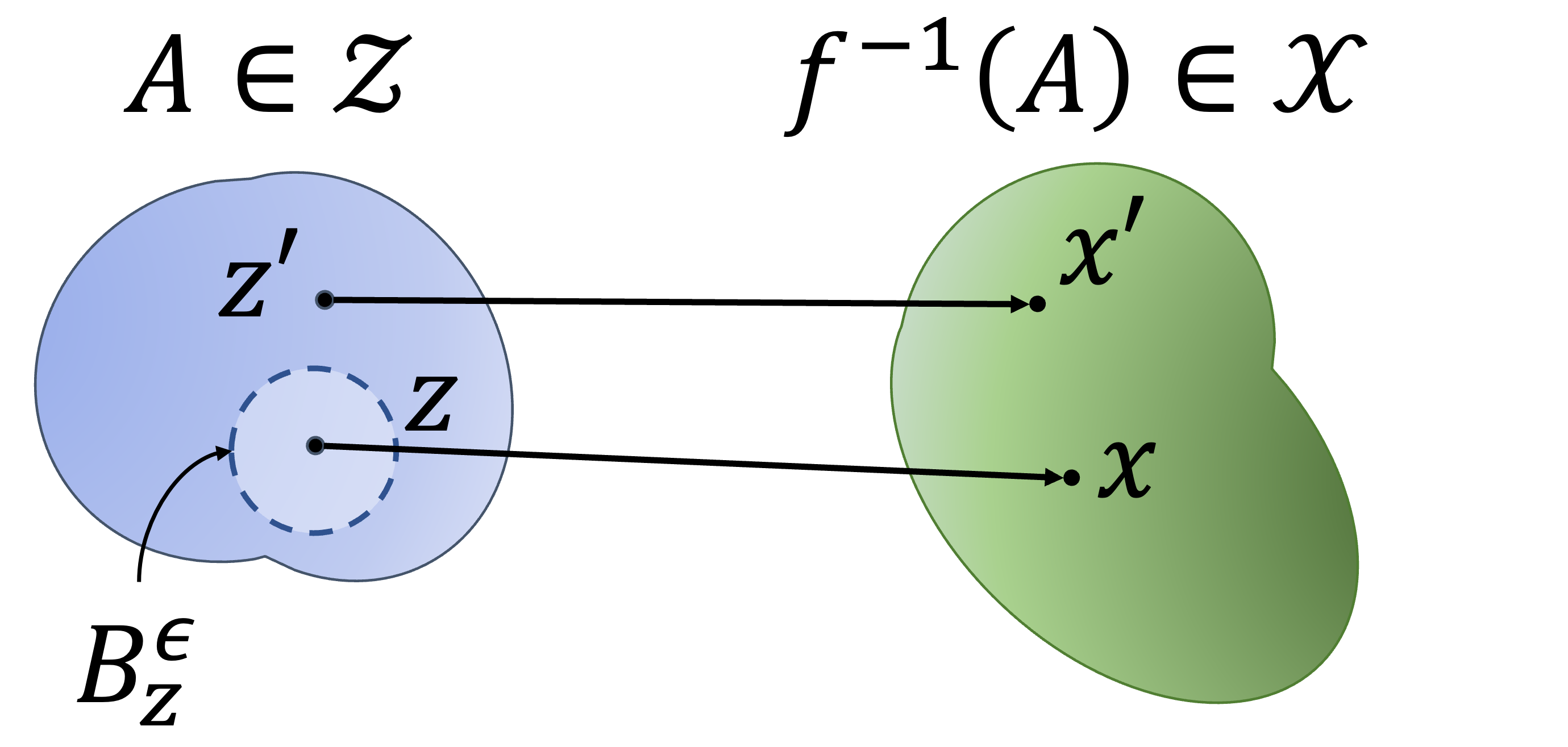}
\caption{Illustration of the relationship in Lemmas~\ref{thm:var} and \ref{lem:ent_to_mse}.
	In Lemma~\ref{lem:ent_to_mse}, the green region $f^{-1}(\mathcal{A}) \subset \mathcal{X}$ is assumed to follow a Gaussian distribution with a mean ${\bf x} = \EX [ f^{-1}({\bf z}')]$.
	Assuming that the latent space $\mathcal{Z}$ is semantically continuous,
	when ${\bf x}'$ is generated from a sufficiently small $\mathcal{A}$,
	they share common semantics, and thus the image complexities $C({\bf x}')$ and $C({\bf x})$  can be considered nearly equal.
}
\label{fig:thm2}
\end{figure*}

\begin{lemma}
\label{thm:var}
Let $f: \mathcal{X} \rightarrow \mathcal{Z}$ be an invertible function being locally $L_{\mathcal{A}}$-Lipschitz for $\mathcal{A} \subset \mathcal{Z}$.
For all ${\bf z}'  \in \mathcal{A}$, let ${\bf x}' =  f^{-1}({\bf z}')$, ${\bf z} = \EX {\bf z}' = [\EX z_{1}', \dots, \EX z_{d}']^{T}$, and ${\bf x} = f^{-1}( {\bf z})$.
Assume  ${\bf z}  \in \mathcal{A}$, then ${\bf x} \in f^{-1}(\mathcal{A})$.
Also, let $\mathcal{B}^{\epsilon}_{{\bf z}} = \{ {\bf z}' \in \mathcal{A}: \norm{ {\bf z}' - {\bf z} }  < \epsilon \}$ and $\mathcal{B}^{\epsilon}_{{\bf z}} \subset \mathcal{A}$ for a constant $\epsilon \geq 0$.
Then, we have
\begin{eqnarray}
	\label{eq:main_var}
	\frac{\epsilon^{2}}{{L_{\mathcal{A}}}^{2}} \left(  1 - \mathbb{P} ( \mathcal{B}^{\epsilon}_{{\bf z}	} ) \right)
	\leq \EX_{ {\bf x}' \sim  f^{-1}(\mathcal{A})} \norm{ {\bf x}'  - {\bf x}}^{2}.
\end{eqnarray}

\end{lemma}
\begin{proof}
Since $f$ is (locally) $L_{\mathcal{A}}$-Lipschitz, we know $ \norm{ {\bf z}' - {\bf z}} = \norm{ f( {\bf x}' ) - f({\bf x})} \leq  L_{\mathcal{A}} \norm{ {\bf x}' - {\bf x}} $.
Then, we have
\begin{eqnarray}
	\mathbb{P} \left( \norm{ {\bf z}' - {\bf z}} \geq \epsilon \right) &\leq& \mathbb{P} \left( L_{\mathcal{A}}  \norm{ {\bf x}' - {\bf x}}  \geq \epsilon \right) \\
	&=& \mathbb{P} \left(   \norm{ {\bf x}' - {\bf x}} \geq \frac{\epsilon}{L_{\mathcal{A}}} \right) \\
	&=& \mathbb{P} \left(   \norm{ {\bf x}' - {\bf x}}^{2} \geq \frac{\epsilon^{2} }{{L_{\mathcal{A}}}{^2}} \right)  \label{eq:before_markov}\\
	&\leq &	\frac{ {L_{\mathcal{A}}}^{2} }{\epsilon^{2}} \EX_{ {\bf x}' \sim  f^{-1}(\mathcal{A})} \norm{ {\bf x}' - {\bf x}}^{2} .
\end{eqnarray}
We applied Markov's inequality to Eq.\ \eqref{eq:before_markov}.
Since $ \mathbb{P} \left( \norm{ {\bf z}' - {\bf z} } \geq \epsilon \right) = 1 - \mathbb{P} ( \mathcal{B}^{\epsilon}_{{\bf z} } ) $, the statement therefore follows.
\end{proof}

\begin{lemma}
\label{lem:ent_to_mse}
\cHiro{The mean square error (MSE) for $d$-dimensional random vectors sampled from Gaussian distribution with mean vector ${\bf x}$ and isotropic covariance matrix $\Sigma_{{\bf x}} =  \sigma^{2} \mathbb{I}_{\text{d}}$,  $\mathcal{N}({\bf x},  \Sigma_{{\bf x}})$, can be expressed using entropy $H_{\Sigma_{{\bf x}}} = \EX_{ {\bf x}' \sim \mathcal{N}({\bf x},  \Sigma_{{\bf x}} )} [- \log p({\bf x}')]$ as follows:
	\begin{eqnarray}
		\EX_{ {\bf x}' \sim \mathcal{N}({\bf x},  \Sigma_{{\bf x}} )}  \norm{ {\bf x}' - {\bf x} }^2
		= \frac{d}{ 2 \pi e } \exp \left( \frac{2}{d} H_{\Sigma_{{\bf x}}} \right).
		\label{eq:ent_to_mse}
	\end{eqnarray}
}
\end{lemma}

\begin{proof}
The MSE can be expressed as the trace of covariance matrix as follows:
\begin{eqnarray}
	&& \EX_{ {\bf x}' \sim \mathcal{N}({\bf x},  \Sigma_{{\bf x}} )}   \norm{ {\bf x}' - {\bf x} }^2 \\
	&=& \EX_{ {\bf x}' \sim \mathcal{N}({\bf x},  \Sigma_{{\bf x}} )} [ ( {\bf x}' - {\bf x} )^{T} ( {\bf x}' - {\bf x} ) ] \\
	&=& \EX_{ {\bf x}' \sim \mathcal{N}({\bf x},  \Sigma_{{\bf x}} )} \left [ \Tr{( ( {\bf x}' - {\bf x} ) ( {\bf x}' - {\bf x}  )^{T} )} \right] \\
	&=& \Tr{\left ( \EX_{ {\bf x}' \sim \mathcal{N}({\bf x},  \Sigma_{{\bf x}} )} [ ( {\bf x}' - {\bf x} ) ( {\bf x}' - {\bf x} )^{T} ] \right)}\\
	&=& \Tr{\left ( \EX_{ {\bf x}' \sim \mathcal{N}({\bf x},  \Sigma_{{\bf x}} )} [ ({\bf x}' - \EX {\bf x}') ({\bf x}' - \EX {\bf x}')^{T} ] \right)}\\
	&=& \Tr{(\Sigma_{{\bf x}} )}.
\end{eqnarray}
Since we assumed $\Sigma_{{\bf x}} = \sigma^{2} \mathbb{I}_{\text{d}}$,
equality holds in Hadamard's inequality for $\Sigma_{{\bf x}}$ as
\begin{eqnarray}
	\frac{1}{d} \Tr{(\Sigma_{{\bf x}} )}  = \left( \det  \Sigma_{{\bf x}}  \right)^{\frac{1}{d}}.
\end{eqnarray}
Also, it is well known that entropy of multivariate Gaussian is
\begin{eqnarray}
	H_{\Sigma_{{\bf x}}}
	&:=& \EX_{ {\bf x}' \sim \mathcal{N}({\bf x},  \Sigma_{{\bf x}} )} [- \log p({\bf x}')] \\
	&=& \frac{d}{2} ( 1 + \log (2 \pi )) +  \frac{1}{2} \log  \det  \Sigma_{{\bf x}}  \\
	&=&  \frac{d}{2} \log (2 \pi e    \left( \det  \Sigma_{{\bf x}}  \right)^{\frac{1}{d}} ).
	\label{eq:ent}
\end{eqnarray}
By replacing  $\left( \det  \Sigma_{{\bf x}}  \right)^{\frac{1}{d}}$ with $\frac{1}{d} \EX   \norm{ {\bf x}' - {\bf x}  }^2 $, we have
\begin{eqnarray}
	H_{\Sigma_{{\bf x}}} = \frac{d}{2} \log (  \frac{2 \pi e}{d} \EX   \norm{  {\bf x}' - {\bf x} }^2)
\end{eqnarray}
and thus we obtain the statement.
\end{proof}

\paragraph{Discretizing continuous data.}
We aim to establish a connection between Lemma 1 and 2 and the notion of image complexity.
In Definition~\ref{def:comp}, image complexity for an image ${\bf x} \in \{0, 1, \ldots, 255 \}^{d}$ is defined using the code length, denoted as $L ({\bf x})$, as $C({\bf x}) = \frac{1}{d}  L ({\bf x})$, and thus we aim to connect Lemma 1 and 2 to $L ({\bf x})$.
However, in Lemma 1 and 2, ${\bf x}$ is a continuous variable, and thus establishing a direct connection with the code length is unfeasible.
Thus, we discretize ${\bf x}$ as $\bar{{\bf x}}$ and code $\bar{{\bf x}}$ with a certain probability mass function (PMF), $P(\bar{{\bf x}})$.
This approach is detailed in \cite[section~3.1]{NEURIPS2019_f6e794a7} and \cite[section~8.3]{cover2012elements}.
We divide $\mathbb{R}^{d}$ into bins of volume $\delta_{x} := 2^{-kd}$ with precision $k$.
When the probability density function $p({\bf x})$ is sufficiently smooth, the PMF can be approximated as
\begin{eqnarray}
\label{eq:pdf_to_pmf}
P(\bar{{\bf x}}) \approx p(\bar{{\bf x}}) \delta_{x}.
\end{eqnarray}
For instance, with $k=8$, a continuous data ${\bf x}$ is discretized as $\bar{{\bf x}} \in \{0,1, \ldots, 255\}^d$.
We utilize this approach in the following.

\begin{observation}
\label{obs:mse_to_comp_by_shannon}
The MSE exhibits the following relationship with the image complexity of discretized version of the mean of ${\bf x}'$, represented as $C(\bar{{\bf x}})$, with the volume of bins used in discretization, $\delta_{x}$:
\begin{eqnarray}
	\EX_{ {\bf x}' \sim \mathcal{N}({\bf x},  \Sigma_{{\bf x}} )}  \norm{ {\bf x}' - {\bf x} }^2
	\leq \frac{d}{ 2 \pi e } \delta_{x}^{\frac{2}{d} } \exp \left( 2 C(\bar{{\bf x}})   \right).
	\label{eq:mse_to_comp}
\end{eqnarray}
\end{observation}
Letting $\bar{{\bf x}}'$ represent a discretized version of ${\bf x}' \in \mathbb{R}^{d}$,
we approximate the entropy $H_{\Sigma_{{\bf x}}}$ as follows:
\begin{eqnarray}
H_{\Sigma_{{\bf x}}}
&=&
\EX_{ {\bf x}' \sim \mathcal{N}({\bf x},  \Sigma_{{\bf x}} )} [- \log p({\bf x}')] \label{eq:ent_in_obs}\\
& \approx &
\EX_{ \bar{{\bf x}}' } [- \log p(\bar{{\bf x}}') ]  \label{eq:ent_w_descrete_x}\\
& \approx &
\EX_{ \bar{{\bf x}}' } [- \log \frac{ P(  \bar{{\bf x}}' )}{ \delta_{x}} ] \\
& = &
\EX_{ \bar{{\bf x}}' } [- \log P(  \bar{{\bf x}}' )] +  \log  \delta_{x}.
\end{eqnarray}
We applied Eq.\ \eqref{eq:pdf_to_pmf} to Eq.\ \eqref{eq:ent_w_descrete_x}.
Applying Shannon's source coding theorem~\cite{shannon1948mathematical} and the definition of image complexity (Definition~\ref{def:comp}), we have
\begin{eqnarray}
\EX_{ \bar{{\bf x}}' } [- \log P(  \bar{{\bf x}}' )]
&\leq&
\EX_{ \bar{{\bf x}}' } [L ( \bar{{\bf x}}' )]  \label{eq:l_in_obs}\\
&=&
\EX_{  \bar{{\bf x}}' } [d C ( \bar{{\bf x}}' )].
\end{eqnarray}
Then, we employ Assumption~\ref{asm:semantic_continuous}, which posits that latent space $\mathcal{Z}$ is semantically continuous.
When considering $\mathcal{A} \subset \mathcal{Z}$ as a very small region, generated samples ${\bf x}, {\bf x}' \in f^{-1}(\mathcal{A})$ share common semantics.
Consequently, we can approximate the image complexities of the discretized version of ${\bf x}$ and ${\bf x}'$ as equivalent, i.e., $C ( \bar{{\bf x}}) \approx C (\bar{{\bf x}}')$, and thus we have $\EX_{  \bar{{\bf x}}' } [C ( \bar{{\bf x}}' )] \approx C ( \bar{{\bf x}})$.
Putting them together, we have  $H_{\Sigma_{{\bf x}}}  \leq d C(\bar{{\bf x}}) + \log \delta_{x}$.
By substituting this inequality to Eq.\ \eqref{eq:ent_to_mse}, we have
$\EX_{ {\bf x}' \sim \mathcal{N}({\bf x},  \Sigma_{{\bf x}} )}  \norm{ {\bf x}' - {\bf x} }^2
\leq
\frac{d}{ 2 \pi e } \exp \left( \frac{2}{d} \left(d C(\bar{{\bf x}}) + \log \delta_{x} \right) \right)
$
and obtain Eq.\ \eqref{eq:mse_to_comp}.

\begin{observation}
\label{obs:theoretical_derived_conjecture}
If $f^{-1}(\mathcal{A})$ follows Gaussian $\mathcal{N}({\bf x},  \Sigma_{{\bf x}} )$, we have
\begin{eqnarray}
	\frac{\epsilon^{2}}{ {L_{\mathcal{A}} }^{2}} \left(  1 - \mathbb{P} ( \mathcal{B}^{\epsilon}_{{\bf z}	} ) \right)
	\leq
	\frac{d}{ 2 \pi e } \delta_{x}^{\frac{2}{d} } \exp \left( 2 C(\bar{{\bf x}})   \right).
\end{eqnarray}
\end{observation}
Combining Eq.\ \eqref{eq:mse_to_comp} to Eq.\ \eqref{eq:main_var} yields the above.

In the following section, we will verify whether the same relationship holds true, even when we remove the assumption of Gaussian distribution, through experimentation.

\subsection{Experimental Observation}
\label{sec:derive_by_exp}
In the previous section \ref{sec:proof}, Eq.\ \eqref{eq:mse_to_comp} holds under the constraint that samples ${\bf x}$ generated from $\mathcal{A} \subset \mathcal{Z}$ follow a Gaussian distribution.
In the following, we experimentally demonstrate that the inequality between the MSE and the image complexity holds true even when the constraint is relaxed.

The details of the experiments are presented in Appendix~\ref{sec:comp_vs_var}.
We generate samples from the latent space $\mathcal{Z}$ of a trained Glow and use them as ${\bf x}$' in Section C.1.
Specifically, for a source image ${\bf x}$, we define a hypersphere $\mathcal{B}^{\delta}_{f({\bf x} )}$ with radius $\delta$ centered at ${\bf z} = f({\bf x})$ in $\mathcal{Z}$.
We randomly sample ${\bf z}'$ from $\mathcal{B}^{\delta}_{f({\bf x} )}$ and generate ${\bf x}'$ from these ${\bf z}'$.
Note that ${\bf x}' $is not guaranteed to follow a Gaussian distribution.

\begin{observation}
For an image data ${\bf x}$, we experimentally observed the following relationship between the MSE and the image complexity $C({\bf x})$:
\begin{eqnarray}
	\EX_{ {\bf x}' \sim f^{-1} ( \mathcal{B}^{\delta}_{f({\bf x} )} ) }   \norm{ {\bf x}' -{\bf x}}^2
	\propto \exp \left(C( {\bf x} ) \right).
	\label{eq:comp_to_mse}
\end{eqnarray}
\end{observation}
To measure $C({\bf x})$, we used the JPEG2000, which is an entropy-based lossless compression algorithm.
This observation was established through experiments conducted on CIFAR-10 and SVHN with different values of $\delta$ as $\delta = \{0.01, 0.1, 1.0\}$.
The details of the experiment are explained in Appendix \ref{sec:comp_vs_var}.


\paragraph{Derivation of Hypothesis~\ref{thm:c}.}
By combining the observation of Eq.\ \eqref{eq:comp_to_mse} with Eq.\ \eqref{eq:main_var}, we obtain
Eq.\ \eqref{eq:main_c}:
\begin{eqnarray}
\label{eq:main_ent}
\frac{\epsilon^{2}}{ L_{\mathcal{A}}^{2} }
\left(  1 - \mathbb{P} ( \mathcal{B}^{\epsilon}_{{\bf z}} ) \right)
& \leq & \EX_{{\bf x}' \sim f^{-1}(\mathcal{B}^{\delta}_{f({\bf x} )})}   \norm{ {\bf x}' - {\bf x}}^2 \\
& = & C_{1} \exp (C({\bf x})),
\end{eqnarray}
where $C_{1}$ is a constant.

\section{Experiments}
The experiments were conducted on a CentOS server equipped with a single NVIDIA V100 GPU.
Compression with JPEG2000 was performed using the Python PIL package.

\subsection{Datasets with controlled Image complexity}
\label{sec:pooling_noise_and_manipulated_C10}

\subsubsection{Pooling Noise Image}
\label{sec:pooling_noise}
We generate images with controlled complexity by following the method proposed in \cite{Serra2020Input}.
First, we sample RGB noise images with a size of $32 \times 32$ from a uniform random distribution.
Next, we apply average pooling to these images and resize them back to $32 \times 32$.
The complexity of the images is controlled by the pooling size, denoted as $\kappa$.
Images with $\kappa = 1$ (i.e., no pooling) are the most complex, while those with $\kappa = 32$ (i.e., constant-color image) are the least complex.
For convenience, we refer to these images as \emph{pooling noise image}.
Sample images are shown in Fig.\ \ref{fig:ctl_noise}.

\subsubsection{Manipulated CIFAR-10}
\label{sec:manipulated_C10}
To create the \emph{manipulated CIFAR-10} dataset, we randomly select 1024 images from the test set of CIFAR-10.
We manipulate the image complexity by adding a $4 \times 4$ patch of weak noise ($[-0.2, 0.2]$) to a randomly selected area of each image.
The complexity of the images increases with the number of patches, denoted as $n$.
We label the set of images with $n$ patches as Noise-4-$n$.
In addition, we decrease the image complexity by applying the pooling operation in the same manner as the pooling noise images.
We denote the set with a pooling size of $\kappa$ as Pool-$\kappa$, and the image complexity decreases from Pool-2 to Pool-16.
Sample images are shown in Fig.\ \ref{fig:ctl_noise}.


\subsection{Measuring MSE and Image Complexity}
\label{sec:comp_vs_var}
\begin{figure*}[h]
\centering
\includegraphics[width=0.4 \textwidth]{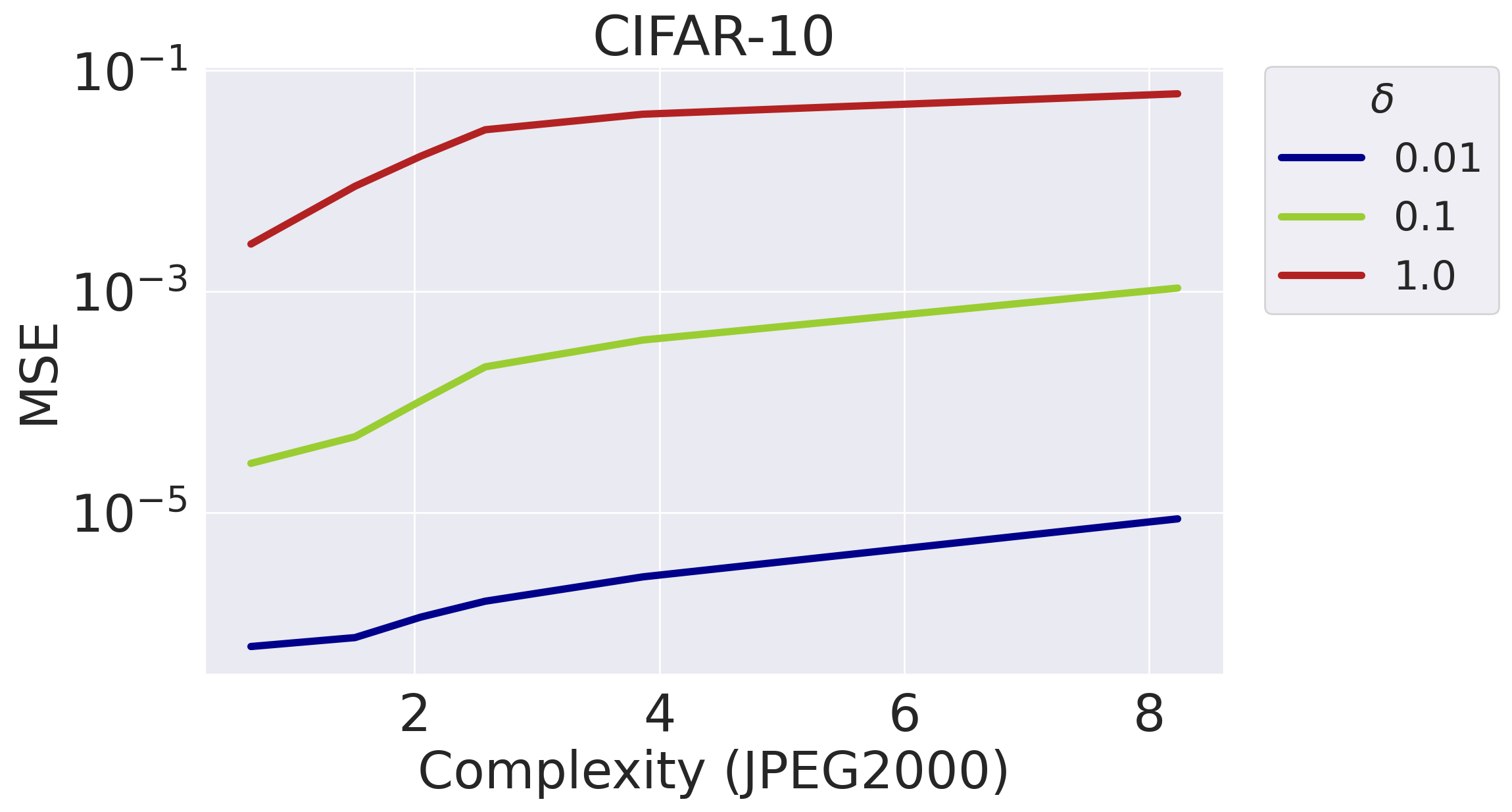}
\hspace{3px}
\includegraphics[width=0.4 \textwidth]{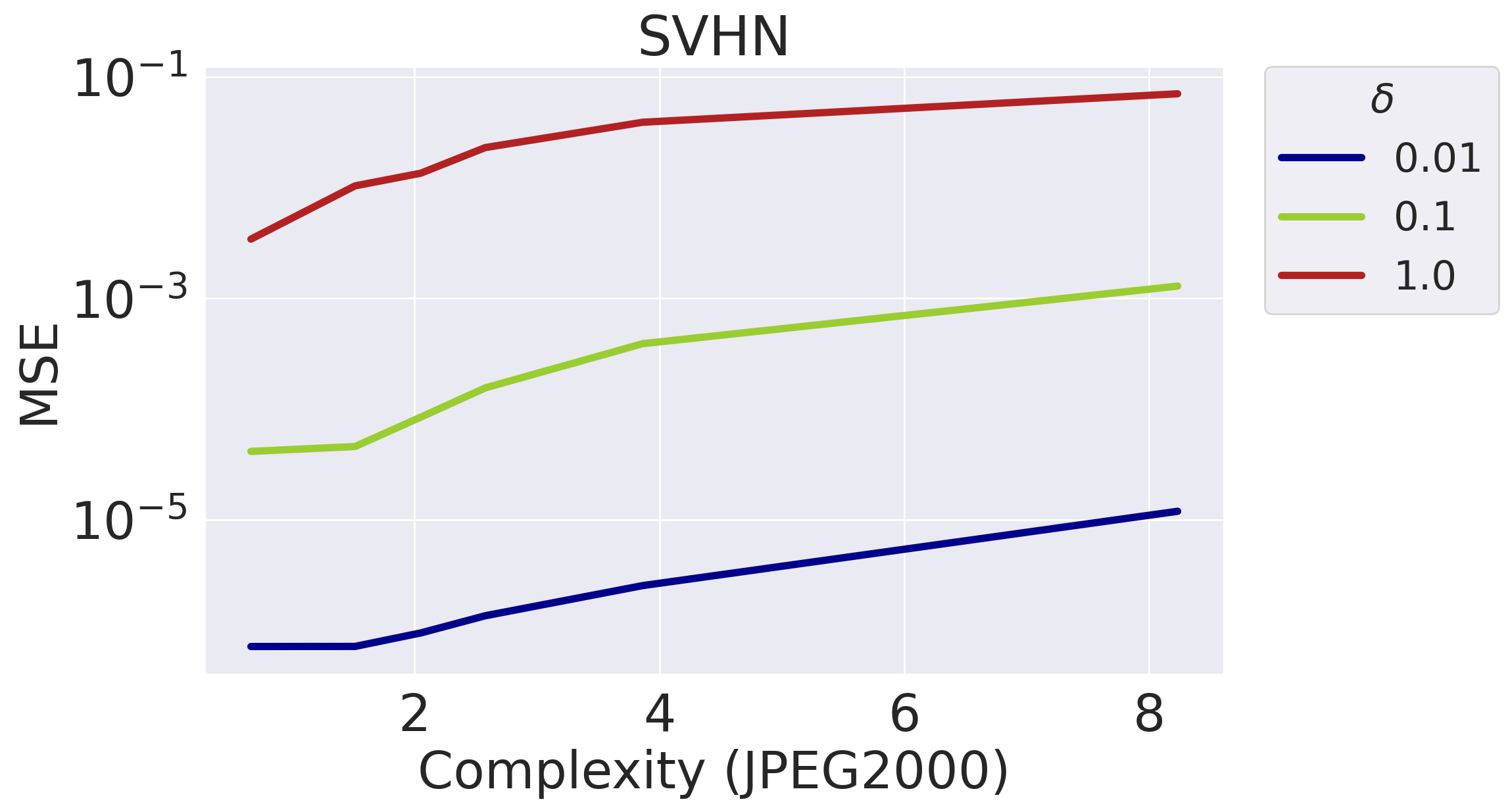}
\caption{Relationship between the Mean Squared Error (MSE), $\EX_{{\bf x}' \sim X_{\kappa}^{\delta}}   \norm{ {\bf x}' - {\bf x}_{\kappa}}^2$, and image complexity. The $x$-axis represents $C({\bf x}_{\kappa})$ computed with JPEG2000. The $y$-axis is displayed on a logarithmic scale, indicating that $C({\bf x}_{\kappa})$ is proportional to $\log \EX_{{\bf x}' \sim X_{\kappa}^{\delta}}   \norm{ {\bf x}' - {\bf x}_{\kappa}}^2$.}
\label{fig:comp_vs_var}
\end{figure*}
Using the pooling noise images with $\kappa = \{1,2,4,8,16,32\}$, we have six \emph{source images} ${\bf x}_{\kappa}$ with varying complexity.
We then obtain the corresponding latent vector ${\bf z}_{\kappa} = f({\bf x}_{\kappa})$ for each source image, where $f$ represents a trained Glow model.
Next, we define the set $B^{\delta}_{{\bf z}_{\kappa}} = \{{\bf z}' : \norm{{\bf z}' - {\bf z}_{\kappa}} \leq \delta \}$, which represents points within a hypersphere with a radius of $\delta$ centered at ${\bf z}_{\kappa}$.
We sample ${\bf z}'$ from $B^{\delta}_{{\bf z}_{\kappa} }$ and obtain a set of images $X_{\kappa}^{\delta} = \{ {\bf x}' = f^{-1}({\bf z}') : {\bf z}' \sim B^{\delta}_{{\bf z}_{\kappa} } \}$.
In our experiment, we varied $\delta$ as $\delta = \{0.01, 0.1,1.0\}$.
We randomly sampled $1024$ ${\bf z}'$ from $B^{\delta}_{{\bf z}_{\kappa} }$ and obtained $1024$ corresponding images in $X_{\kappa}^{\delta}$.
For two Glow models trained on CIFAR-10 and SVHN datasets,
we measured $\EX_{{\bf x}' \sim X_{\kappa}^{\delta}}   \norm{ {\bf x}' - {\bf x}_{\kappa}}^2$ for each source image ${\bf x}_{\kappa}$.
We then represented these measurements on the $y$-axis in Fig.\ \ref{fig:comp_vs_var}.
The $x$-axis represents the image complexity $C({\bf x}_{\kappa})$ computed using JPEG2000.
$C({\bf x}_{\kappa})$ varies with different values of ${\kappa}$.
From the plotted figures, we observed that, irrespective of the value of $\delta$, the relationship between $C({\bf x})$ and $ \log \EX_{{\bf x}' \sim X_{\kappa}^{\delta}}   \norm{ {\bf x}' - {\bf x}_{\kappa}}^2$
approximately follows a linear trend.
This suggests the following connection:
\begin{eqnarray}
\EX_{{\bf x}' \sim f^{-1}(\mathcal{B}^{\delta}_{f({\bf x} )})} \norm{ {\bf x}' - {\bf x}}^2
\propto \exp \left(C( {\bf x} ) \right).
\end{eqnarray}

\subsection{Models}
\subsubsection{Model Specification}
\label{sec:dgms}
We implemented the DGMs using publicly available code.
The architecture of Glow and CV-Glow\footnote{github.com/kmkolasinski/deep-learning-notes/tree/master/seminars/2018-10-Normalizing-Flows-NICE-RealNVP-GLOW} mainly depends on two parameters: the depth of flow, denoted as $K$, and the number of levels, denoted as $L$.
The flow consists of affine coupling and $1 \times 1$ convolutions, which are performed $K$ times, followed by the factoring-out operation \cite{dinh2016density}.
This sequence is repeated $L$ times.
For CIFAR-10 and SVHN, we set $K=32$ and $L=3$.
For MNIST and FMNIST, we set $K=48$ and $L=2$.
For ImageNet, we set $K=64$ and $L=5$.
We applied data augmentation techniques of random $2 \times 2$ translations and horizontal flipping on CIFAR-10.
Random cropping at a size of $224 \times 224$ and horizontal flipping were applied to ImageNet.
We trained the models for $99100$ iterations for ImageNet and $45100$ iterations for other datasets using the Adam optimizer
\cite{kingma2014adam}
with $\beta_{1} = 0.9$ and $\beta_{2} = 0.999$.
The batch size is $8$ for ImageNet and $256$ for other datasets.
The learning rate was set to $1e-4$ for the first $5100$ iterations and 5e-5 for the remaining iterations for all datasets.
For ResFlow,\footnote{github.com/rtqichen/residual-flows} we used the pre-trained model on CIFAR-10 and MNIST and trained the models on SVHN and FMNIST with the same setting as CIFAR-10 and MNIST, respectively.
For iResNet,\footnote{github.com/jhjacobsen/invertible-resnet} we used the pre-trained model on CIFAR-10 and trained the models on SVHN, MNIST, and FMNIST with the same setting as for CIFAR-10.
For IDF,\footnote{github.com/jornpeters/integer\_discrete\_flows} we trained the models on all datasets using the settings provided by the original author for the CIFAR-10 model.
For PixelCNN++,\footnote{github.com/pclucas14/pixel-cnn-pp} we utilized the pre-trained model on CIFAR-10 and trained the models on SVHN with the same setting as CIFAR-10.

\subsection{Sensitivity to Image Complexity}
\label{sec:arch_and_sensitivity}
The strength of the positive correlation between $\log p({\bf z})$
and the volume $\log |\text{det} \ J_{f}({\bf x})|$ discussed in Section \ref{sec:exp_for_hypo}, Experiment 3, depends on the architecture of the NF, specifically its flexibility measured as the Lipschitz constant, $L$.
This can be observed from the relationship $|\text{det} \ J_{f}({\bf x})| < L^{d}$ (Remark~\ref{rem:volume}).
For  CV-Glow and IDF, where the volume term is fixed, it was observed that the increase in $\log p({\bf z})$ was equivalent to the increase in $\log p({\bf x})$.
On the other hand, for iResFlow and ResFlow, which restrict the Lipschitz constant ($L \leq 1$),
the impact of variation in $\log p({\bf x})$ on the volume was reduced.
As a result, more of the variation in $\log p({\bf x})$ was directed towards the variation in $\log p({\bf z})$.
In Glow, there is no such restriction on the Lipschitz constant.
Therefore, the sensitivity of $\log p({\bf z})$ to variations in image complexity is the lowest for Glow, followed by iResFlow and ResFlow, and the highest for  CV-Glow and IDF.
The same conclusion can be drawn from Eq.\ \eqref{eq:main_c}, where
the sensitivity of the density around ${\bf z}$, $\mathbb{P}(\mathcal{B}^{\epsilon}_{{\bf z} } )$,
to variations in $C({\bf x})$ increases when the flexibility is restricted (i.e., $L$ is small).
Based on this property, we select the NF architecture for OOD detection, as discussed in Section~\ref{sec:gmm} and Appendix~\ref{sec:gmm_detail}.

\subsection{Preparation for Experiments 4 and 5}
\label{sec:prep_exp_3_4}

\subsubsection{Selection of NF architecture for OOD detection}
\label{sec:arch_for_oodd}
For Experiment 4, we employed three Normalizing Flow (NF) models: Glow, iResNet, and ResFlow.
In this study, OOD detection is treated as the task of distinguishing an in-distribution dataset from other datasets.
The architecture of the NF model plays a pivotal role in effectively capturing large-scale differences between datasets, while remaining less sensitive to small-scale variations among individual samples.
To achieve robustness in this regard, it is desirable for NF models to be somewhat insensitive to subtle differences between individual samples.
The flexibility of $J_{f}({\bf x})$ in NFs, as discussed in Appendix~\ref{sec:arch_and_sensitivity}, determines the sensitivity of  $\log p({\bf z})$ to sample-specific differences in image complexity.
From this standpoint, the discriminative capacity using $\log p({\bf z})$ for dataset differentiation can be ranked.
Glow exhibits the lowest sensitivity, followed by iResNet and ResFlow, while CV-Glow and IDF show the highest sensitivity.
We validate these characteristics for Glow, ResFlow, and IDF in Fig.\ \ref{fig:comp_vs_z__ood}, with supplementary results for the other two NF models presented in Appendix \ref{sec:extra_exp_3}.
The plots for IDF exhibit a high sensitivity to $\log p({\bf z})$, leading to significant overlap between the In-Dist and OOD datasets.
Similar overlapping plots were also observed for CV-Glow.
In contrast, Glow, iResNet, and ResFlow exhibit less overlap and provide enhanced discrimination.
Consequently, we opted for Glow, iResNet, and ResFlow.

\subsubsection{OOD detection with GMM}
\label{sec:gmm_detail}
We present a complexity-aware OOD detection method using GMM.
In the problem set-up of OOD detection, In-Dist data are assumed to be available,
and we train a GMM to fit the In-Dist data at hand.
Here is the operational procedure of our proposed method:
\begin{enumerate}
\item Compute the image complexity $C_{{\bf x}_{\text{train}}} = C({\bf x}_{\text{train}})$ and the log likelihood in latent space $L_{{\bf z}_{\text{train}}} = \log p({\bf z}_{\text{train}})$ for the training images ${\bf x}_{\text{train}}$ from the In-Dist dataset (i.e., CIFAR-10, SVHN, MNIST, FMNIST, and ImageNet).
${\bf z}_{\text{train}} = f({\bf x}_{\text{train}})$ is obtained by applying the NF model $f$ to ${\bf x}_{\text{train}}$.
\item  Train a GMM, denoted as  $\mathcal{G} = \{\pi_{k}, \mu_{k}, \Sigma_{k}\}$, utilizing the computed $C_{{\bf x}_{\text{train}}}$ and $L_{{\bf z}_{\text{train}}}$.
In this representation, $\pi_{k}, \mu_{k}$, and $\Sigma_{k}$ correspond to the mixing coefficient, mean, and covariance matrix for the $k$-th component of the GMM. The parameter $K$ signifies the number of mixture components.
\item For a given test image ${\bf x}_{\text{test}}$, compute its  image complexity $C_{{\bf x}_{\text{test}}} = C({\bf x}_{\text{test}})$ and the log likelihood in latent space $L_{{\bf z}_{\text{test}}} = \log p({\bf z}_{\text{test}})$, where ${\bf z}_{\text{test}} = f({\bf x}_{\text{test}})$.
\item Compute the log likelihood for the trained GMM $\mathcal{G}$ as $S_{\text{GMM}}({\bf x}_{\text{test}}) = \log \Sigma_{k=1}^{K} \pi_{k} \mathcal{N} \left ((C_{{\bf x}_{\text{test}}}, L_{{\bf z}_{\text{test}}}) |  \mu_{k}, \Sigma_{k} \right )$.
$S_{\text{GMM}}({\bf x}_{\text{test}})$ represents the In-Dist score of ${\bf x}_{\text{test}}$,
where a larger value indicates that ${\bf x}_{\text{test}}$ is more likely to belong to the In-Dist.
\end{enumerate}

The dimensionality of GMM is as small as two, so there is no need to concern about the discrepancy between the high-density region and the typical set region, which occurs in high-dimensional spaces (Appendix~\ref{sec:typical_set}).

In the experiments, we had $1024$ training samples ${\bf x}_{\text{train}}$ and  $2048$ test samples ${\bf x}_{\text{test}}$, consisting of $1024$ samples from the OOD dataset and $1024$ samples from the test portion of the In-Dist dataset.
We used the scikit-learn~\cite{scikit-learn} with its default hyper-parameter values, including the number of Gaussian components $K=3$.
We also tried different values for $K$, such as $K=2$ and $K=4$, but the results were similar to using $K=3$.

\subsubsection{Competitors}
\label{sec:comparisons_detail}
In addition to TTL and CALT, we compare the performance of the proposed methods to the existing NF-based methods on CIFAR-10 and SVHN as follows:

\setlength{\leftmargini}{0.5cm}
\begin{enumerate}

\item The Watanabe-Akaike Information Criterion (WAIC) \cite{choi2019generative} measures $\mathbb{E}[\log p({\bf x})] - \text{Var}[\log p({\bf x})]$ with an ensemble of five Glow models.
The five models were trained separately with different initial parameters.
CV-Glow and a background model used in the LRB method are included as components of the ensemble.

\item The likelihood-ratio to background model (LRB) \cite{NEURIPS2019_1e795968} measures the likelihood ratio $\log p({\bf x}) - \log p_0({\bf x})$ where $p_0({\bf x})$  is a background model trained using the training data with additional random noise sampled from the Bernoulli distribution with a probability 0.15.
LRB utilizes two Glow models for the likelihood computation.

\item The likelihood-ratio to general model (LRG) \cite{NEURIPS2020_f106b7f9} measures the likelihood ratio $\log p({\bf x}) - \log p_{g}({\bf x})$  where $p_{g}({\bf x})$  is a general model trained on broader and more general data than In-Dist data.
Specifically, TinyImageNet~\cite{ILSVRC15}  is used for training the general model.
\end{enumerate}

\section{More Experimental Results}

\subsection{Additional Results of Experiment 1}
\label{sec:extra_exp_1}
Here are the additional results of Experiment~1 with other NF models:
Fig.\ \ref{fig:extra_comp_vs_z_pooling} shows the results on the pooling noise images.
Fig.\  \ref{fig:extra_comp_vs_z_manipulated} shows the results on the manipulated CIFAR-10.
It can be observed that CV-Glow assigns an extremely small $\log p({\bf z})$ to images with high complexity, such as Noise-2, Noise-1 in Fig.\ \ref{fig:extra_comp_vs_z_pooling}, as well as  Noise-4-8, and Noise-4-16 in Fig.\  \ref{fig:extra_comp_vs_z_manipulated},
This leads to divergent values of $\norm{{\bf z}}$.
To ensure proper visualization of the figures, the $y$-axis scale, (upper limit of $\norm{{\bf z}}$) was capped at $100$ in the CV-Glow plots.

\begin{figure}[h]
\centering
\includegraphics[height=3.4cm, trim={0 0 0 1.1cm},clip]{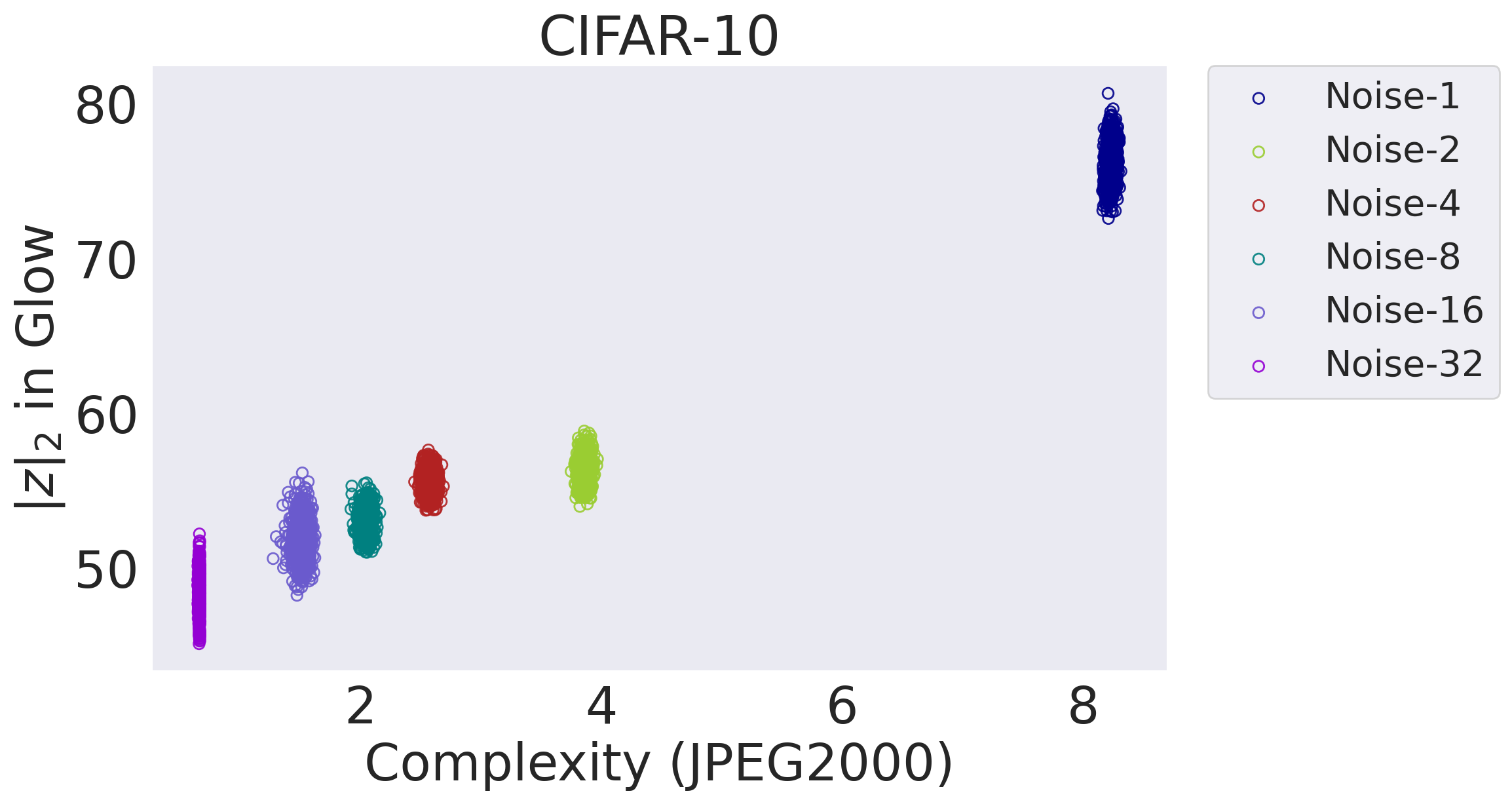}
\includegraphics[height=3.4cm, trim={0 0 0 1.05cm},clip]{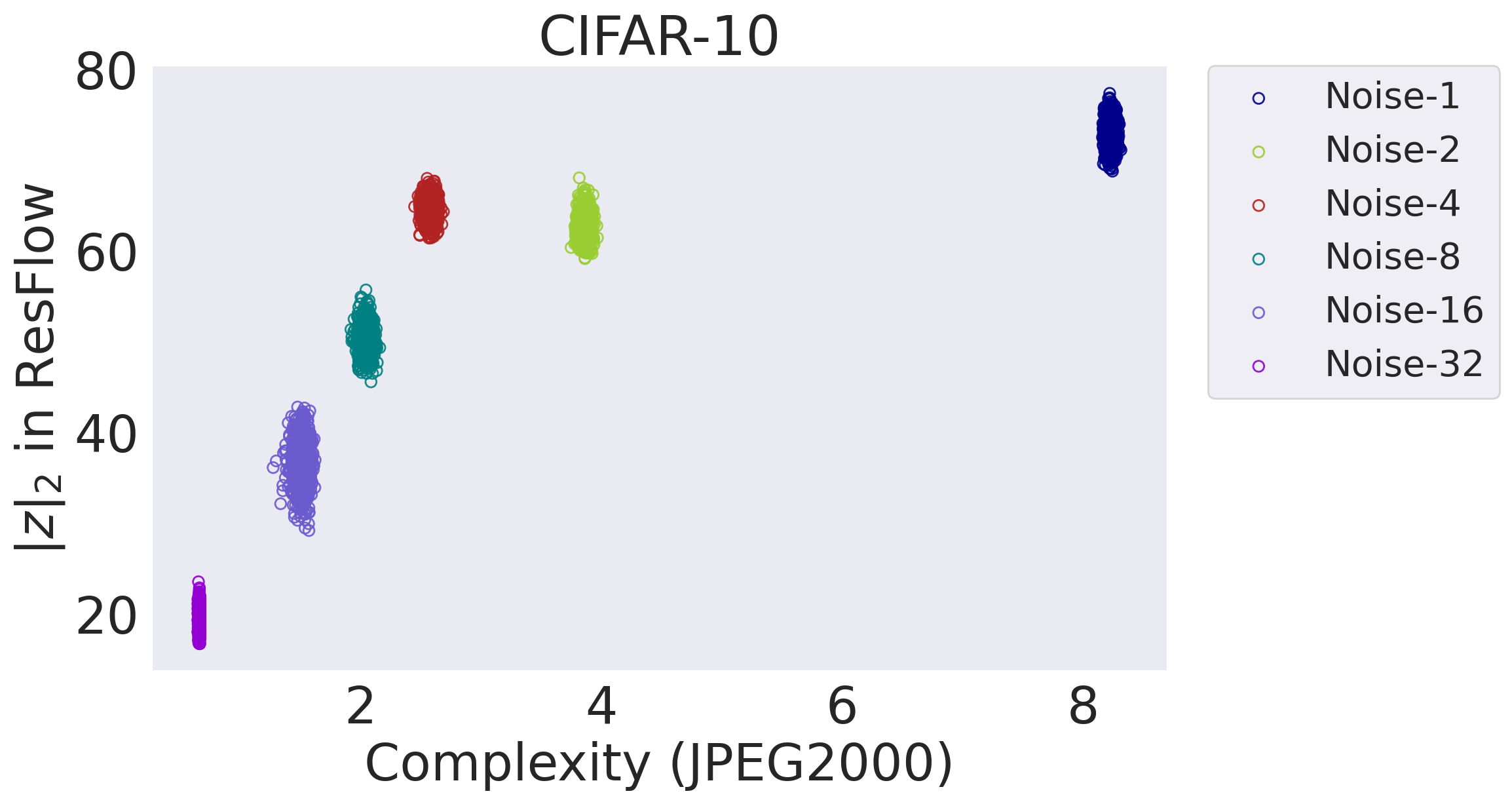}
\includegraphics[height=3.4cm, trim={0 0 0 1.1cm},clip]{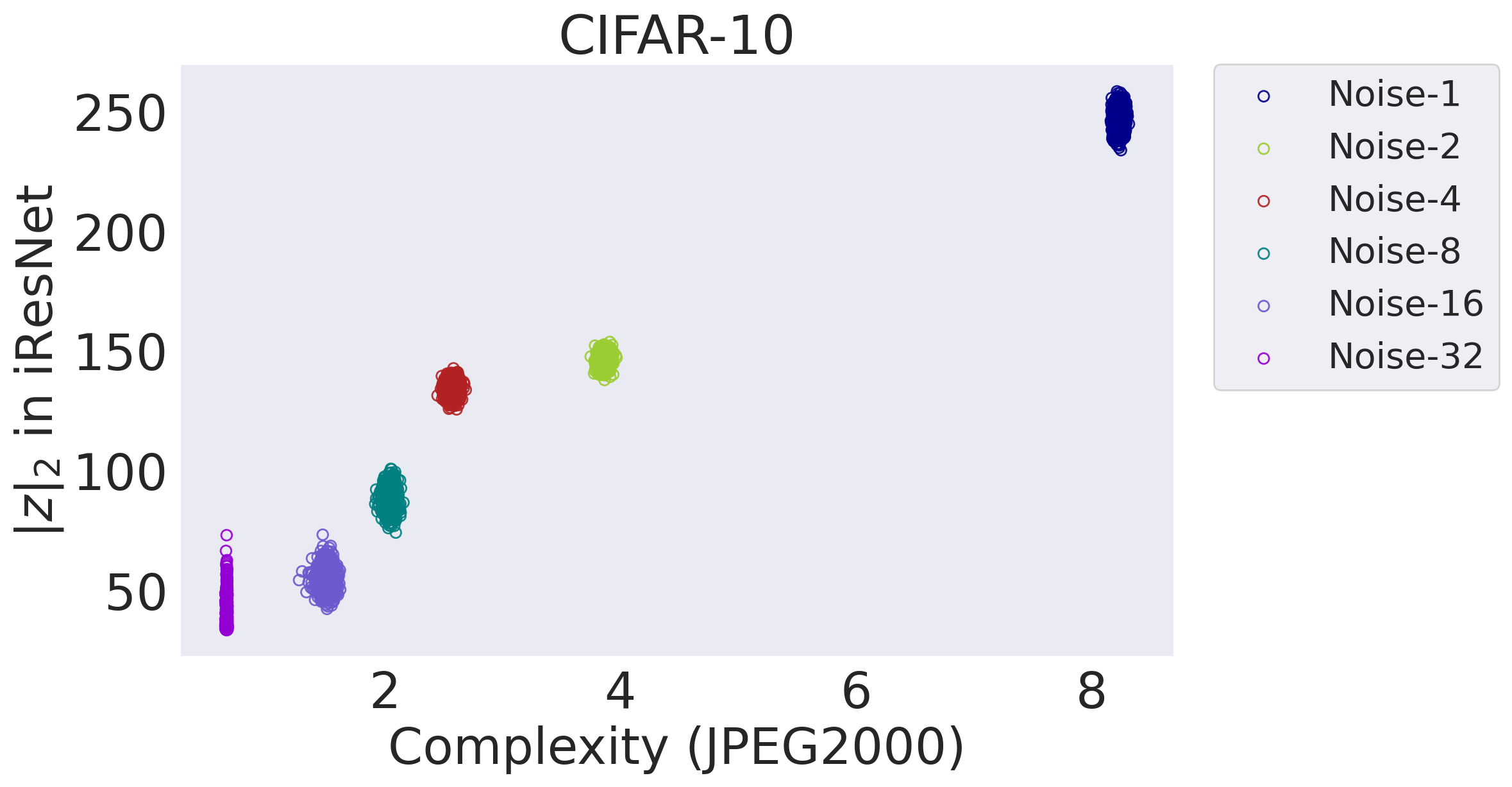}

\includegraphics[height=3.4cm, trim={0 0 0 1.1cm},clip]{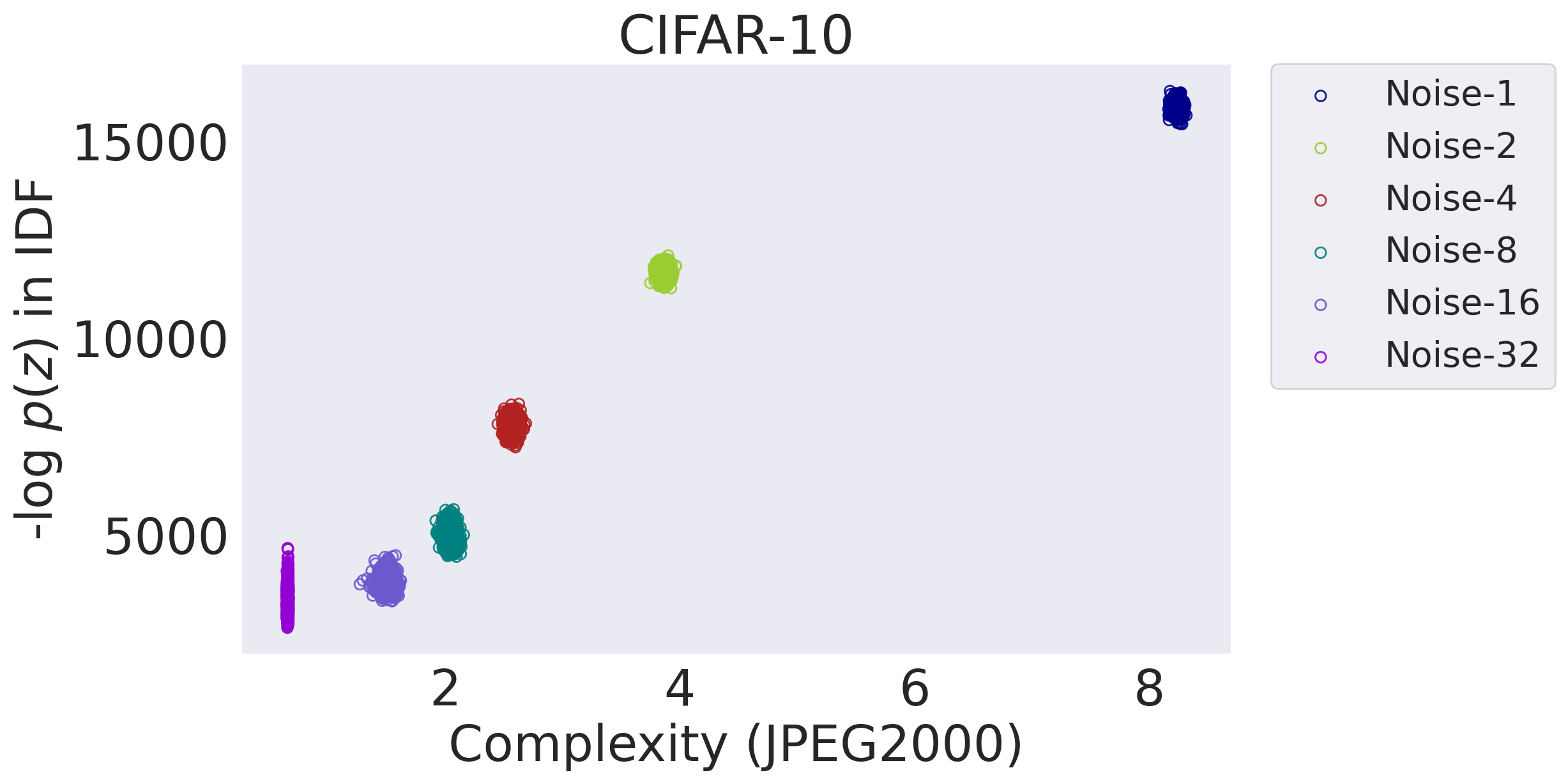}
\includegraphics[height=3.4cm, trim={0 0 0 1.05cm},clip]{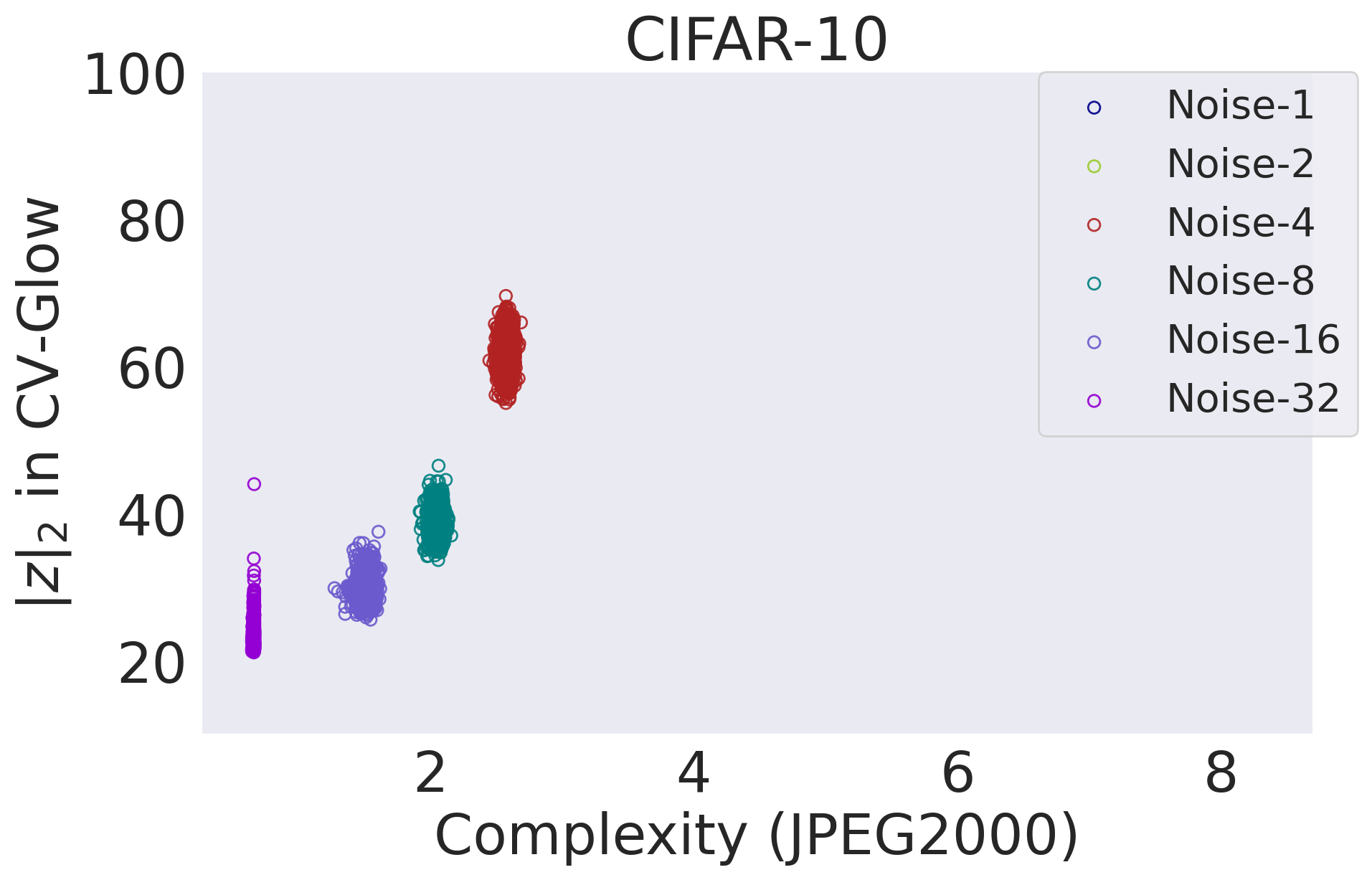}

\caption{Complexity vs.\ $\norm{{\bf z}} ( \propto - \sqrt{\log p({\bf z})})$ on the pooling noise images. In-Dist is CIFAR-10.
	The NF models are Glow, ResFlow, and iResNet for the top three and IDF and CV-Glow for the bottom two.
	In CV-Glow, $\norm{{\bf z}}$ for the samples of Noise -1 and Noise-2 became extremely large, so we omitted plotting them.}
\label{fig:extra_comp_vs_z_pooling}
\end{figure}

\begin{figure}[h]
\centering
\includegraphics[height=3.4cm, trim={0 0 0 1.15cm},clip]{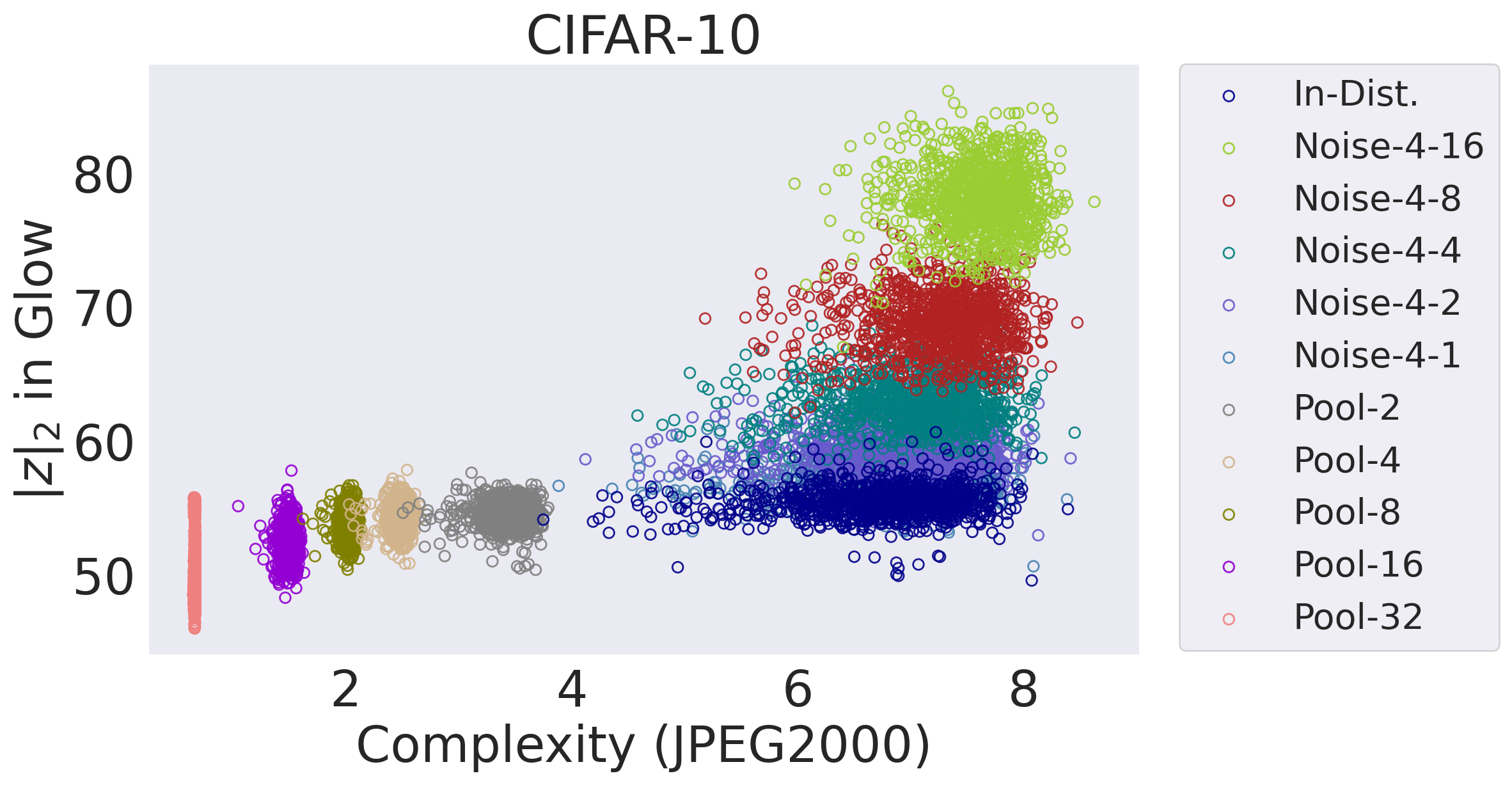}
\includegraphics[height=3.4cm, trim={0 0 0 1.15cm},clip]{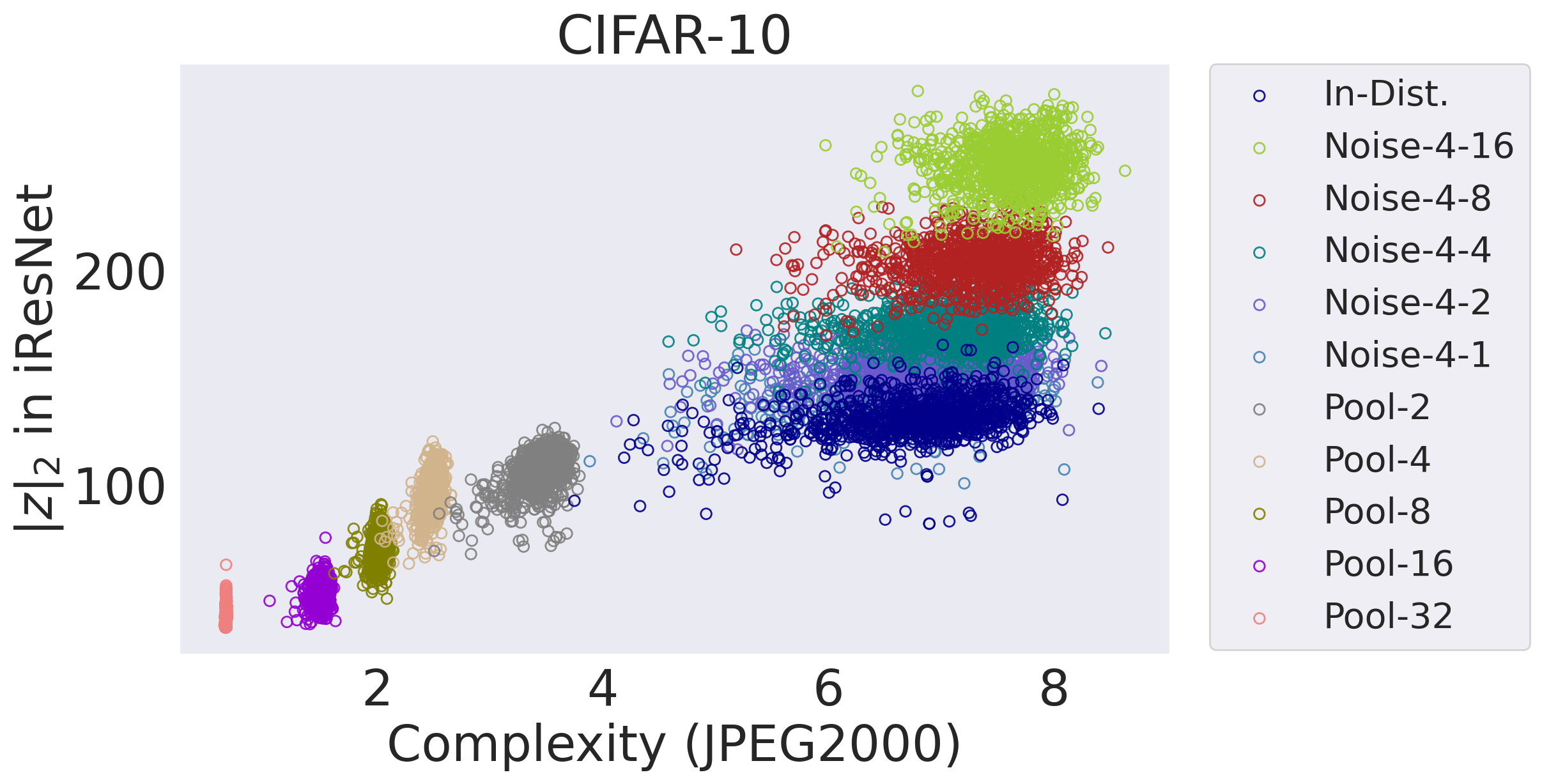}
\vspace{10px}

\includegraphics[height=3.4cm, trim={0 0 0 1.15cm},clip]{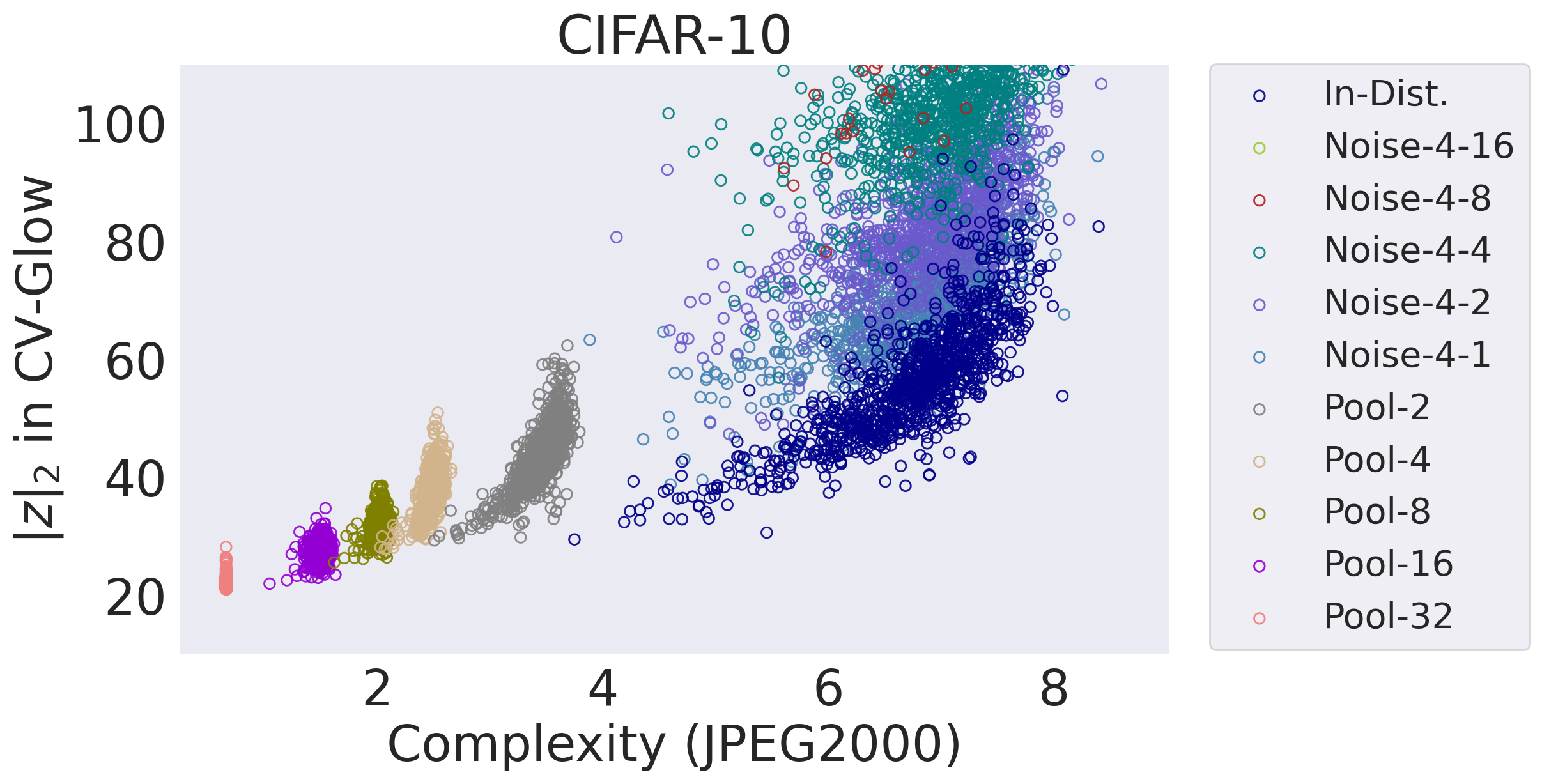}
\hspace{10px}
\includegraphics[height=3.4cm, trim={0 0 0 1.15cm},clip]{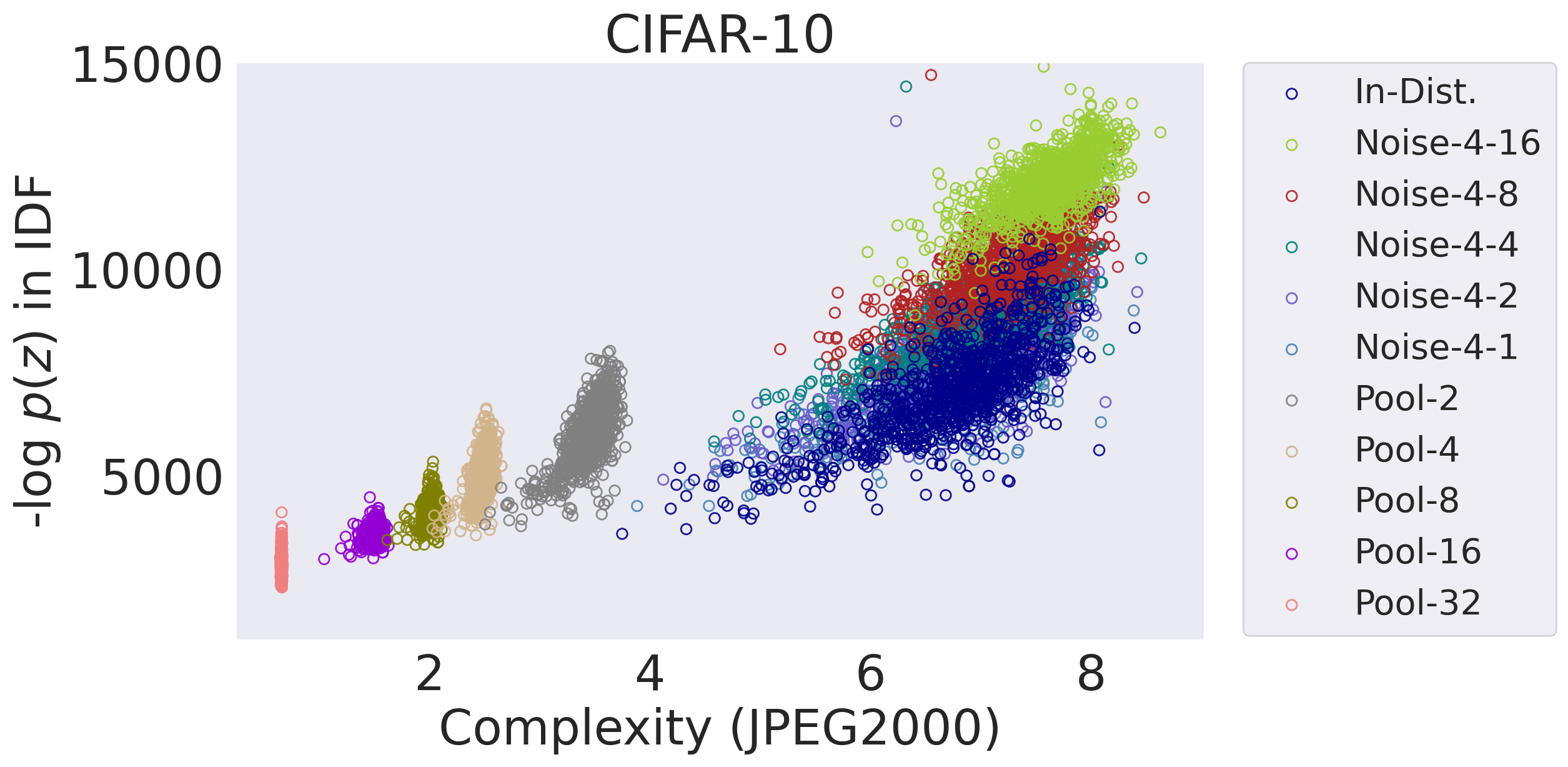}

\caption{Complexity vs.\ $\norm{{\bf z}} ( \propto - \sqrt{\log p({\bf z})})$ on the manipulated CIFAR-10. In-Dist is CIFAR-10. The NF models are Glow, iResNet, CV-Glow, and IDF, from left to right and top to bottom.}
\label{fig:extra_comp_vs_z_manipulated}
\end{figure}

\subsection{Additional Results of Experiment 2 and 3}
\label{sec:extra_exp_2}
We present the results of Experiment~2 and 3 with other NF models in Fig.\ \ref{fig:extra_volume_vs_z}.

\begin{figure}[h]
\centering

\includegraphics[height=3.4cm, trim={0 0 0 1.15cm},clip]{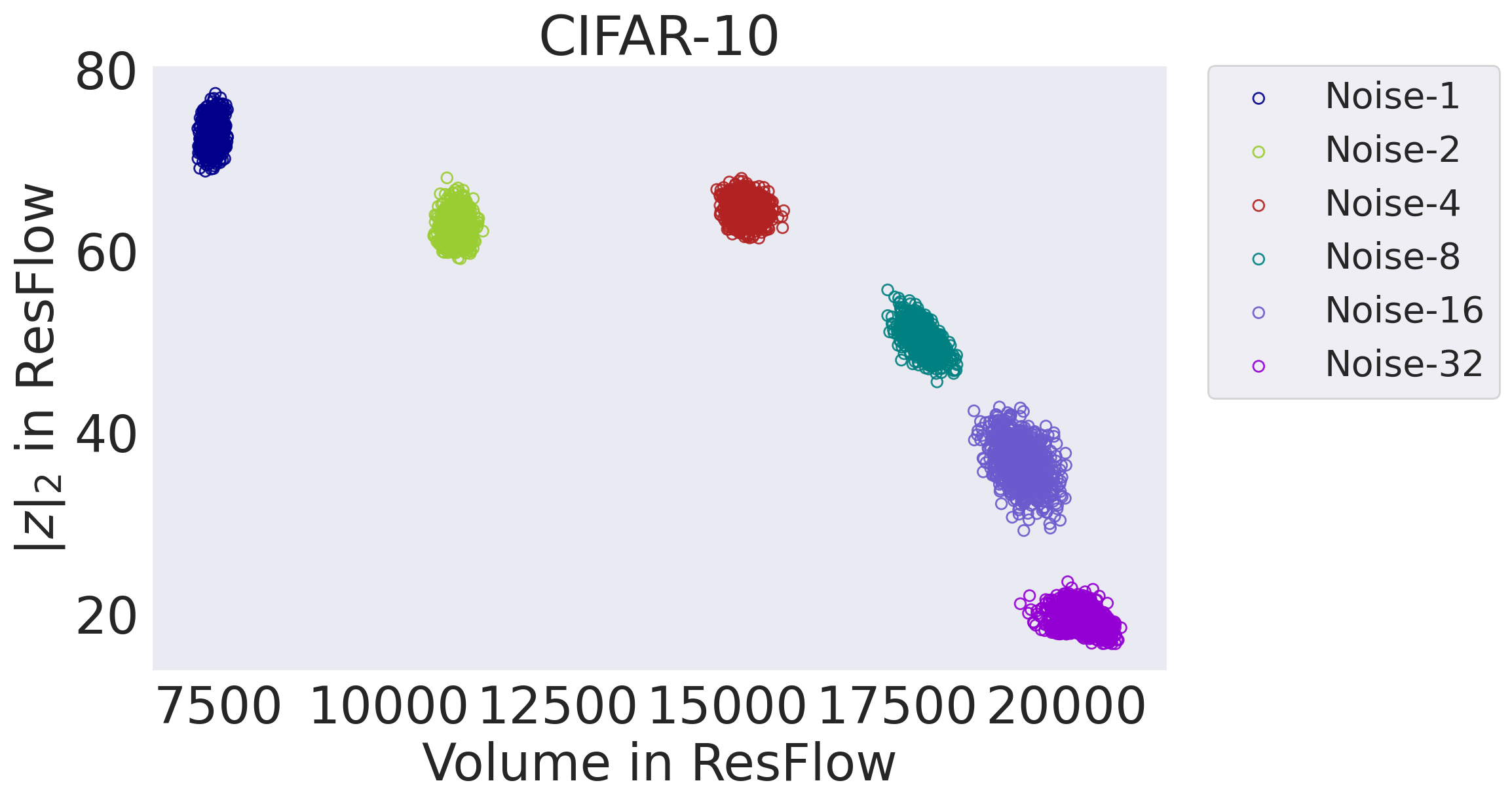}
\hspace{10px} 
\includegraphics[height=3.4cm, trim={0 0 0 1.15cm},clip]{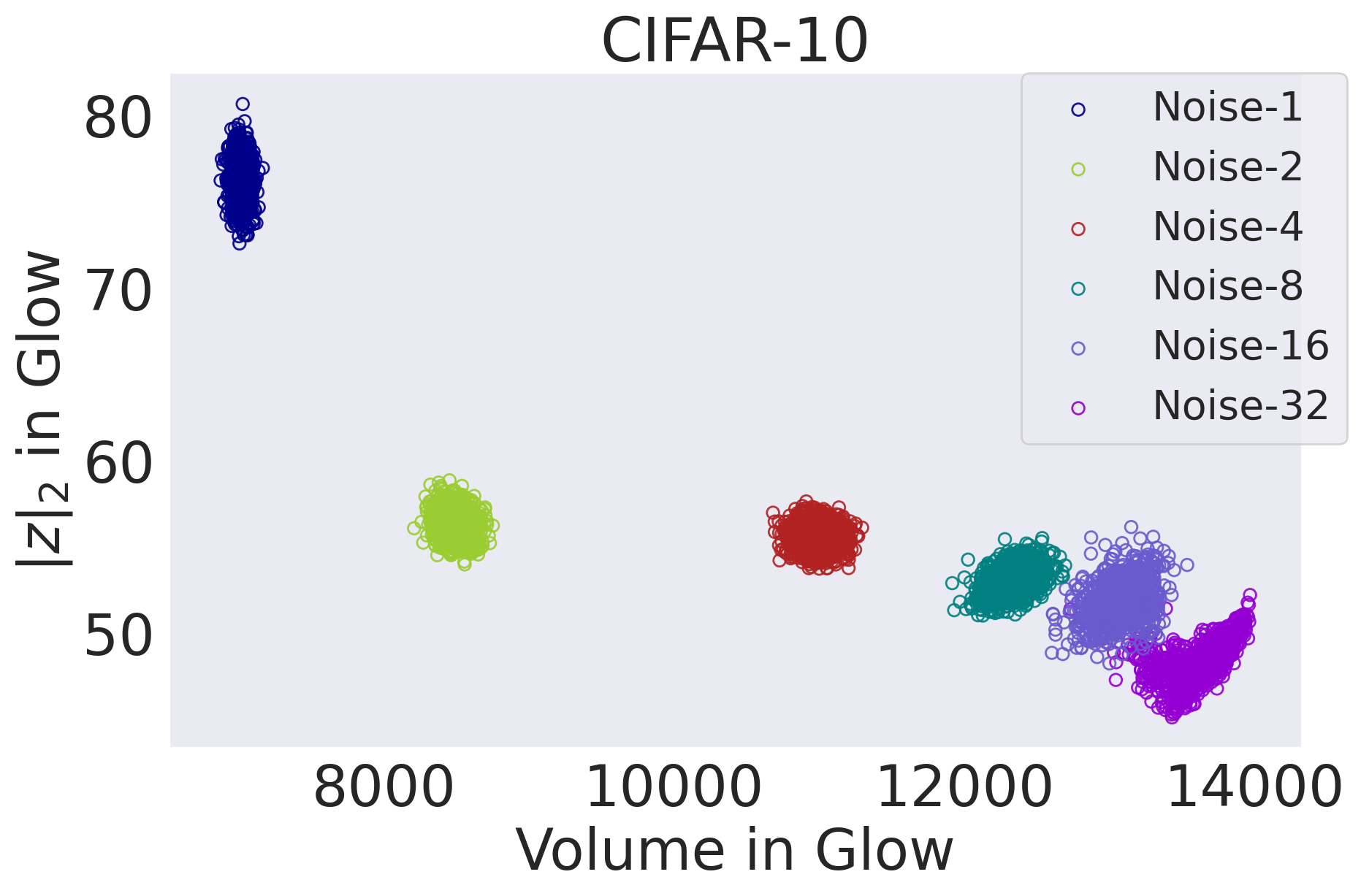}
\includegraphics[height=3.4cm, trim={0 0 0 1.15cm},clip]{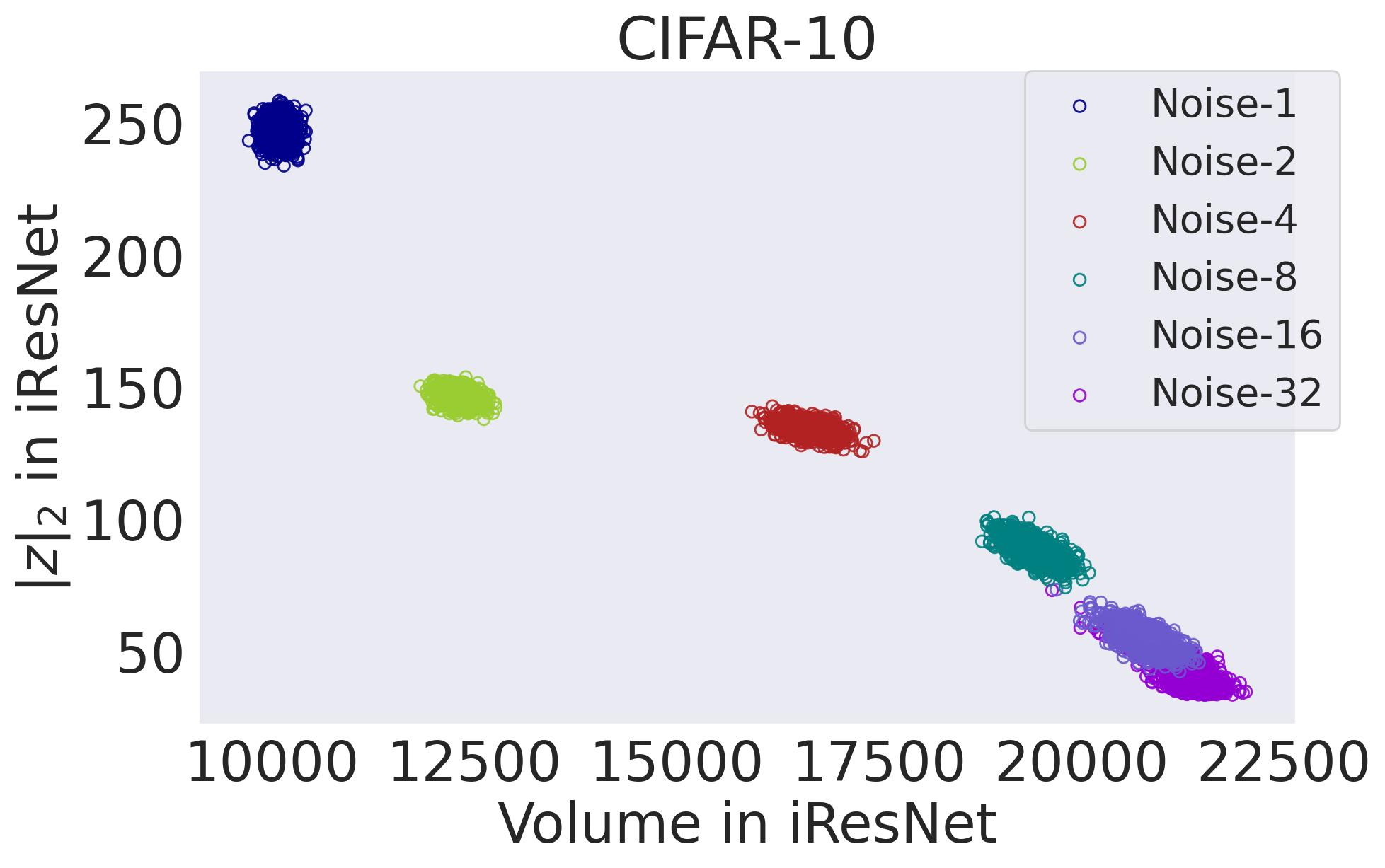}

\includegraphics[height=3.4cm, trim={0 0 0 1.15cm},clip]{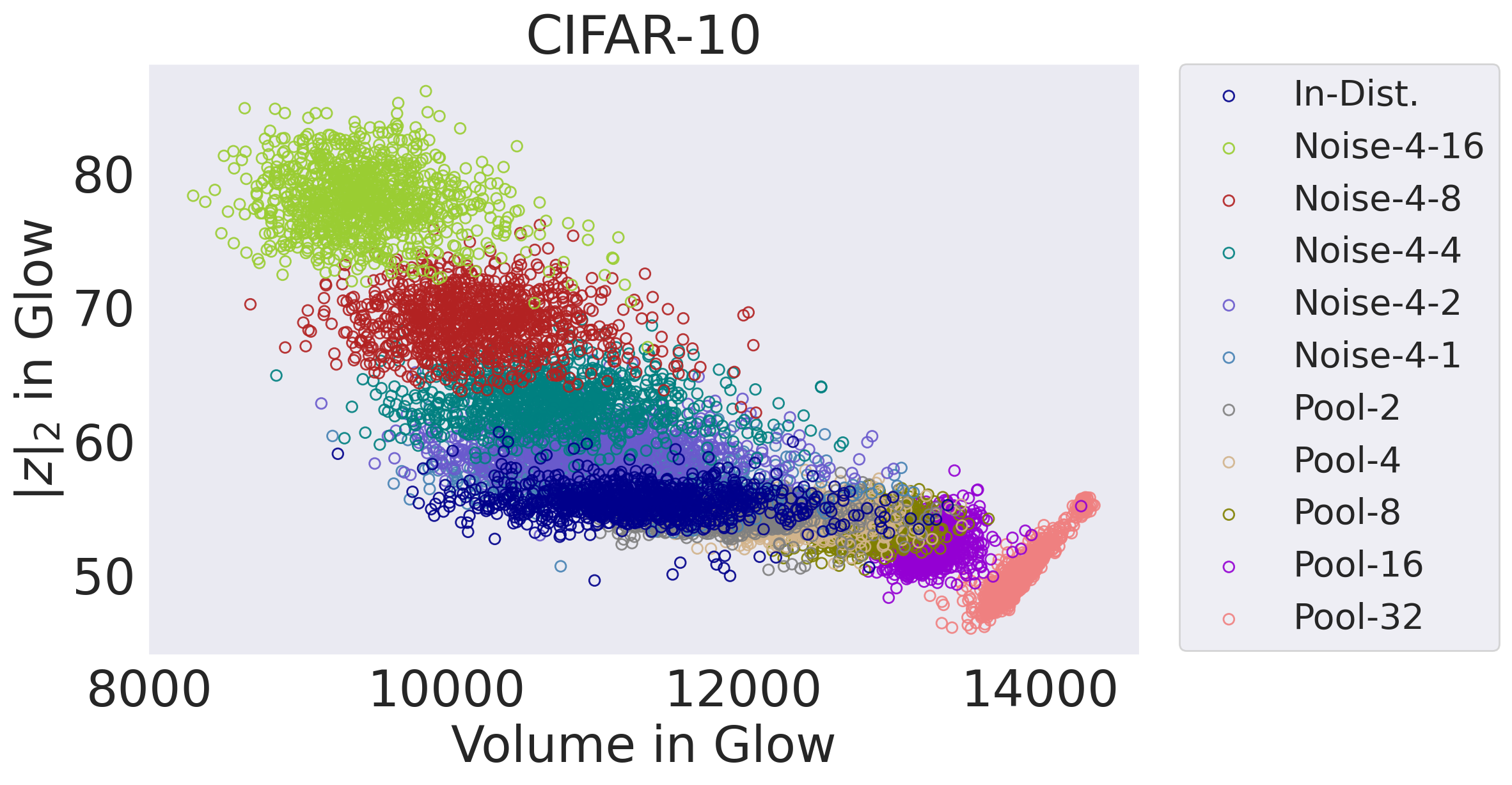}
\includegraphics[height=3.4cm, trim={0 0 0 1.15cm},clip]{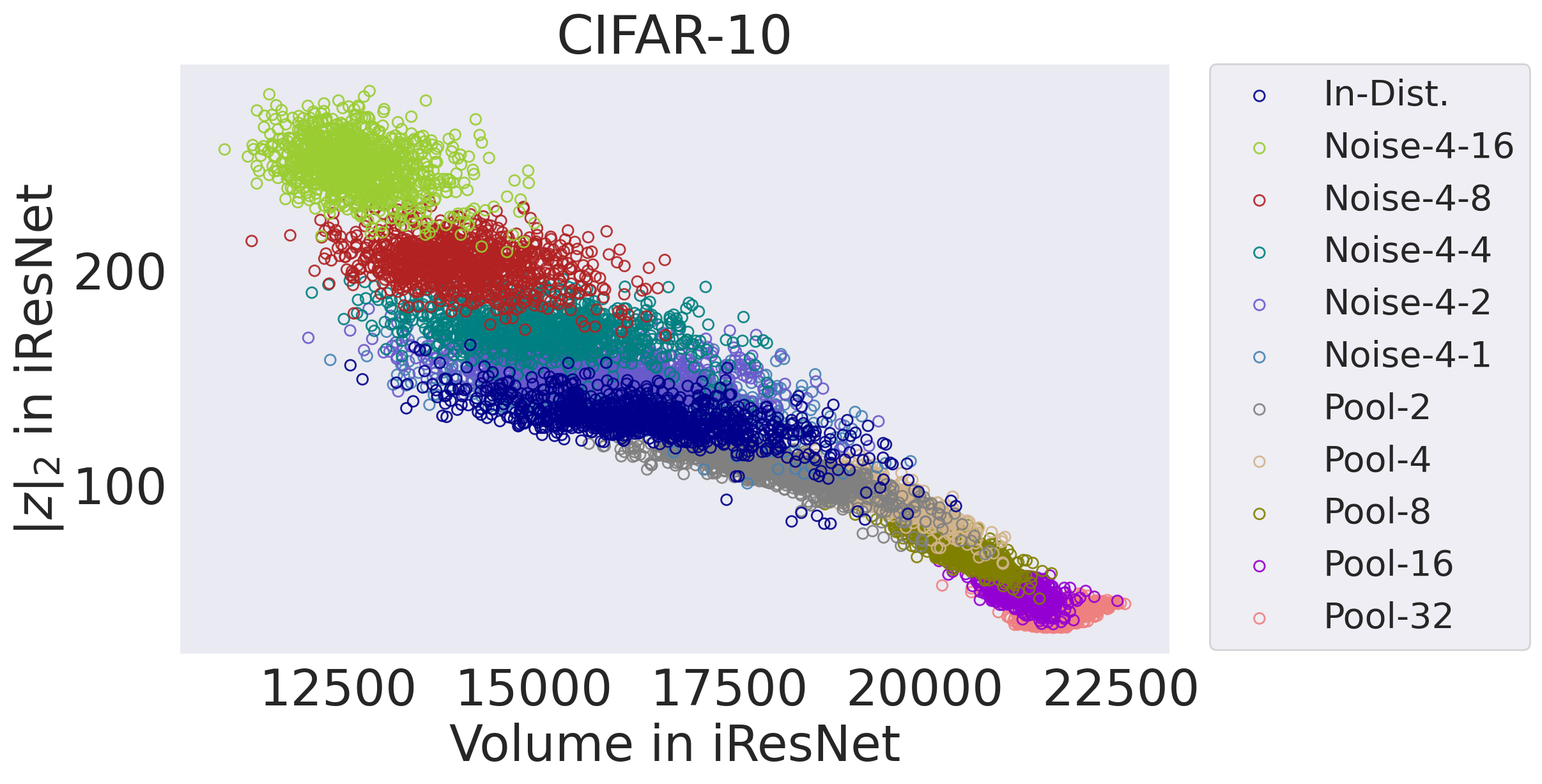}
\caption{Volume vs.\ $\norm{{\bf z}} ( \propto - \sqrt{\log p({\bf z})})$. In-Dist is CIFAR-10. Upper three are for the pooling noise images with ResFlow, Glow, and iResNet. Lower two are for the manipulated CIFAR-10 with Glow and iResNet.}
\label{fig:extra_volume_vs_z}
\end{figure}

\subsection{Additional Results of Experiment 4 and 5}
\label{sec:extra_exp_3}
We present the results of Experiment~3 with other NF models in Fig.\ \ref{fig:extra_comp_vs_z__ood}.
\begin{figure}[h]
\centering
\includegraphics[height=3.15cm, trim={0 0 6cm 1.15cm}, clip]{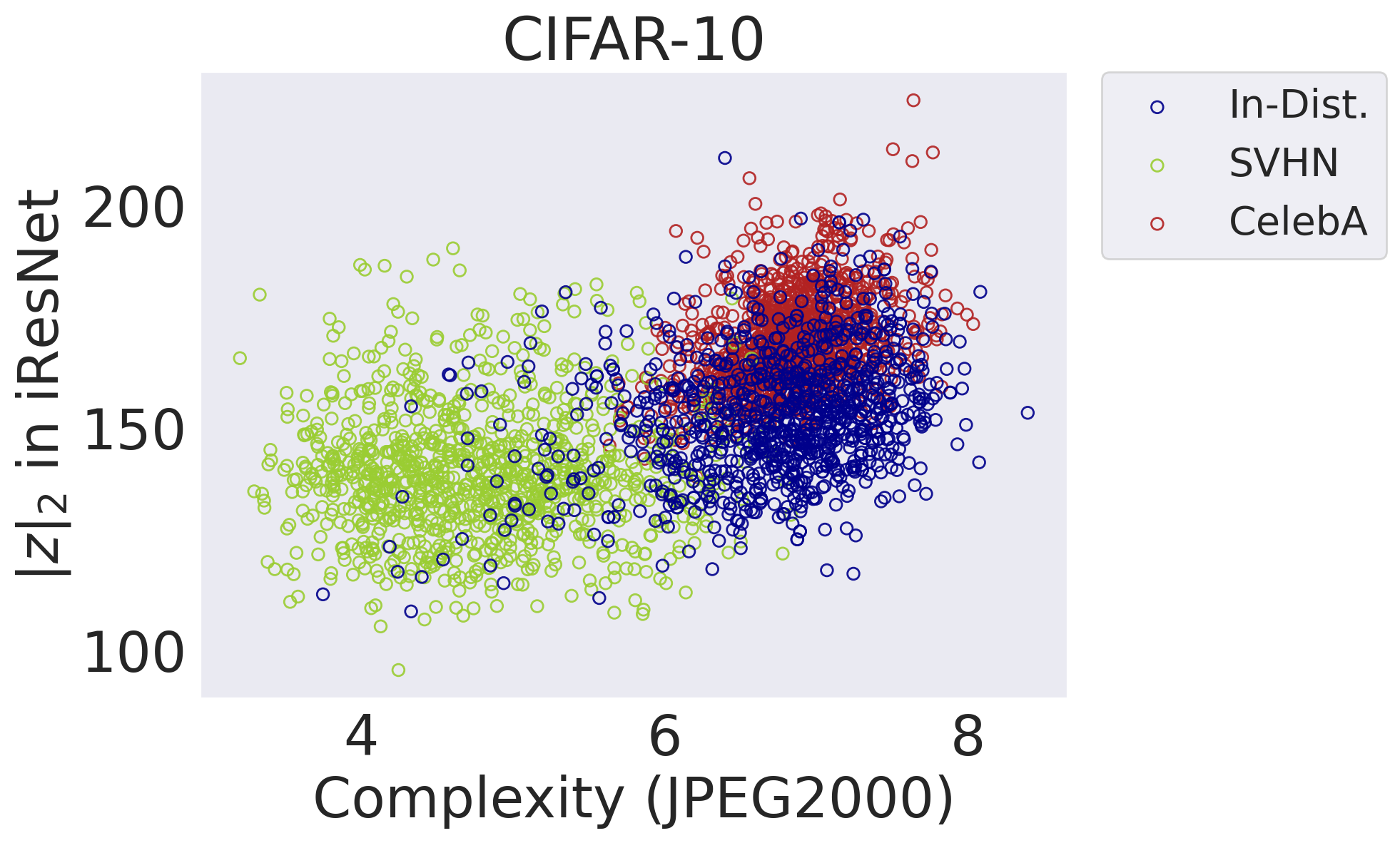}
\includegraphics[height=3.15cm, trim={0 0 6cm 1.15cm},clip]{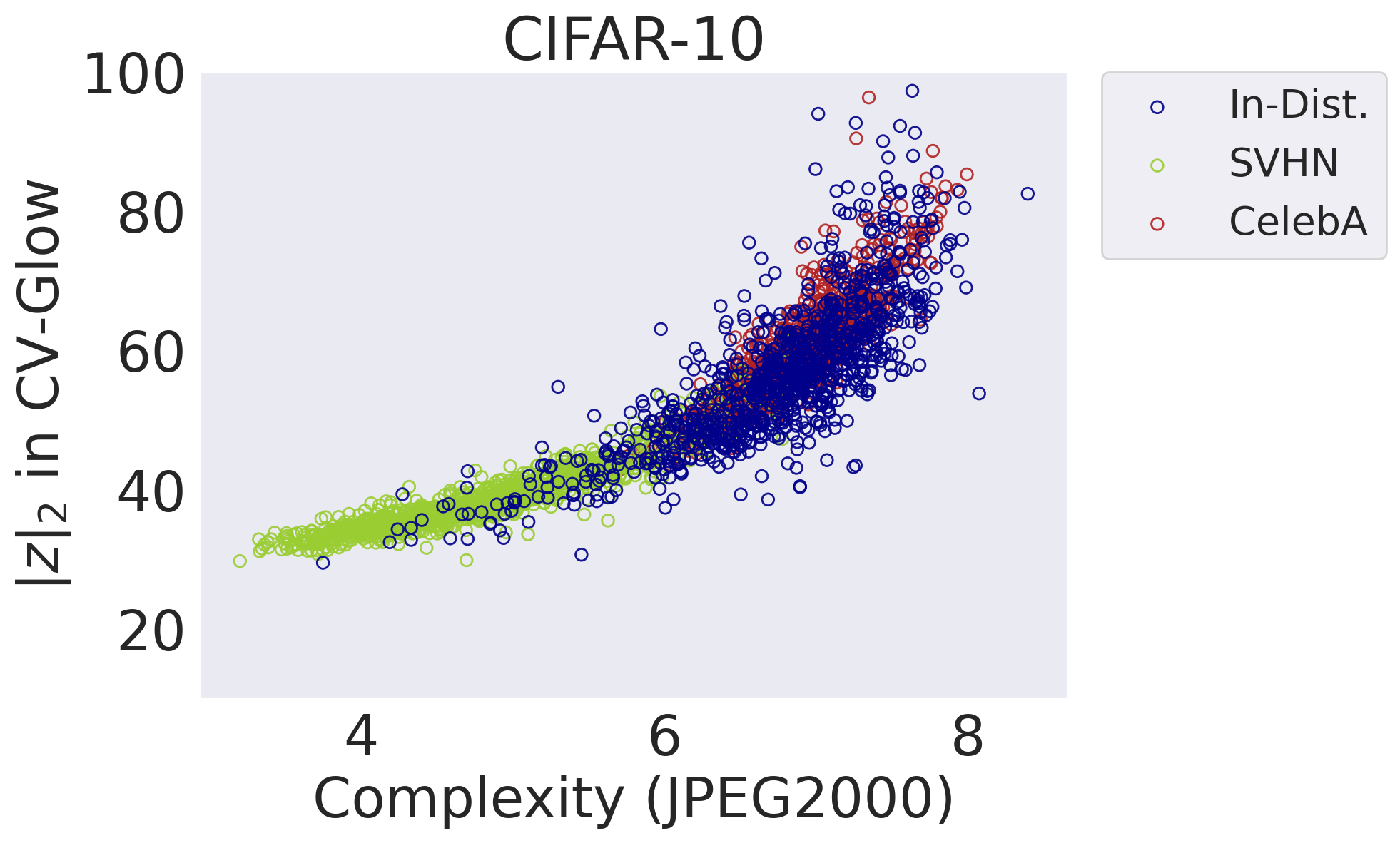}
\includegraphics[height=3.15cm, trim={0 0 0 1.15cm},clip]{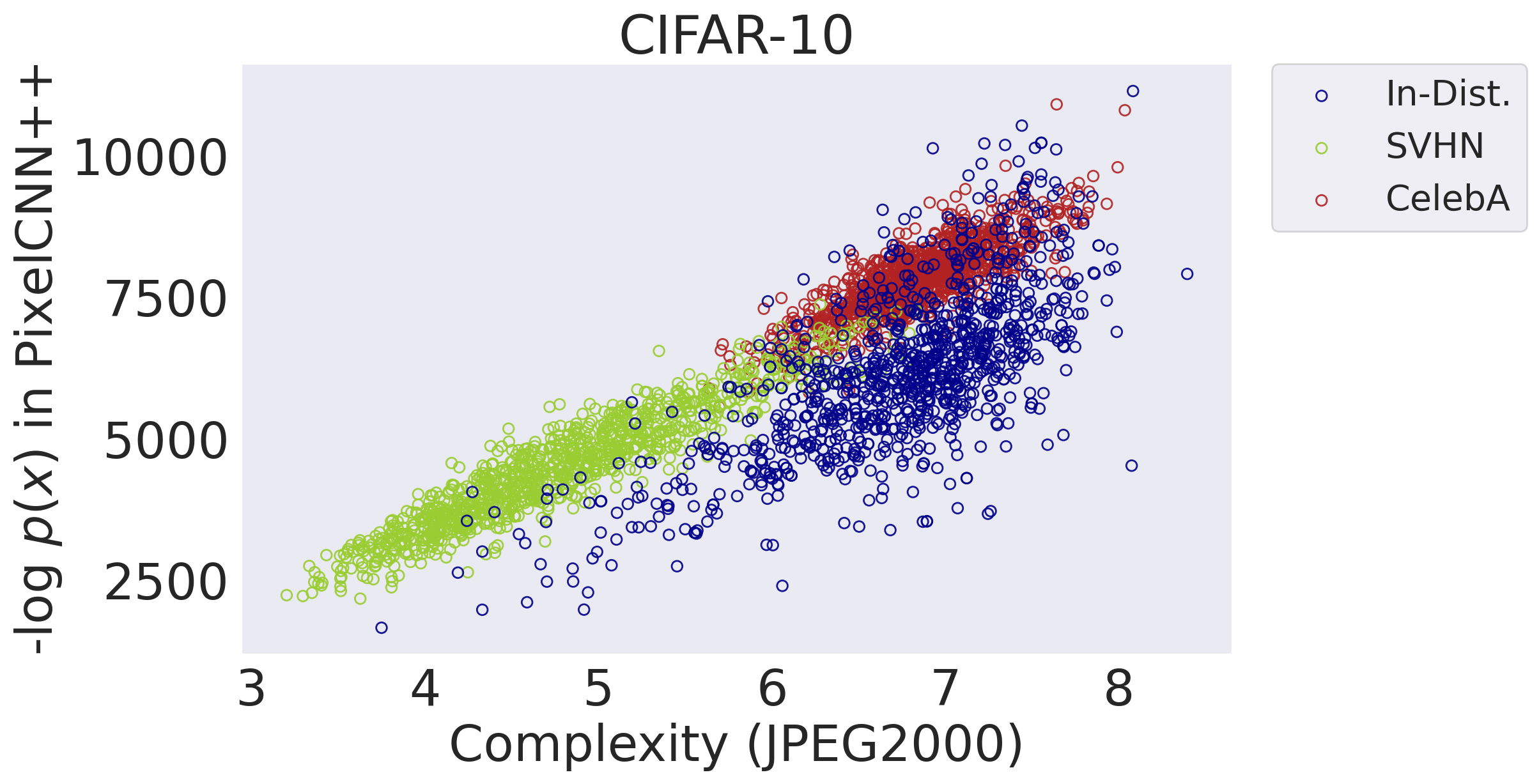}
\caption{Complexity vs.\ $\norm{{\bf z}}$. In-Dist is CIFAR-10. Top is for iResNet and middle is for CV-Glow. $- \log p({\bf x})$ is used instead of $\norm{{\bf z}}$ for PixelCNN++ in the bottom. }
\label{fig:extra_comp_vs_z__ood}
\end{figure}

\subsubsection{AUROC and AUPR on MNIST and FMNIST}
We show  the detection performance of our methods on MNIST and FMNIST are shown.
Tables \ref{tab:AUROC_MNIST_FMNIST} and \ref{tab:AUPR_MNIST_FMNIST} show the AUROC and AUPR, respectively.
The proposed method successfully detects all cases; Glow performed best, followed by ResFlow and iResNet.
The iResNet is less effective in the case where In-Dist is MNIST and OOD is FMNIST.
While CALT fails in the two cases, TTL is comparable to our method with Glow.

\begin{table}[h]
\centering
\caption{AUROC (\%)$\uparrow$. In-Dist datasets are MNIST (left) and FMNIST (right). CALT fails in two cases (lower than $60\%$) on FMNIST.}
\label{tab:AUROC_MNIST_FMNIST}
\begin{adjustbox}{max width = \columnwidth}
	\begin{threeparttable}[h]
		\vspace{0.20cm}
		\scalebox{0.85}{
			\begin{tabular}{lrrr}
				\toprule
				&  FMNIST &  OMNIGLOT &  NotMNIST \\
				\midrule
				TTL         &   98.88 &     98.84 &     99.96 \\
				CALT        &   84.83 &     93.54 &     64.11 \\
				\midrule
				Ours (Glow) &   98.82 &     99.97 &     99.98 \\
				Ours (ResFlow)  &   95.53 &     99.87 &     99.98 \\
				Ours (iResNet)  &   61.03 &     94.52 &     98.54 \\
				\bottomrule
			\end{tabular}
			\hspace{0.50cm}
			\begin{tabular}{lrrr}
				\toprule
				&  MNIST &  OMNIGLOT &  NotMNIST \\
				\midrule
				TTL         &  99.06 &     97.01 &     99.06 \\
				CALT        &  \underline{52.51} &     97.37 &     \underline{52.51} \\
				\midrule
				Ours (Glow) &  98.55 &     99.55 &     98.55 \\
				Ours (ResFlow) &  93.32 &     88.32 &     93.32 \\
				Ours (iResNet)  &  97.59 &     86.63 &     97.59 \\
				\bottomrule
			\end{tabular}
		}
	\end{threeparttable}
\end{adjustbox}
\end{table}

\begin{table}[h]
\centering
\caption{AUPR (\%)$\uparrow$. In-Dist datasets are MNIST (left) and FMNIST (right). }
\label{tab:AUPR_MNIST_FMNIST}
\begin{adjustbox}{max width=\columnwidth}
	\begin{threeparttable}[t]
		\vspace{0.20cm}
		\scalebox{0.85}{
			\begin{tabular}{lrrr}
				\toprule
				&  FMNIST &  OMNIGLOT &  NotMNIST \\
				\midrule
				TTL               &   99.17 &     99.20 &     99.97 \\
				CALT              &   86.16 &     95.10 &     52.69 \\
				\midrule
				Ours (Glow)       &   98.68 &     99.97 &     99.98 \\
				Ours (ResFlow)    &   93.88 &     99.88 &     99.98 \\
				Ours (iResNet)    &   57.93 &     96.64 &     97.79 \\
				\bottomrule
			\end{tabular}
			\hspace{0.50cm}
			\begin{tabular}{lrrr}
				\toprule
				&  MNIST &  OMNIGLOT &  NotMNIST \\
				\midrule
				TTL               &  84.17 &     97.26 &     99.15 \\
				CALT              &  85.41 &     97.60 &     47.03 \\
				\midrule
				Ours (Glow)       &  89.54 &     99.64 &     98.60 \\
				Ours (ResFlow)    &  87.63 &     87.72 &     88.80 \\
				Ours (iResNet)    &  51.67 &     90.08 &     96.21 \\
				\bottomrule
			\end{tabular}
		}
	\end{threeparttable}
\end{adjustbox}
\end{table}

\subsubsection{AUPR on CIFAR-10 and SVHN}
\label{sec:aupr_C10_SVHN}
We show  the detection performance of our methods as the area under the precision-recall curve (AUPR).
Table \ref{tab:AUPR_CIFAR10_SVHN} shows the AUPR on CIFAR-10 and SVHN.
Table \ref{tab:AUPR_pcnnpp} shows the AUPR for PixelCNN++ on CIFAR-10 and SVHN.

Other than our methods, LRG is the only method that scores higher than 60\% in the AUROC results shown in Tables \ref{tab:AUROC_CIFAR10_SVHN} in all test cases.
We consider this result as follows: LRG uses the general likelihood $\log p_{g}({\bf x})$ computed by the NF model trained on more \bsq{general} data than the In-Dist data and computes the ratio as $S_{\text{LRG}} = \log p({\bf x}) - \log p_{g}({\bf x})$.\footnote{We trained the NF model for $\log p_{g}({\bf x})$ using TinyImageNet.}
Because the variation caused by $C({\bf x})$ that occurs in $\log p({\bf x})$ would have occurred in $\log p_{g}({\bf x})$ as well, subtracting $\log p_{g}({\bf x})$ could cancel out the influence of $C({\bf x})$ indirectly, leading to the success of LRG.

\begin{table*}[h]
\centering
\caption{AUPR (\%)$\uparrow$. In-Dist datasets are CIFAR-10 (top) and SVHN (bottom)}
\label{tab:AUPR_CIFAR10_SVHN}
\begin{adjustbox}{max width=\textwidth}
	\begin{threeparttable}
		\vspace{0.20cm}
		\scalebox{0.80}{
			\begin{tabular}{lrrrrrrrrrrrr}
				\toprule
				&  SVHN &  CelebA &   TIN &   Bed &  Living &  Tower &  Noise-1 &  Noise-2 &  Noise-4 &  Noise-8 &  Noise-16 &  Noise-32 \\
				\midrule
				TTL               & 47.03 &   87.58 & 86.06 & 91.02 &   91.89 &  90.23 &   100 &    53.83 &    40.19 &    85.98 &     95.41 &     99.97 \\
				CALT              & 92.06 &   59.64 & 39.21 & 55.43 &   42.28 &  57.11 &    98.77 &   100 &   100 &   100 &    100 &    100 \\
				LRB               & 46.78 &   72.81 & 42.16 & 43.78 &   42.69 &  38.81 &    30.81 &    30.69 &    33.40 &    33.31 &     32.97 &     32.41 \\
				LRG               & 62.08 &   95.37 & 66.44 & 71.78 &   67.17 &  69.74 &   100 &    78.26 &    99.29 &    99.31 &     97.73 &     96.72 \\
				WAIC              & 78.09 &   42.76 & 74.10 & 81.10 &   87.21 &  84.32 &      100 &   100 &      89.9  &    94.70 &     98.87 &     99.78 \\
				\midrule
				Ours (Glow)       & 92.17 &   83.35 & 86.10 & 86.34 &   88.20 &  87.09 &   100 &    99.58 &   100 &   100 &    100 &    100 \\
				Ours (ResFlow)    & 93.10 &   99.93 & 91.84 & 93.69 &   94.25 &  92.03 &   100 &    99.99 &   100 &    99.97 &    100 &    100 \\
				Ours (iResNet)    & 88.38 &   65.42 & 59.11 & 81.47 &   86.56 &  78.76 &   100 &    99.95 &    32.58 &    99.64 &    100 &    100 \\
				\bottomrule
			\end{tabular}
		}
	\end{threeparttable}
\end{adjustbox}

\begin{adjustbox}{max width=\textwidth}
	\begin{threeparttable}[t]
		\vspace{0.20cm}
		\scalebox{0.80}{
			\begin{tabular}{lrrrrrrrrrrrr}
				\toprule
				&  CIFAR10 &  CelebA &    TIN &    Bed &  Living &  Tower &  Noise-1 &  Noise-2 &  Noise-4 &  Noise-8 &  Noise-16 &  Noise-32 \\
				\midrule
				TTL               &    98.26 &   99.98 &  99.90 &  99.99 &  100 &  99.90 &   100 &   100 &    99.98 &    78.25 &     83.79 &     99.93 \\
				CALT              &    31.58 &   31.09 &  31.70 &  33.10 &   31.76 &  35.01 &   100 &   100 &   100 &   100 &    100 &    100 \\
				LRB               &    30.76 &   31.19 &  31.77 &  33.13 &   32.34 &  33.19 &    46.35 &    30.69 &    30.70 &    31.44 &     33.47 &     32.56 \\
				LRG               &    92.82 &   99.96 & 100 & 100 &  100 & 100 &   100 &   100 &    99.67 &    98.84 &     95.89 &     95.42 \\
				WAIC              &    99.58 &   99.44 &  99.83 &  99.96 &   99.99 &  99.56 &   100 &   100 &   100 &    98.82 &     96.95 &     98.61 \\
				\midrule
				Ours (Glow)       &    97.03 &   99.96 &  99.88 &  99.97 &  100 &  99.80 &   100 &   100 &   100 &   100 &    100 &    100 \\
				Ours (ResFlow)    &    57.80 &   66.50 &  78.01 &  65.91 &   83.07 &  51.01 &   100 &    94.77 &   100 &   100 &    100 &    100 \\
				Ours (iResNet)    &    86.56 &   98.79 &  98.56 &  99.66 &   99.88 &  98.25 &   100 &   100 &    99.97 &    70.85 &     98.96 &     99.10 \\
				\bottomrule
			\end{tabular}
		}
	\end{threeparttable}
\end{adjustbox}
\end{table*}

\begin{table*}[h]
\centering
\caption{AUPR (\%) $\uparrow$ for our proposed complexity-aware method and the likelihood test with PixelCNN++ as a baseline. In-Dist datasets are CIFAR-10 (top) and SVHN (bottom).}
\label{tab:AUPR_pcnnpp}
\begin{adjustbox}{max width=\textwidth}
	\begin{threeparttable}[t]
		\vspace{0.20cm}
		\scalebox{0.80}{
			\begin{tabular}{lrrrrrrrrrrr}
				\toprule
				&  SVHN &  CelebA &   TIN &   Bed &  Living &  Tower &   Noise-2 &  Noise-4 &  Noise-8 &  Noise-16 &  Noise-32 \\
				\midrule
				PixelCNN++        & 85.30 &   35.58 & 35.13 & 33.27 &   31.97 &  35.81 &       30.69 &    58.34 &    99.39 &    100 &    100 \\
				Ours & 96.04 &   83.57 & 71.79 & 80.09 &   80.79 &  82.46 &    100 &   100 &   100 &    100 &    100 \\
				\bottomrule
			\end{tabular}
		}
	\end{threeparttable}
\end{adjustbox}
%
\begin{adjustbox}{max width=\textwidth}
	\begin{threeparttable}[t]
		\vspace{0.20cm}
		\scalebox{0.80}{
			\begin{tabular}{lrrrrrrrrrrr}
				\toprule
				&  CIFAR-10 &  CelebA &    TIN &    Bed &  Living &  Tower &   Noise-2 &  Noise-4 &  Noise-8 &  Noise-16 &  Noise-32 \\
				\midrule
				PixelCNN++  &    30.88 &   30.69 &  30.77 &  30.70 &   30.69 &  30.76 &        30.69 &    30.69 &    77.85 &     99.27 &     95.55 \\
				Ours &    89.55 &   97.29 &  96.57 &  98.20 &   99.29 &  97.17 &       100 &   100 &   100 &    100 &    100 \\
				\bottomrule
			\end{tabular}
		}
	\end{threeparttable}
\end{adjustbox}
\end{table*}

\subsubsection{AUROC and AUPR on ImageNet}
\label{sec:ret_IMAGENT}
We verified the validity of our claims even for large-sized images of $224 \times 224$.
Table \ref{tab:ret_IMAGENET} show the AUROC and AUPR on ImageNet.
Our proposed approach using a GMM, which utilizes the two variables $\log p({\bf z})$ from a trained Glow and the image complexity, shows an almost flawless ability to detect artificially generated noise images from the pooling noise image dataset.

\begin{table*}[h]
\centering
\caption{The results on ImageNet. Top is AUROC (\%)$\uparrow$ and  bottom is AUPR (\%)$\uparrow$. Failure cases (lower than $60\%$) are underlined in the AUROC table.}
\label{tab:ret_IMAGENET}
\begin{adjustbox}{max width=\textwidth}
	\begin{threeparttable}[t]
		\vspace{0.20cm}
		\scalebox{0.85}{
			\begin{tabular}{lrrrrrrrrr}
				\toprule
				& Noise-1 & Noise-2 &  Noise-4 &  Noise-8 &  Noise-16 &  Noise-32 &  Noise-64 &  Noise-128 &  Noise-224\\
				\midrule
				TTL                & 100 & 100 & \underline{57.45} & 99.53 & 61.78 & 89.42 & 99.69 & 99.99 & 99.93  \\
				CALT            & 100 & 100 & 74.39 &100 & 100  & 100  &100  &100 &100  \\
				\midrule
				Ours (Glow)& 100 & 100 & 98.09 & 100 & 99.51 & 99.80 & 100 & 100 & 100  \\
				\bottomrule
			\end{tabular}
		}
	\end{threeparttable}
\end{adjustbox}
\begin{adjustbox}{max width=\textwidth}
	\begin{threeparttable}[t]
		\vspace{0.20cm}
		\scalebox{0.85}{
			\begin{tabular}{lrrrrrrrrr}
				\toprule
				& Noise-1 & Noise-2 &  Noise-4 &  Noise-8 &  Noise-16 &  Noise-32 &  Noise-64 &  Noise-128 &  Noise-224\\
				\midrule
				TTL                & 100 & 100 & 48.44 & 98.65 & 50.95 & 78.05 & 99.59 & 99.99 & 99.93  \\
				CALT            & 100 & 100 & 86.18 & 100 & 100  & 100  &100  &100 &100  \\
				\midrule
				Ours (Glow)& 100 & 100 & 99.03 & 100 & 99.75 & 99.89 & 100 & 100 & 100  \\
				\bottomrule
			\end{tabular}
		}
	\end{threeparttable}
\end{adjustbox}
\end{table*}

\end{document}